\documentclass[11pt,english]{article}
\usepackage[T1]{fontenc}
\usepackage[utf8]{inputenc}
\usepackage{geometry}
\usepackage{algorithmic}
\usepackage{algorithm}
\usepackage{color}
\usepackage{mathrsfs}
\usepackage{float}
\usepackage{amsmath} 
\usepackage{amssymb}
\usepackage{amsfonts}
\usepackage{amsthm}
\usepackage{stmaryrd}
\usepackage{stackrel}
\usepackage{graphicx}
\usepackage{hyperref}
\usepackage{dsfont}
\usepackage{tikz}
\usepackage{framed}
\usepackage{appendix}
\usepackage{cleveref}
\usepackage{enumerate}
\usepackage{placeins}
\usepackage{authblk}

\numberwithin{equation}{section}
\theoremstyle{plain}
\newtheorem{theorem}{Theorem}
\newtheorem{mydef}{Definition}
\newtheorem{corollary}{Corollary}
\newtheorem{lemme}{Lemma}

\usepackage[round]{natbib}
\usepackage{hyperref}
\hypersetup{
colorlinks = true,
urlcolor = blue, 
linkcolor = blue,
citecolor = blue,
}

\newcommand{\Dn}{\mathscr{D}_{n}}
\newcommand{\bX}{\textbf{X}}
\newcommand{\bestcut}{\mathcal{C}}
\newcommand{\bx}{\textbf{x}}
\renewcommand{\P}{\mathds{P}}
\newcommand{\R}{\mathds{R}}
\newcommand{\E}{\mathbb{E}}

\renewcommand{\path}{\mathscr{P}}
\newcommand{\pathset}{\hat{\path}}
\newcommand{\B}{\bfseries}
\newcommand{\setpstar}{\mathcal{U}^{\star}}
\newcommand{\setpemp}{\mathcal{U}_n}

\begin{document}
\emergencystretch 3em

\title{\textbf{SIRUS: Stable and Interpretable RUle Set for Classification}}
\author{Cl\'ement B\'enard\thanks{Safran Tech, Sorbonne Universit\'e} \hspace*{3mm}
G\'erard Biau \thanks{Sorbonne Universit\'e} \hspace*{3mm}
S\'ebastien Da Veiga \thanks{Safran Tech} \hspace*{3mm}
Erwan Scornet\thanks{Ecole Polytechnique}}
\date{}

\maketitle

\begin{abstract}
State-of-the-art learning algorithms, such as random forests or neural networks, are often qualified as ``black-boxes'' because of the high number and complexity of operations involved in their prediction mechanism. This lack of interpretability is a strong limitation for applications involving critical decisions, typically the analysis of production processes in the manufacturing industry. In such critical contexts, models have to be interpretable, i.e., simple, stable, and predictive. 
To address this issue, we design SIRUS (Stable and Interpretable RUle Set), a new classification algorithm based on random forests, which takes the form of a short list of rules.
While simple models are usually unstable with respect to data perturbation, SIRUS achieves a remarkable stability improvement over cutting-edge methods. Furthermore, SIRUS inherits a predictive accuracy close to random forests, combined with the simplicity of decision trees. These properties are assessed both from a theoretical and empirical point of view, through extensive numerical experiments based on our \texttt{R/C++} software implementation \texttt{sirus} available from \texttt{CRAN}. \vspace*{1mm} 
\end{abstract}

\section{Introduction} \label{sec_intro}

State-of-the-art learning algorithms, typically tree ensembles or neural networks, are well-known for their remarkable predictive performance. However, this high accuracy comes at the price of complex prediction mechanisms: a large number of operations are computed for a given prediction. Because of this complexity, learning algorithms are often considered as black-boxes. This lack of interpretability is a serious limitation for many applications involving critical decisions, such as healthcare, criminal justice, or industrial process optimization. This latter example is interesting to illustrate how interpretability can be essential.
Indeed, in the manufacturing industry, production processes involve complex physical and chemical phenomena, whose control and efficiency are of critical importance. In practice, data is collected along the manufacturing line, describing both the production environment and its conformity.
The retrieved information enables to infer a link between the manufacturing conditions and the resulting quality at the end of the line, and then to increase the process efficiency.
Since the quality of the produced entities is often characterized by a pass or fail output, the problem is in fact a classification task, and state-of-the-art learning algorithms can successfully catch patterns of these complex and nonlinear physical phenomena.
However, any decision impacting the production process has long-term and heavy consequences, and therefore cannot simply rely on a blind stochastic modelling.
As a matter of fact, a deep physical understanding of the forces in action is required, and this makes black-box algorithms inappropriate. In a word, models have to be interpretable, i.e., provide an understanding of the internal mechanisms that build a relation between inputs and outputs, to provide insights to guide the physical analysis.
This is for example typically the case in the aeronautics industry, where the manufacturing of engine parts involves sensitive casting and forging processes.
Interpretable models allow us to gain knowledge on the behavior of such production processes, which can lead, for instance, to identify or fine-tune critical parameters, improve measurement and control, optimize maintenance, or deepen understanding of physical phenomena. 
In the following paragraphs, we deepen the discussion about the definition of interpretability to highlight the limitations of the most popular interpretable nonlinear models: decision trees and rule algorithms \citep{guidotti2018survey}. Despite their high predictivity and simple structure, these methods are unstable, which is a strong operational limitation. The goal of this article is to introduce \textbf{SIRUS} (\textbf{S}table and \textbf{I}nterpretable \textbf{RU}le \textbf{S}et),
an interpretable rule classification algorithm which considerably improves stability over state-of-the-art methods, while preserving their simple structure, accuracy, and computational complexity.

As stated in \cite{ruping2006learning}, \cite{lipton2016mythos}, \cite{doshi2017towards}, or \cite{murdoch2019interpretable}, to date, there is no agreement in statistics and machine learning communities about a rigorous definition of interpretability. There are multiple concepts behind it, many different types of methods, and a strong dependence on the area of application and the audience. Here, we focus on models intrinsically interpretable, which directly provide insights on how inputs and outputs are related, as opposed to the post-processing of black-box models. In that case, we argue that it is possible to define minimum requirements for interpretability through the triptych ``simplicity, stability, and predictivity'',
in line with the framework recently proposed by \cite{yu2019three}.
Indeed, in order to grasp how inputs and outputs are related, the structure of the model has to be simple. The notion of simplicity is implied whenever
interpretability is invoked \citep[e.g.,][]{ruping2006learning, freitas2014comprehensible, letham2015statistical, letham2015interpretable, lipton2016mythos, ribeiro2016should, murdoch2019interpretable}
and essentially refers to the model size, complexity, or the number of operations performed in the prediction mechanism.
\cite{yu2013stability} defines stability as another fundamental requirement for interpretability: conclusions of a statistical analysis have to be robust to small data perturbations to be meaningful. Indeed, a specific analysis is likely to be run multiple times, eventually adding a small new batch of data, and an interpretable algorithm should be insensitive to such modifications. Otherwise, unstable models provide us with a partial and arbitrary analysis of the underlying phenomena, and arouses distrust of the domain experts.
Finally, if the predictive accuracy of an interpretable model is significantly lower than the one of a state-of-the-art black-box algorithm,
it clearly misses strong patterns in the data and will therefore be useless, as explained in \cite{breiman2001statistical}. For example, the trivial model that outputs the empirical mean of the observations for any input is simple, stable, but brings in most cases no useful information. Thus, we add a good predictivity as an essential requirement for interpretability.

Decision trees are a class of supervised learning algorithms that recursively partition the input space and make local decisions in the cells of the resulting partition. Trees can model highly nonlinear patterns while having a simple structure, and are therefore good candidates when interpretability is required. 
However, trees are unstable to small data perturbations \citep{oates1997ects, guidotti2019stability}. More precisely, as explained in \cite{breiman2001statistical}: by randomly removing only $2-3\%$ of the training data, the tree structure can be quite different, which is a strong limitation to their practical use. 
Another class of supervised learning methods that can model nonlinear patterns while retaining a simple structure are the so-called rule models. As such, a rule is defined as a conjunction of
constraints on input variables, which form a hyperrectangle in the input space where the estimated output is constant. A collection of rules is combined to form a model. Here, the term ``rule'' does not stand for ``classification rule'' but, as is traditional in the rule learning literature, to a piecewise constant estimate that simply reads ``if \textit{conditions on $\bx$}, then \textit{response}, else \textit{default response}''. Despite their simplicity and excellent predictive skills, rule algorithms are unstable and, from this point of view, share the same limitation as decision trees \citep{letham2015interpretable, murdoch2019interpretable}. 

In line with the above, we design SIRUS in the present paper, a new rule classification algorithm which inherits an accuracy close to random forests and the simplicity of decision trees, while having a stable structure. The core aggregation principle of random forests is kept, but instead of aggregating predictions, SIRUS focuses on the probability that a given hyperrectangle (i.e., a node) is contained in a randomized tree. The nodes with the highest probability are robust to data perturbation and represent strong patterns. They are therefore selected to form a stable rule ensemble model.
Here, we provide a first illustration of SIRUS with a simple and real case: the Titanic dataset \citep{titanicdata}. The survival status of $887$ passengers are recorded, as well as various personal characteristics: age, sex, class, number of siblings and parents aboard, and the paid fare. SIRUS outputs the following simple set of $7$ rules, which enables to grasp at a glance the main patterns to explain passenger survival:
\begin{center}
\setlength{\fboxrule}{1pt}
\fbox{\begin{minipage}{0.71\textwidth}
\textbf{Average survival rate} $p_s = 39 \%$. \\[1em]
\begin{tabular}{c c c c c c}
\textbf{if} & \texttt{sex} is male & \textbf{then} & $p_s = 19\%$ & \textbf{else} & $p_s = 74\%$ \\[0.7em]
\textbf{if} & $1^{st}$ or $2^{nd}$ \texttt{class} & \textbf{then} & $p_s = 56\%$ & \textbf{else} & $p_s = 24\%$ \\[0.7em]
\textbf{if} & \begin{tabular}{c} $1^{st}$ or $2^{nd}$ \texttt{class} \\[-0.3em] \textbf{\&} \texttt{sex} is female \end{tabular} & \textbf{then} & $p_s = 95\%$ & \textbf{else} & $p_s = 25\%$ \\[0.7em]
\textbf{if} & \texttt{fare} $< 10.5$£  & \textbf{then} & $p_s = 20\%$ & \textbf{else} & $p_s = 50\%$ \\[0.7em]
\textbf{if} &  \begin{tabular}{c} no \texttt{parents or} \\[-0.3em] \texttt{children aboard} \end{tabular} & \textbf{then} & $p_s = 35\%$ & \textbf{else} & $p_s = 51\%$ \\[0.7em]
\textbf{if} & \begin{tabular}{c} $2^{st}$ or $3^{nd}$ \texttt{class} \\[-0.3em] \textbf{\&} \texttt{sex} is male \end{tabular} & \textbf{then} & $p_s = 14\%$ & \textbf{else} & $p_s = 64\%$ \\[0.7em]
\textbf{if} & \begin{tabular}{c}  \texttt{sex} is male \\[-0.3em] \textbf{\&} \texttt{age} $\geq 15$ \end{tabular} & \textbf{then} & $p_s = 16\%$ & \textbf{else} & $p_s = 72\%$ 
\end{tabular}
\end{minipage}}
\end{center}
To generate the prediction for a new query point $\bx$, SIRUS checks for each rule whether the conditions are satisfied to assign one of the two possible $p_s$ output values. Let us say for example that $x^{(sex)}$ is female, then $\bx$ satisfies the condition of the first rule, which returns $p_s = 74\%$. Next, the $7$ rule outputs are averaged to provide the predicted probability of survival for $\bx$. 
The model is stable: when a $10$-fold cross-validation is run to simulate data perturbation, $5$ to $6$ rules are consistent across two folds in average. The model error (1-AUC) is $0.17$, close to the $0.13$ of random forests, whereas simplicity is drastically increased: $7$ rules versus about $10^4$ operations for a forest prediction.

First, we review the main rule algorithms and present their mechanism principles in Section \ref{sec_literature}.
Next, Section \ref{sec_algo} is devoted to the detailed description of SIRUS. One of the main contributions of this work is the development of a software implementation, via the \texttt{R/C++} package \texttt{sirus} \citep{Rsirus} available from \texttt{CRAN}, based on \texttt{ranger}, a high-performance random forest implementation \citep{wright2017ranger}. In Section \ref{sec_theory}, we show that the good empirical behavior of SIRUS is theoretically understood by proving its asymptotic stability. Then, in Section \ref{sec_xp}, we illustrate the efficiency of our algorithm through numerical experiments on real datasets. Finally, Section \ref{sec_conclusion} summarizes the main contributions of the article and provides directions for future research.

\section{Related Work} \label{sec_literature}

As stated in the introduction, SIRUS has two types of competitors: decision trees and rule algorithms. More precisely, the latter can further be split into three different kinds: classical rule algorithms based on greedy heuristics, those built on top of frequent pattern mining algorithms, and those extracted from tree ensembles.

Decision trees may be the most popular competitors of SIRUS because of their simple structure. The main algorithms are CART \citep{leo1984classification} and C5.0 \citep{quinlan1992c4}. However, trees are unstable as we have already highlighted.
A widespread method to stabilize decision trees is bagging \citep{breiman1996bagging}, in which multiple trees are grown on perturbed data and aggregated together. Random forests is an algorithm developped by \cite{breiman2001random} that improves over bagging by randomizing the tree construction. Predictions are stable, accuracy is increased, but the final model is unfortunately a black box. Thus, simplicity of trees is lost, and some post-treatment mechanisms are needed to understand how random forests make their decisions. Nonetheless, even if they are useful, such treatments only provide partial information and can be difficult to operationalize for critical decisions \citep{rudin2018please}. For example, variable importance \citep{breiman2001random, breiman2003atechnical} identifies variables that have a strong impact on the output, but not which inputs values are associated to output values of interest. Similarly, local approximation methods such as LIME \citep{ribeiro2016should} or \citet{tolomei2017interpretable} do not provide insights on the global relation.

Rule learning originates from the influential AQ system of \cite{michalski1969quasi}.
Many algorithms based on greedy heuristics were subsequently developped in the 1980's and 1990's, including 
Decision List \citep{rivest1987learning}, CN2 \citep{clark1989cn2}, FOIL \citep[First-Order Inductive Learner,][]{quinlan1990learning,quinlan1995induction}, IREP \citep[Incremental Reduced Error Pruning,][]{furnkranz1994incremental}, RIPPER \citep[Repeated Incremental Pruning to Produce Error Reduction,][]{cohen1995fast}, 
PART \cite[Partial Decision Trees,][]{frank1998generating}, SLIPPER \cite[Simple Learner with Iterative Pruning to Produce Error Reduction,][]{cohen1999simple}, LRI \cite[Leightweight Rule Induction,][]{weiss2000lightweight}, and ENDER \citep[Ensemble of Decision Rules,][]{dembczynski2010ender}.
Since these methods are based on greedy heuristics, they are computationally fast, but similarly to decision trees, they are unstable and their accuracy is often limited.

At the end of the 1990's a new type of rule algorithms based on frequent pattern mining is introduced with CBA \citep[Classification Based on Association Rules,][]{liu1998integrating}, then extended with CPAR \citep[Classification based on Predictive Association Rules,][]{yin2003cpar}. Frequent pattern mining is originally used to identify frequent occurrences in database mining. Since the output $Y \in \{0,1\}$ is discrete and the input data can be discretized, we can generate candidate rules for classification by identifying frequent patterns associated with each output label. This exhaustive search for association rules is computationally costly (exponential with the input dimension), and efficient heuristics are used, essentially Apriori \citep{agrawal1993mining} and Eclat \citep{zaki1997parallel}.
The rule aggregation mechanism is specific to each algorithm.
More recently, BRL \citep[Bayesian Rule List,][]{letham2015interpretable} uses a more sophisticated Bayesian framework for the rule aggregation than the simple approach of CBA and CPAR, while IDS \citep[Interpretable Decision Sets]{lakkaraju2016interpretable} uses a multi-objective optimization to select interpretable rules. Finally, CORELS \citep[Certifiably Optimal RulE ListS]{angelino2017learning} generates optimal rule lists for categorical data.
Interestingly, these methods exhibit quite good stability properties as we will see, higher than decision trees, but on the other hand, their predictive accuracy is worse.

The last decade has seen a resurgence of rule models through powerful algorithms based on rule extraction from tree ensembles, especially RuleFit \citep{friedman2008predictive} and Node harvest \citep{meinshausen2010node}. Notice that SIRUS is also based on this principle. More specifically, RuleFit extracts all the rules of a boosted tree ensemble \citep{friedman2003importance}, while Node harvest is based on random forests. Then, the extracted rules are linearly combined in a sparse linear model, respectively a logistic regression with a Lasso penalty \citep{tibshirani1996regression} for RuleFit, and a constraint quadratic linear program for Node harvest. These two methods have a computational complexity comparable to random forests and SIRUS, since the main step of all these algorithms is to grow a tree ensemble with a large number of trees. However, both algorithms are unstable, and both output quite complex and long lists of rules. Even running RuleFit or Node harvest multiple times on the same dataset produces quite different rule lists because of the randomness in the tree ensembles---see Appendix \ref{appendix_stability}. On the other hand, SIRUS is built to have its structure converged for the given dataset, as explained later in Section \ref{sec_algo}.

To the best of our knowledge, the signed iterative random forest method  \citep[s-iRF,][]{kumbier2018refining} is the only procedure that tackles both rule learning and stability. Using random forests, s-IRF manages to extract stable signed interactions, i.e., feature interactions enriched with a thresholding behavior for each variable, lower or higher, but without specific thresholding values.
Therefore, s-IRF can be difficult to operationalize since it does not provide any specific input thresholds, and thus no precise information about the influence of input variables. On the other hand, an explicit rule model identifies specific regions of interest in the input space.

\section{SIRUS Algorithm} \label{sec_algo}
Within the general framework of supervised (binary) classification, we assume to be given an i.i.d.~sample $\mathscr{D}_{n} = \{(\bX_{i},Y_{i}), i=1, \hdots, n\}$. Each $(\bX_i,Y_i)$ is distributed as the generic pair $(\bX,Y)$ independent of $\mathscr{D}_{n}$, where $\bX=(X^{(1)},\hdots,X^{(p)})$ is a random vector taking values in $\mathds R^p$ and $Y \in \{0,1\}$ is a binary response. Throughout the document, the distribution of $(\bX,Y)$ is assumed to be unknown and is denoted by $\mathds P_{\bX,Y}$. For $\bx \in \mathds R^p$, our goal is to accurately estimate the conditional probability $\eta(\bx) = \mathds P (Y = 1 |\bX = \bx)$ with few simple and stable rules. 

To tackle this problem, SIRUS first builds a (slightly modified) random forest. Next, each hyperrectangle of each tree of the forest is turned into a simple decision rule, and the collection of these elementary rules is ranked based on their frequency of appearance in the forest. Finally, the most significant rules are retained and are averaged together to form an ensemble model.
We describe the four steps of SIRUS algorithm in the following paragraphs: the rule generation, rule selection, rule post-treatment, and the rule aggregation. This section ends with a discussion of SIRUS stability.

\paragraph{Rule generation.}
SIRUS uses at its core the random forest method \citep{breiman2001random}, slightly modified for our purpose. As in the original procedure, each single tree in the forest is grown with a greedy heuristic that recursively partitions the input space using a random variable $\Theta$.
The essential difference between our approach and Breiman's one is that, prior to all tree constructions, the empirical $q$-quantiles of the marginal distributions over the whole dataset are computed: in each node of each tree, the best split can be selected among these empirical quantiles only. 
This constraint is critical to stabilize the forest structure and keeps almost intact the predictive accuracy, provided $q$ is not too small (typically of the order of 10---see the experimental Subsection \ref{subsec_parameters}). Apart from this difference, the tree growing is similar to Breiman's original procedure. The tree randomization $\Theta$ is independent of the sample and has two independent components, denoted by $\Theta^{(S)}$ and $\Theta^{(V)}$, which are respectively used
for the subsampling mechanism and randomization of the split direction. 
Throughout the manuscript, we let $\hat{q}_{n,r}^{(j)}$ be the empirical $r$-th $q$-quantile of $\{X_1^{(j)}, \hdots, X_n^{(j)}\}$, with typically $q = 10$.
The construction of the individual trees is summarized in Algorithm \ref{alg:split-cell} below.
\begin{algorithm}[htbp]
\caption{Tree construction }
\label{alg:split-cell}
\begin{algorithmic}[1]
  \STATE \textbf{Parameters:} Number of quantiles $q$, number of subsampled observations $a_n$, number of eligible directions for splitting $\texttt{mtry}$.
  \STATE Compute the empirical $q$-quantiles for each marginal distribution over the whole dataset. 
  \STATE Subsample with replacement $a_n$ observations, indexed by $\Theta^{(S)}$. Only these observations are used to build the tree.
  \STATE Initialize the cell $H$ as the root of the tree.
  \STATE Draw uniformly at random a subset $\Theta^{(V)} \subset \{1, \hdots, p\}$ of cardinality \texttt{mtry}.
  \STATE For all $j \in  \Theta^{(V)}$, compute the CART-splitting criterion at all empirical $q$-quantiles of $X^{(j)}$ that split the cell $H$ into two non-empty cells.
  \STATE Choose the split that maximizes the CART-splitting criterion.
  \STATE Recursively repeat \textbf{lines $5-7$} for the two resulting children cells $H_L$ and $H_R$.
\end{algorithmic}
\end{algorithm}

The main step of SIRUS is to extract rules from the modified random forest. The cornerstone of this extraction mechanism is the notion of path in a decision tree. Indeed, a path describes the sequence of splits to go from the root of the tree to a specific (inner or terminal) node. Since a hyperrectangle is associated to each node, a rule can be defined as a piecewise constant estimate with this hyperrectangle as support.
Therefore, to rigorously define the rule extraction, we introduce the symbolic representation of a path in a tree. We insist that such definition is valid for both terminal leaves and inner nodes, which are all used by SIRUS.
To begin, we follow the example shown in Figure \ref{fig_paths} with a tree of depth $2$ partitioning the input space $\R^2$. 
For instance, let us consider the node $\path_6$ defined by the sequence of two splits $\smash{X_i^{(2)} \geq \hat{q}_{n,4}^{(2)}}$
and $\smash{X_i^{(1)} \geq \hat{q}_{n,7}^{(1)}}$. The first split is symbolized by the triplet $(2,4,R)$, whose components respectively stand for the variable index $2$, the quantile index $4$, and the right side $R$ of the split. Similarly, for the second split we cut coordinate 1 at quantile index 7, and pass to the right. Thus, the path to the considered node is defined by $\path_6 =\{ (2,4,R), (1,7,R)\}$. Also notice that the first split already defines the path $\path_2 = \{(2,4,R)\}$, associated to the right inner node at the first level of the tree.
Of course, this generalizes to each path $\mathscr P$ of length $d$ under the symbolic compact form
\[
\mathscr{P} =\{ (j_{k},r_{k},s_{k}), \, k = 1,\hdots,d\},
\]
where, for $k \in \{1,\hdots,d\}$, the triplet $(j_{k},r_{k},s_{k})$ describes how to move from level $(k-1)$ to level $k$, with
a split using the coordinate $j_k\in \{1,\hdots,p\}$, the index $r_{k}\in \{1,\hdots,q - 1\}$ of the corresponding quantile, and
a side $s_{k} = L$ if we go the the left and $s_{k} = R$ if we go to the right.
\begin{figure}
\centering
\begin{tikzpicture}

\draw[->] (-0.7,0) -- (5.5,0);
\draw (5.5,0) node[right] {$x^{(1)}$};
\draw [->] (0,-0.7) -- (0,5);
\draw (0,5) node[above] {$x^{(2)}$};
\draw [color = blue, line width = 0.3 mm] (3.5,2) -- (3.5,5);
\draw [dashed] (3.5,2) -- (3.5,-0.7);
\draw [color = blue, line width = 0.3 mm] (2.75,2) -- (2.75,-0.7);
\draw (4,0) node[below] {$\hat{q}_{n,7}^{(1)}$};
\draw (2,0) node[below] {$\hat{q}_{n,5}^{(1)}$};
\draw [color = blue, line width = 0.3 mm] (5.5,2) -- (-0.7,2);
\draw (0,2) node[below left] {$\hat{q}_{n,4}^{(2)}$};
\draw (0,3.5) node[right, color = blue] {\begin{tabular}{c} $\path_5 =\{ (2,4,R),$ \\ $(1,7,L)\}$ \end{tabular}};
\draw (3.3,3.5) node[right, color = blue] {\begin{tabular}{c} $\path_6 =\{ (2,4,R),$ \\ $(1,7,R)\}$ \end{tabular}};
\draw (-0.35,0.7) node[right, color = blue] {\begin{tabular}{c} $\path_3 =\{ (2,4,L),$ \\ $(1,5,L)\}$ \end{tabular}};
\draw (2.5,1) node[right, color = blue] {\begin{tabular}{c} $\path_4 =\{ (2,4,L),$ \\ $(1,5,R)\}$ \end{tabular}};

\draw (7.5,2.5) --(9,5) --(10.5,2.5);
\node[below left] at (8.5,4.7) {$X_i^{(2)} < \hat{q}_{n,4}^{(2)}$};
\node[below right] at (9.5,4.7) {$X_i^{(2)} \geq \hat{q}_{n,4}^{(2)}$};
\node[above left, color = blue] at (7.6,2.4) {$\path_1$};
\node[above right, color = blue] at (10.4,2.4) {$\path_2$};
\draw (9.5,0) --(10.5,2.5) --(11.5,0);
\node[left] at (10.9,1.2) {$X_i^{(1)} < \hat{q}_{n,7}^{(1)}$};
\node[right] at (10.1,0.6) {$X_i^{(1)} \geq \hat{q}_{n,7}^{(1)}$};
\node[below, color = blue] at (9.5,0) {$\path_5$};
\node[below, color = blue] at (11.5,0) {$\path_6$};
\draw (6.5,0) --(7.5,2.5) --(8.5,0);
\node[left] at (7.8,1.2) {$X_i^{(1)} < \hat{q}_{n,5}^{(1)}$};
\node[right] at (7.3,0.6) {$X_i^{(1)} \geq \hat{q}_{n,5}^{(1)}$};
\node[below, color = blue] at (6.5,0) {$\path_3$};
\node[below, color = blue] at (8.5,0) {$\path_4$};

\end{tikzpicture}
\caption{\small{Example of a root node $\R^2$ partitionned by a randomized tree of depth 2: the tree on the right side, the associated paths and hyperrectangles of length $d=2$ on the left side.}}
\label{fig_paths}
\end{figure}
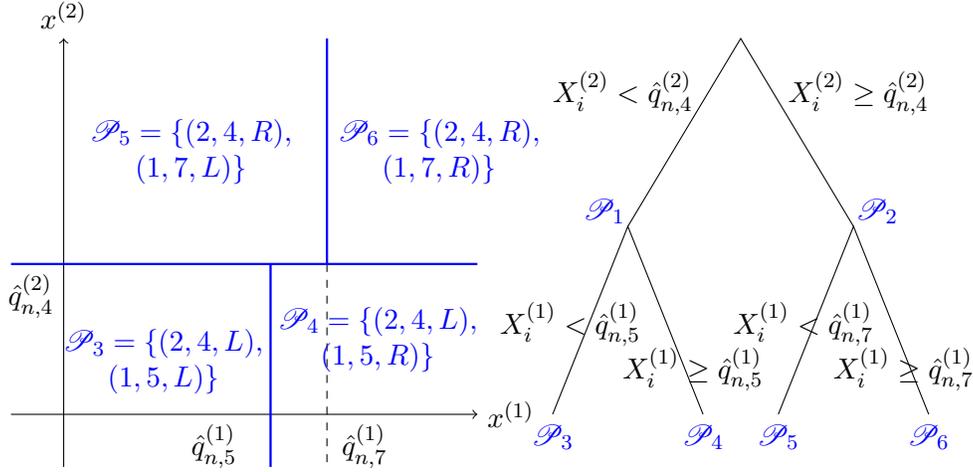
The set of all possible such paths is denoted by $\Pi$. It is important to note that $\Pi$ is in fact a deterministic (that is, non random) quantity, which only depends upon the dimension $p$ and the order $q$ of the quantiles.
Of course, given a path $\mathscr{P} \in \Pi$ one can recover the hyperrectangle (i.e., the tree node) $\hat{H}_{n}(\path)$ associated with $\mathscr{P}$ and the entire dataset $\Dn$ via the correspondence
\begin{align} \label{eq_emp_H}
\hat{H}_{n}(\path)=\left\{ \bx \in \mathds{R}^{p}:
\begin{cases}
\bx^{(j_{k})}<\hat{q}_{n,r_{k}}^{(j_{k})} & \textrm{if }s_{k}=L\\
\bx^{(j_{k})}\geq\hat{q}_{n,r_{k}}^{(j_{k})} & \textrm{if }s_{k}=R
\end{cases}\;,\, k=1,\hdots,d \right\}.
\end{align}
Finally, an elementary rule $\hat{g}_{n,\path}$ can be defined from $\hat{H}_{n}(\path)$ as a piecewise constant estimate: $\hat{g}_{n,\path}(\bx)$ returns the empirical probability that the output $Y$ is of class $1$ conditional on whether the query point $\bx$ belongs to $\hat{H}_{n}(\path)$ or not. Thus, the rule $\hat{g}_{n,\path}$ associated to the path $\path \in \Pi$ is formally defined by 
\begin{align*}
\forall \bx \in \R^{p}, \quad \hat{g}_{n,\path} (\bx) = 
\begin{cases}
\frac{1} { N_n (\hat{H}_{n}(\path) )}
\sum_{i = 1}^{n} Y_{i} \mathds{1}_{\bX_{i} \in \hat{H}_{n}(\path)} &\textrm{if}~\bx \in \hat{H}_{n}(\path)\\
\frac{1} { n - N_n (\hat{H}_{n}(\path) )}
\sum_{i = 1}^{n} Y_{i} \mathds{1}_{\bX_{i} \notin \hat{H}_{n}(\path)}&\textrm{otherwise}
\end{cases},
\end{align*}
using the convention $0/0=0$, and where $N_n(\hat{H}_n(\path))$ is the number of observations in the node associated with $\path$.
This formal definition can be illustrated with the Titanic dataset presented in the introduction. For the fourth rule, \texttt{fare} is the $6$th variable and since $\hat{q}_{n,4}^{(6)} = 10.5$, the corresponding path is $\path = \{(6,4,L)\}$, and the associated rule is thus
\begin{align*}
\hat{g}_{n,\path} (\bx) = 
\begin{cases}
0.20 &\textrm{if}~x^{(6)} < 10.5\\
0.50 &\textrm{if}~x^{(6)} \geq 10.5
\end{cases}.
\end{align*}
Finally, a $\Theta$-random tree generates a collection of paths in $\Pi$, one for each internal and terminal nodes. In the sequel, we let $T(\Theta,\mathscr {D}_{n})$ be the list of such extracted paths, a random subset of $\Pi$.

\paragraph{Rule selection.} Using our modified random forest algorithm, we are able to generate a large number $M$ of trees, randomized by $\Theta_1, \hdots, \Theta_M$, i.i.d.~copies of the generic variable $\Theta$, and then to extract a large collection of rules.
Since we are interested in selecting the most important rules, i.e., those which  represent strong patterns between the inputs and the output, we select rules that are shared by a large portion of trees. Such occurrence frequency is formally defined by
\[
\hat{p}_{M,n}(\path)=\frac{1}{M}\sum_{\ell=1}^{M}\mathds{1}_{\path\in T(\Theta_{\ell},\mathscr{D}_{n})},
\]
which is the Monte-Carlo estimate of the probability that a path $\path$ belongs to a $\Theta$-random tree, that is
\[
p_n(\mathscr{P})=\mathbb{P}(\mathscr{P}\in T(\Theta,\mathscr{D}_{n})|\mathscr{D}_{n}).
\]
As a general strategy, once the modified random forest has been built, we draw the list of all paths that appear in the forest and only retain those that occur with a frequency larger than the threshold $p_0 \in (0,1)$, the only influential parameter of SIRUS---see Subsection \ref{subsec_parameters} for its tuning procedure. We are thus interested in the set of the extracted paths
\begin{align} \label{eq_emp_pathset}
\pathset_{M,n,p_{0}}=\{ \path \in \Pi:\hat{p}_{M,n}(\path)>p_{0}\}. 
\end{align}
An important feature of SIRUS algorithm is to stop the growing of the forest with an appropriate number of trees $M$. 
Although the right order of magnitude for $M$ is required, no fine tuning is necessary.
Indeed, the uncertainty of the importance estimate $\hat{p}_{M,n}(\path)$ of each rule decreases with $M$, whereas the computational cost linearly increases with $M$. Thus, to obtain a robust rule extraction, $M$ needs to be high enough to make the uncertainty of $\hat{p}_{M,n}(\path)$ negligible. More precisely, $M$ is set to get the same list of selected rules $\pathset_{M,n,p_{0}}$ when SIRUS is run multiple times on the same dataset $\Dn$. On the other hand, $M$ should be small enough to avoid useless computations. Therefore, the growing of the forest is automatically stopped when $95\%$ of the selected rules would be shared by a new run of SIRUS on $\Dn$ in average, as it is possible to derive a simple stopping criterion based on the properties of the estimates $\hat{p}_{M,n}(\path)$---all the technical details are provided in Subsection \ref{subsec_parameters}.
A random forest is usually built with around $500$ trees, as the predictive accuracy cannot be significantly increased by adding more trees. SIRUS typically grows $10$ times more trees to obtain a robust rule extraction.

Besides, we insist that the quantile discretization is critical for the rule selection. The expected value of the rule importance is
\[
\E[\hat{p}_{M,n}(\path)] = \P(\path \in T(\Theta,\mathscr{D}_{n})),
\]
but without the discretization, the list of extracted paths from a random tree $T(\Theta, \Dn)$ takes values in an uncountable space when at least one component of $\bX$ is a continuous random variable, and therefore the above quantity is null, making the path selection procedure unstable with respect to data perturbation.

\paragraph{Rule post-treatment.}
By construction, there is some redundancy in the list of rules generated by the set of distinct paths $\pathset_{M,n,p_{0}}$.
The hyperrectangles associated with the paths extracted from a $\Theta$-random tree overlap, and so the corresponding rules are linearly dependent. Therefore a post-treatment to filter $\pathset_{M,n,p_{0}}$ is needed to remove redundancy and obtain a compact rule model. The general idea is straightforward: if the rule associated with the path $\path \in \pathset_{M,n,p_{0}}$ is a linear combination of rules associated with paths with a higher frequency in the forest, then $\path$ is removed from $\pathset_{M,n,p_{0}}$.

To illustrate the post-treatment, let the tree of Figure \ref{fig_paths} be the $\Theta_1$-random tree grown in the forest.
Since the paths of the first level of the tree, $\path_1$ and $\path_2$, always occur in the same trees, we have $\hat{p}_{M,n}(\path_1) = \hat{p}_{M,n}(\path_2)$. If we assume these quantities to be greater than $p_0$, then $\path_1$ and $\path_2$ both belong to $\pathset_{M,n,p_0}$. However, by construction, $\path_1$ and $\path_2$ are associated with the same rule, and we therefore enforce SIRUS to keep only $\path_1$ in $\pathset_{M,n,p_{0}}$. 
Each of the paths of the second level of the tree, $\path_3$, $\path_4$, $\path_5$, and $\path_6$, can occur in many different trees, and their associated $\hat{p}_{M,n}$ are distinct (except in very specific cases). Assume for example that $\hat{p}_{M,n}(\path_1) > \hat{p}_{M,n}(\path_4) > \hat{p}_{M,n}(\path_5) > \hat{p}_{M,n}(\path_3) > \hat{p}_{M,n}(\path_6) > p_0$. Since $\hat{g}_{n, \path_3}$ is a linear combination of $\hat{g}_{n, \path_4}$ and $\hat{g}_{n, \path_1}$, $\path_3$ is removed. Similarly $\path_6$ is redundant with $\path_1$ and $\path_5$, and it is therefore removed.
Finally, among the six paths of the tree, only $\path_1$, $\path_4$, and $\path_5$ are kept in the list $\pathset_{M,n,p_{0}}$.

\paragraph{Rule aggregation.} 
Now, the resulting small set of rules $\pathset_{M,n,p_{0}}$ is combined to form a simple, compact, and stable rule classification model.
We simply average the set of elementary rules 
$\{ \hat{g}_{n,\path}: \path \in \hat{\mathscr{P}}_{M,n,p_{0}} \}$
that have been selected in the first steps of SIRUS. The aggregated estimate $\hat{\eta}_{M,n,p_{0}}(\bx)$ 
of $\eta(\bx)$ is thus defined by
\begin{align} \label{eq_eta}
\hat{\eta}_{M,n,p_{0}}(\bx)
= \frac{1}{| \hat{\mathscr{P}}_{M,n,p_{0}}|} \sum_{\path \in \pathset_{M,n,p_{0}}} 
 \hat{g}_{n, \path}( \bx).
\end{align}
Finally, the classification procedure assigns class $1$ to an input $\bx$ if the aggregated estimate $\hat{\eta}_{M,n,p_{0}}(\bx)$ is above a given threshold, and class $0$ otherwise.
In the introduction, we presented an example of a list of $7$ rules for the Titanic dataset. In this case, for a new input $\bx$, $\hat{\eta}_{M,n,p_{0}}(\bx)$ is simply the average of the output probability of survival $p_s$ over the $7$ selected rules.

In past works on rule ensemble models, such as RuleFit \citep{friedman2008predictive} and Node harvest \citep{meinshausen2010node}, rules are also extracted from a tree ensemble and then combined together through a regularized linear model. 
In our case, it happens that the parameter $p_0$ alone is enough to control sparsity. Indeed, in our experiments, we observe that adding such linear model in the aggregation method hardly increases the accuracy and hardly reduces the size of the final rule set, while it can significantly reduce stability, add a set of coefficients that makes the model less straightforward to interpret, and requires more intensive computations.
We refer to the experiments in Appendix \ref{simu_mean_logit} for a comparison between $\hat{\eta}_{M,n,p_{0}}$ defined a as simple average (\ref{eq_eta}) versus a definition with a logistic regression.

\paragraph{Categorical and numerical discrete variables.}
For the sake of clarity, the description of SIRUS algorithm is limited to the case of numerical continuous variables. However, SIRUS can obviously handle numerical discrete and categorical data, as it is the case for random forests. On one hand, numerical discrete variables are left untouched since the number of possible split points is already finite, and the rule definition introduced for continuous variables also applies.
On the other hand, we naturally extend the rule definition for categorical variables to ``if $X^{(1)}$ is \textit{category 1 or 2} then \textit{response} else \textit{default response}''---see the Titanic dataset example in the introduction. 
Originally, categorical variables are efficiently handled in trees by transformation in ordered variables. Such ordering of categories is done with respect to the output mean for each category---see \citet{leo1984classification, friedman2001elements}, and we follow \texttt{ranger} implementation. Notice that trees are likely to often cut on categorical variables with a high number of categories, as highlighted in \citet{strobl2006bias}. Consequently, SIRUS is likely to output irrelevant rules associated to such categorical variables. Thus, it is best to discard categorical variables with a high number of categories, or transform them by regrouping categories or using one-hot-encoding before running SIRUS.
Finally, note that ordinal variables (e.g. $X^{(1)} \in \{$small, medium, big$\}$) are treated like categorical variables.

\paragraph{Stability.} \label{subsec_stability}
The three main properties to assess the interpretability of SIRUS are simplicity, stability, and predictivity, as already stated. On one hand, a measure of simplicity is naturally provided by the number of rules, and predictivity is given by the missclassification rate or the AUC. On the other hand, stability requires a more thorough discussion.
In the statistical learning theory, stability refers to the stability of predictions \citep[e.g.,][]{vapnik1998statistical}. In particular, \citet{rogers1978finite}, \citet{devroye1979distribution}, \citet{bousquet2002stability}, and \citet{poggio2004general} show that stability and predictive accuracy are closely connected. In our case, we are more concerned by the stability of the internal structure of the model, and, to our knowledge, no general definition exists. So, we state the following tentative definition:
a rule learning algorithm is stable if two independent estimations based on two independent samples result in two similar lists of rules.
Thus, given a new sample $\mathscr{D}_{n}'$ independent of $\mathscr{D}_{n}$, we define $\hat{p}'_{M,n}(\path)$ and the corresponding set of paths $\smash{\pathset_{M,n,p_{0}}'}$ 
based on a modified random forest drawn with a parameter $\Theta'$ independent of $\Theta$. 
Then, we measure the stability of SIRUS by the proportion of rules shared by the two sets $\smash{\pathset_{M,n,p_{0}}}$ and $\smash{\pathset_{M,n,p_{0}}'}$, selected over these two runs of SIRUS on independent samples. We take advantage of a dissimilarity measure between two sets, the so-called Dice-Sorensen index, often used to assess the stability of variable selection methods \citep{chao2006abundance,zucknick2008comparing,boulesteix2009stability,he2010stable,alelyani2011dilemma}. This index is defined by
\begin{align}
\hat{S}_{M,n,p_{0}}=\frac{2\big|\pathset_{M,n,p_{0}}\cap\pathset_{M,n,p_{0}}'\big|}{\big|\pathset_{M,n,p_{0}}\big|+\big|\pathset_{M,n,p_{0}}'\big|} 
\label{stability_emp_version}
\end{align}
with the convention $0/0 = 1$.
This is a measure of stability taking values between $0$ and $1$: if the intersection between $\smash{\pathset_{M,n,p_{0}}}$ and
$\smash{\pathset_{M,n,p_{0}}'}$ is empty, then $\smash{\hat{S}_{M,n,p_{0}} = 0}$, while if $\smash{\pathset_{M,n,p_{0}} = \pathset_{M,n,p_{0}}'}$, then $\smash{\hat{S}_{M,n,p_{0}} = 1}$. 
Notice that it is possible to use other metrics to assess the distance between two finite sets \citep{zucknick2008comparing}: the Jaccard Index is another popular example.
Although the stability values slightly vary with a different definition, both the asymptotic stability of SIRUS---see Section \ref{sec_theory}---and the empirical stability comparisons between algorithms---see Section \ref{sec_xp}---are insensitive to the stability metric choice.

\section{Theoretical Analysis of Stability} \label{sec_theory}
Among the three minimum requirements for interpretability defined in Section \ref{sec_intro}, simplicity and predictivity are quite easily met for rule models \citep{cohen1999simple, meinshausen2010node,  letham2015interpretable}. 
On the other hand, as \cite{letham2015interpretable} recall, building a stable rule ensemble is challenging. Therefore the main goal of this section is to prove the asymptotic stability of SIRUS, i.e., provided that the sample size is large enough, SIRUS systematically outputs the same list of rules when run multiple times with independent samples. On the other hand, we also argue that existing tree-based rule algorithms are unstable by design.

In order to show the asymptotic stability of SIRUS, we first need to introduce formal definitions of the mathematical elements involved in the empirical algorithm. We additionally define the theoretical counterpart of SIRUS, an abstract procedure which is not based on the sample $\Dn$, but only on the unknown distribution $\P_{\bX,Y}$. Next, we will prove the stochastic convergence of SIRUS towards its theoretical counterpart. This means that the list of selected rules does not depend on the training data $\Dn$, but only on $\P_{\bX,Y}$, provided that the sample size is large enough. Therefore, the same list of rules is output when SIRUS is run multiple times on independent samples. This mathematical analysis highlights that the remarkable stable behavior of SIRUS in practice has theoretical groundings, and that the discretization of the cut values with the quantiles, as well as using random forests, are the cornerstones to stabilize rule models extracted from tree ensembles.

\paragraph{Empirical algorithm.}
First, we define the empirical CART-splitting criterion used to find the optimal split at each node of each tree of the forest. In our context of binary classification where the output $Y \in \{0, 1\}$, maximizing the so-called empirical CART-splitting criterion is equivalent to maximizing the criterion based on Gini impurity \citep[see, e.g.,][]{biau2016random}. More precisely, at node $H$ and for a cut performed along the $j$-th coordinate at the empirical $r$-th $q$-quantile $\hat{q}_{n,r}^{(j)}$, this criterion reads
\begin{align} \label{eq_emp_criterion}
\begin{split}
L_{n}(H, \hat{q}_{n,r}^{(j)}) & \stackrel{\tiny{\rm def}}{=}\frac{1}{N_{n}(H)}\sum_{i=1}^{n}(Y_{i}-\overline{Y}_{H})^{2}\mathds{1}_{\bX_{i}\in H}\\
& \quad -\frac{1}{N_{n}(H)}\sum_{i=1}^{n}\big(Y_{i}-\overline{Y}_{H_{L}}\mathds{1}_{X_{i}^{(j)}<\hat{q}_{n,r}^{(j)}}
-\overline{Y}_{H_{R}}\mathds{1}_{X_{i}^{(j)}\geq \hat{q}_{n,r}^{(j)}}\big)^{2}\mathds{1}_{\bX_{i}\in H},
\end{split}
\end{align}
where $\overline{Y}_{H}$ is the average of the $Y_{i}$'s such that $\bX_i \in H$, $N_{n}(H)$
is the number of data points $\bX_i$ falling into $H$, 
\[
H_{L}\stackrel{\tiny{\rm def}}{=}\{ \bx \in H: \bx^{(j)}<\hat{q}_{n,r}^{(j)}\}, \quad H_{R}\stackrel{\tiny{\rm def}}{=}\{ \bx \in H: \bx^{(j)}\geq
\hat{q}_{n,r}^{(j)}\},
\]
and for $r \in \{1,\hdots,q-1\}$ the empirical $r$-th $q$-quantile of $\{X_1^{(j)}, \hdots, X_n^{(j)}\}$ is defined by
\begin{align} \label{eq_emp_quantile}
\hat{q}_{n,r}^{(j)}=\inf\big\{ x\in\mathds{R}:\frac{1}{n}\stackrel[i=1]{n}{\sum}\mathds{1}_{X_{i}^{(j)}\leq x}\geq\frac{r}{q}\big\}.
\end{align}
Note that, for the ease of reading,  (\ref{eq_emp_criterion}) is defined for a tree built with the entire dataset $\Dn$ without resampling. As it is often the case in the theoretical analysis of random forests, we assume throughout this section that the subsampling of $a_n$ observations to build each tree is done without replacement to alleviate the mathematical analysis.

Recall that the rule selection is based on the probability $p_n(\mathscr{P})$ that a $\Theta$-random tree of the forest
contains a particular path $\mathscr{P} \in \Pi$, that is,
\[
p_n(\mathscr{P})=\mathbb{P}(\mathscr{P}\in T(\Theta,\mathscr{D}_{n})|\mathscr{D}_{n}),
\]
and that the Monte-Carlo estimate $\hat{p}_{M,n}(\mathscr P)$ of $p_n(\mathscr{P})$ is directly computed using the random forest, and takes the form
\[
\hat{p}_{M,n}(\path)=\frac{1}{M}\sum_{\ell=1}^{M}\mathds{1}_{\path\in T(\Theta_{\ell},\mathscr{D}_{n})}.
\]
Clearly, $\hat{p}_{M,n}(\mathscr{P})$ is a good estimate of 
$p_n(\mathscr{P})$ when $M$ is large since, by the law of large numbers, conditional on $\mathscr{D}_{n}$, 
\[
\lim\limits_{M \to \infty} \hat{p}_{M,n}(\mathscr{P}) =p_n(\mathscr{P}) \quad \textrm{a.s.}
\]
We also see that $\hat{p}_{M,n}(\mathscr{P})$ is unbiased since $\mathds{E}[\hat{p}_{M,n}(\path)| \mathscr{D}_{n}] = p_n(\mathscr{P}).$

\paragraph{Theoretical algorithm.}
Next, we define all theoretical counterparts of the empirical quantities involved in SIRUS, which do not depend on $\mathscr{D}_{n}$ but only on the unknown distribution  $\P_{\bX,Y}$ of $(\bX,Y)$. For a given integer $q \geq 2$ and $r\in \{1, \hdots, q - 1\}$, the theoretical $q$-quantiles are defined by 
\[
q_{r}^{\star(j)}=\textrm{inf}\big\{ x\in\mathds{R}:\mathds{P}(X^{(j)}\leq x)\geq\frac{r}{q}\big\},
\]
i.e., the population version of $\hat q_{n,r}^{(j)}$ defined in (\ref{eq_emp_quantile}). Similarly, for a given hyperrectangle $H\subseteq\mathds{R}^{p}$, we let the theoretical CART-splitting criterion be
\begin{align*}
L^{\star}(H,q_{r}^{\star(j)})& =\mathbb{V}[Y|\bX\in H]\\ & \quad -\mathds{P}(X^{(j)}<q_{r}^{\star(j)}|\bX\in H)\times
 \mathbb{V}[Y|X^{(j)}<q_{r}^{\star(j)},\bX\in H]\\ & \quad -\mathds{P}(X^{(j)}\geq q_{r}^{\star(j)}|\bX\in H)\times
 \mathbb{V}[Y|X^{(j)}\geq q_{r}^{\star(j)},\bX\in H].
\end{align*}
Based on this criterion, we denote by $T^{\star}(\Theta)$ the list of all paths contained in the theoretical tree built with randomness $\Theta$, 
where splits are chosen to maximize the theoretical criterion $L^{\star}$ instead of the empirical one $L_{n}$, defined in (\ref{eq_emp_criterion}). We stress again that the list $T^{\star}(\Theta)$ does not depend upon $\mathscr{D}_{n}$ but only upon the unknown distribution of $(\bX,Y)$.
Next, we let $p^{\star}(\path)$ be the theoretical counterpart of $p_n(\path)$, that is 
\[
p^{\star}(\path)=\mathds{P}(\path\in T^{\star}(\Theta)),
\]
and finally define the theoretical set of selected paths $\mathscr{P}^{\star}_{p_0}$ by 
$\{ \path \in \Pi : p^{\star}(\path) > p_0\}$ (with the same post-treatment as for the empirical procedure---see Section \ref{sec_algo}).
Notice that, in the case where multiple splits have the same value of the theoretical CART-splitting criterion, one is randomly selected. 

\paragraph{Consistency of the path selection.}
The construction of the rule ensemble model essentially relies on the path selection and on the estimates $\hat{p}_{M,n}(\path)$,
$\path \in \Pi$. Therefore, our theoretical analysis first focuses on the asymptotic properties of those estimates in Theorem \ref{theorem_consistency}.
Our consistency results hold under conditions on the subsampling rate $a_{n}$ and the number of trees $M_n$, together with some assumptions on the distribution of the random vector $\bX$. They are given below.
\begin{enumerate}
\item[(A1)] The subsampling rate $a_n$ satisfies $\lim\limits_{n \to \infty} a_{n}=\infty$ and $\lim\limits_{n \to \infty} \frac{a_{n}}{n}=0$.
\item[(A2)] The number of trees $M_n$ satisfies $\lim\limits_{n \to \infty} M_n = \infty$.
\item[(A3)] $\bX$ has a strictly positive density $f$ with respect to the Lebesgue measure. Furthermore, for all $j\in \{1,\hdots,p\}$,
 the marginal density $f^{(j)}$ of $X^{(j)}$ is continuous, bounded, and strictly positive.
\end{enumerate}
We can now state the consistency of the occurrence frequency of each possible path $\path \in \Pi$ in the modified random forest.
\begin{theorem}
\label{theorem_consistency}
If Assumptions (A1)-(A3) are satisfied, then, for all $\path \in \Pi$, we have
\[
\lim\limits_{n \to \infty} \hat{p}_{M_{n},n}(\path) = p^{\star}(\path) \quad \textrm{in probability.}
\]
\end{theorem}

\paragraph{Stability.}
The only source of randomness in the selection of the rules lies in the estimates $\hat{p}_{M_{n},n}(\path)$.
Since Theorem \ref{theorem_consistency} states the consistency of such an estimation, the path selection consistency follows, for all  threshold values $p_0$ that do not belong to the finite set $\setpstar = \{ p^{\star}(\path): \path \in \Pi \}$ of all theoretical probabilities of appearance for each path $\path$. Indeed, if $p_0 =  p^{\star}(\path)$ for some $\path \in \Pi$, 
then $\P(\hat{p}_{M_n,n}(\path)>p_{0})$ does not necessarily converge to $0$ and the path selection can be inconsistent.
Then, we can deduce that SIRUS is asymptotically stable in the following Corollary \ref{corollary_stability}.

\begin{corollary}
\label{corollary_stability}
\hspace{0.2cm}Assume that Assumptions (A1)-(A3) are satisfied. Then, provided $p_{0}\in[0,1] \setminus \mathcal{U}^{\star}$, we have
\[
\lim \limits_{n \to \infty} \P ( \hat{\mathscr{P}}_{M_n,n,p_{0}} = \mathscr{P}^{\star}_{p_0} ) = 1,
\]
and then
\[
\lim \limits_{n \to \infty}  \hat{S}_{M_{n},n,p_{0}} = 1 \quad  \textrm{in probability.}
\]
\end{corollary}

\paragraph{Competitors.}
As we will discuss further in the experimental Section \ref{sec_xp}, CART, C5.0, RuleFit, and Node harvest are top competitors of SIRUS, which are also based on rule extraction from trees. However, these algorithms do not include a pre-processing step of discretization, which makes them unstable by design. To see this, we first adapt the definition of an extracted path without discretization as $\mathscr{P} =\{ (j_{k},z_{k},s_{k}), \, k = 1,\hdots,d\}$, where $z_k \in \R$ is now the cutting value of the $k$-th split.
For any rule algorithm, we also define $\hat{S}_{M,n}$ as the proportion of rules shared between the output rule lists over two runs with two independent samples. Note that $M=1$ for CART and C5.0, and as already mentioned, it is possible to define a rule algorithm from CART, by extracting its nodes, as in C5.0. Thus, we obtain that for any tree-based rule algorithm, $\hat{S}_{M,n} = 0$ almost surely.
Indeed, since the input $\bX$ takes continuous values (Assumption (A3)) and decision trees can cut at the middle of two observations in all directions, the probability that a cutting value from the tree built with $\Dn$ and one from the tree built with $\Dn'$ are equal is null.

However, recall that in the experiments, we include a pre-processing discretization step to stabilize competitors and enable fair comparisons. With this modification, they reach a value of $\hat{S}_{M,n} > 0$, but still not in par with SIRUS. This shows that the high stability improvement of SIRUS does not only come from the discretization, but mainly from the rule selection procedure, based on the probability of the rule occurrence in a random tree.

\paragraph{Proofs.}
The proof of Theorem \ref{theorem_consistency} is to be found in Appendix \ref{sec_theorem_consistency}. It is however interesting to give a sketch of the proof here. Corollary \ref{corollary_stability} is a direct consequence of Theorem \ref{theorem_consistency}, the full proof follows.

\begin{proof}[Sketch of proof of Theorem \ref{theorem_consistency}]
The consistency is obtained by showing that $\hat{p}_{M_{n},n}(\path)$ is asymptotically unbiased with a null variance. 
The result for the variance is quite straightforward since the variance of $\hat{p}_{M_{n},n}(\path)$ can be broken into two terms:
the variance generated by the Monte-Carlo randomization, which goes to $0$ as the number of trees increases (Assumption (A2)), and the variance of $p_n(\path)$. Following \cite{mentchquantifying2016}, since $p_n(\path)$ is a bagged estimate it can be seen as an 
infinite-order U-statistic, and a classic bound on the variance of U-statistics gives that $\mathbb{V}[p_n(\path)]$ converges to $0$ if $\lim_{n \to \infty} \frac{a_n}{n} = 0$, which is true by Assumption (A1). Next, proving that $\hat{p}_{M_{n},n}(\path)$ is asymptotically unbiased requires to dive into the internal mechanisms of the random forest algorithm. To do this, we have to show that the CART-splitting criterion is consistent (Lemma 3) and asymptotically normal (Lemma 4) when cuts are limited to empirical quantiles (estimated on the same dataset) and the number of trees grows with $n$.
When cuts are performed on the theoretical quantiles, the law of large numbers and the central limit theorem can be directly applied, so that the proof of Lemmas 3 and 4 boils down to showing that the difference between  the empirical CART-splitting criterion evaluated at empirical and theoretical quantiles converges to $0$ in probability fast enough. This is done in Lemma 2 thanks to Assumption (A3).
\end{proof}

\begin{proof}[Proof of Corollary \ref{corollary_stability}]
The first result is a consequence of Theorem \ref{theorem_consistency} since
\begin{align*}
\P \big( & \hat{\mathscr{P}}_{M_n,n,p_{0}} \neq \mathscr{P}^{\star}_{p_0} \big) \leq \sum_{\path \in \Pi}
\P ( \hat{p}_{M_{n},n}(\path) > p_0) \mathds{1}_{p^{\star}(\path) \leq p_0}
+ \P ( \hat{p}_{M_{n},n}(\path) \leq p_0) \mathds{1}_{p^{\star}(\path) > p_0}.
\end{align*}
Next, we have
\[
\hat{S}_{M_{n},n,p_{0}}=\frac{2\underset{\path \in \Pi}{\sum}\mathds{1}_{\hat{p}_{M_{n},n}(\path)>p_{0}
\cap\hat{p}_{M_{n},n}'(\path)>p_{0}}}{\underset{\path\in\Pi}{\sum}
\mathds{1}_{\hat{p}_{M_{n},n}(\path)>p_{0}}+\mathds{1}_{\hat{p}_{M_{n},n}'(\path)>p_{0}}}.
\]
Since $p_{0} \notin \mathcal{U}^{\star}$, we deduce from Theorem \ref{theorem_consistency} and the continuous mapping theorem that,
for all $\path \in \Pi$,
\[
\lim \limits_{n \to \infty} \mathds{1}_{\hat{p}_{M_{n},n}(\path)>p_{0}}
= \mathds{1}_{p^{\star}(\path)>p_{0}} \quad \textrm{in probability}.
\]
Therefore, $\lim \limits_{n \to \infty} \hat{S}_{M_{n},n,p_{0}} = 1$ in probability.
\end{proof}

\section{Experiments} \label{sec_xp}

We begin this section by providing overall experimental settings. Next, we focus on a case study to illustrate SIRUS with an industrial process example: the semi-conductor manufacturing process SECOM data \citep{Dua:2019}. In particular, it shows the excellent performance of SIRUS on real data in a noisy and high-dimensional setting. 
In Subsection \ref{subsec_competitors}, we use $19$ UCI datasets \citep{Dua:2019} to perform extensive comparisons between SIRUS and its main competitors. We show that SIRUS produces much more stable rule lists, while preserving a predictive accuracy and computational complexity comparable to the top competitors.
Finally, in Subsection \ref{subsec_parameters}, we detail the tuning procedure of the single hyperparameter $p_0$, along with a thorough discussion on the design of SIRUS. In particular, the cut limitations to the quantiles and the number of constraints in the selected rules are analyzed, and we also provide the stopping criterion for the number of trees.

\subsection{Experiment Description}

\paragraph{Performance metrics.}
We first introduce relevant metrics to assess the three interpretability properties in the experiments.
By definition, the size (i.e., the simplicity) of the rule ensemble is the number of selected rules, i.e., $|\pathset_{M,n,p_{0}}|$.
To measure the error, 1-AUC is used and estimated by $10$-fold cross-validation (repeated $10$ times for robustness and standard deviation estimates).
With respect to stability, an independent dataset is not available for real data to compute $\hat{S}_{M,n,p_{0}}$ as defined in (\ref{stability_emp_version}) in the Section \ref{sec_algo}. Nonetheless, we can take advantage of the cross-validation process to compute a stability metric: the proportion of rules shared by two models built during the cross-validation, averaged over all possible pairs \citep{guidotti2019stability}.

\paragraph{Datasets.}
We have conducted experiments on the SECOM data, as well as $19$ diverse public datasets from the UCI repository (\citealp[][]{Dua:2019}; data is described in Table \ref{table_datasets}).
\begin{table}
\setlength{\tabcolsep}{3pt}
\centering
\begin{tabular}{|c | c | c | c|}
  \hline \hline
  \textbf{Dataset} & \textbf{Sample size} & \begin{tabular}{c}\textbf{Total number} \\ \textbf{of variables} \end{tabular} &
 \begin{tabular}{c}\textbf{Number of} \\ \textbf{categorical variables} \end{tabular} \\
  \hline
    Authentification & 1372 & 4 & 0 \\
    Breast Wisconsin & 699 & 9 & 9 \\
    Credit Approval	& 690 & 15 & 9 \\
    Credit German & 1000 & 20 & 13 \\
    Diabetes & 768 & 8 & 0 \\	
    Haberman & 306 & 3 & 0 \\	
    Heart C2 & 303 & 13 & 7 \\	
    Heart H2 & 294 & 13  & 7 \\	
    Heart Statlog & 270 & 13 & 3 \\	
    Hepatitis & 155 & 19 & 0 \\
    Ionosphere & 351 & 33 & 0 \\
    Kr vs Kp & 3196 & 36 & 36 \\
    Liver Disorders & 345 & 6 & 0 \\
    Mushrooms & 8124 & 21 & 21 \\
    SECOM & 1567 & 590 & 0 \\
    Sonar & 208 & 60 & 0 \\
    Spambase & 4601 & 57 & 0 \\
    Titanic & 887 & 6 & 1 \\
    Vote & 435 & 16 & 16 \\
    Wilt & 4339 & 5 & 0 \\
  \hline \hline
\end{tabular}
\vspace*{1.5mm}
\caption{\small{Description of UCI datasets}}
\label{table_datasets}
\end{table}
These experiments aim at illustrating the good behavior of SIRUS over its competitors in various settings. 
To compare stability of the different methods, data is discretized using the 10-empirical quantiles for each continuous variable and the same stability metric is used for all algorithm comparisons. For simplicity and predictivity metrics, we do not apply this pre-processing step of discretization, unless the algorithm only handles categorical data.

\paragraph{Competitors.}
For decision trees, we run both CART and C5.0, and trees are pruned to maximize their performance. Notice that, to enable simplicity and stability comparisons for CART, a list of rules is extracted from its nodes, as it is originally possible for C5.0.
For rule algorithms based on greedy heuristics, we evalute RIPPER, PART, and FOIL. 
Next, for rule algorithms based on tree ensembles, we evaluate RuleFit and Node harvest. Note that categorical features are transformed in multiple binary variables as it is required by the two software implementations, and RuleFit is limited to rule predictors. 
For RuleFit, the lasso penalty is tuned by cross-validation as defined in \citet{friedman2008predictive}. As advertised in \citet{meinshausen2010node}, Node harvest does not require parameter tuning by default, but it is also possible to add a regularization term to reduce the model size. We use the same tuning procedure as for SIRUS to maximize accuracy with the smallest possible model---see Subsection \ref{subsec_parameters}.
Finally, for rule algorithms based on frequent pattern mining, we run the experiments for CBA and BRL. Note that we use default settings for BRL, since modifying its parameters does not significantly improve accuracy and can hurt stability.
We use available R implementations: \texttt{rpart} \citep[CART]{therneau2018package}, \texttt{C50} \citep[C5.0]{kuhn2014c50}, \texttt{RWeka} \citep[RIPPER, PART]{hornik2020package}, \texttt{arulesCBA} \citep[FOIL, CBA]{johnson1arulescba}, \texttt{pre} \citep[RuleFit]{fokkema2017pre}, \texttt{nodeHarvest} \citep[Node harvest]{meinshausen2015package}, and \texttt{sbrl} \citep[BRL]{yang2017scalable}.
We also use our \texttt{R}/\texttt{C++} software implementation \texttt{sirus} \citep{Rsirus} (available from \texttt{CRAN}), adapted from \texttt{ranger}, a fast random forest implementation \citep{wright2017ranger}.
Besides, notice that for SIRUS experiments, we use the default settings of random forests well known for their excellent behavior, in particular $\texttt{mtry} = \lfloor \frac{p}{3} \rfloor$. We set $q = 10$ quantiles and tune $p_0$ as specified in Subsection \ref{subsec_parameters}.

\subsection{Case Study: Manufacturing Process Data}
SIRUS is run on a real manufacturing process of semi-conductors, the SECOM dataset \citep{Dua:2019}. Data is collected from sensors and process measurement points to monitor the production line, resulting in $590$ numeric variables.
Each of the $1567$ data points represents a single production entity associated with a pass or fail output ($0/1$) for in-house line testing.
As it is often the case for a production process, the dataset is unbalanced and contains $104$ fails, i.e., a failure rate $p_f$ of $6.6$\%.
We proceed to a simple pre-processing of the data: missing values (about $5$\% of the total) are replaced by the median.

Figure \ref{figure_semicon_auc} shows predictivity versus the number of rules when $p_0$ varies, with the optimal $p_0$ displayed. Notice that the relation between $p_0$ and the number of rules is monotone by construction, but also highly nonlinear. Therefore, we use the number of rules for the x-axis of Figure \ref{figure_semicon_auc} to improve readability.
The 1-AUC value is $0.30$ for SIRUS (for the optimal $p_0=0.04$), $0.29$ for Breiman's random forests, and $0.48$ for a pruned CART tree. 
Thus, in that case, CART tree predicts no better than the random classifier, whereas SIRUS has a similar accuracy to random forests.
The final model has $6$ rules and a stability of $0.72$, i.e., in average $4$ to $5$ rules are shared by $2$ models built in a 10-fold cross-validation process, simulating data perturbation. By comparison, Node harvest outputs $36$ rules with a value of $0.32$ for 1-AUC.
\begin{figure}
\begin{center}
\includegraphics[height=10cm,width=10cm]{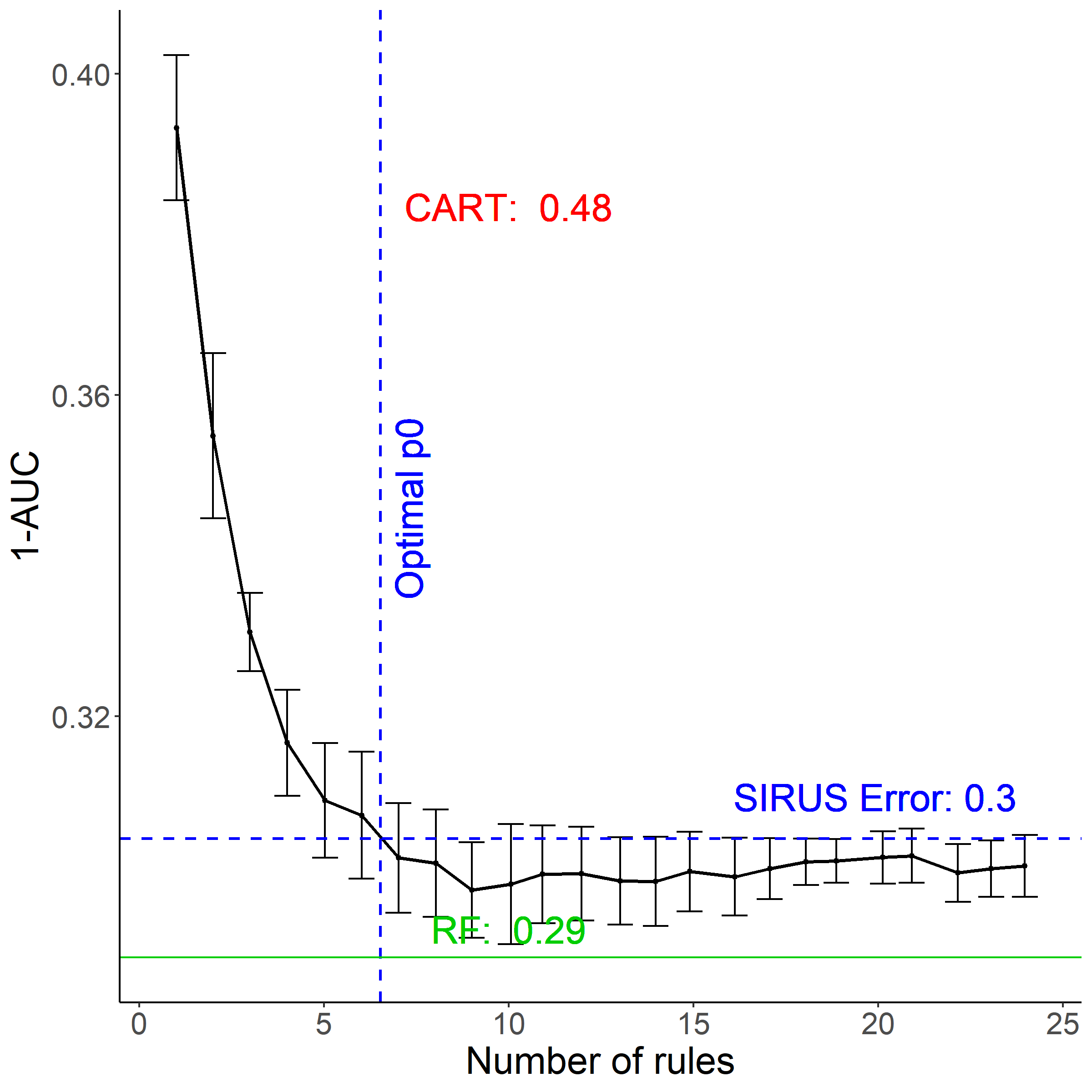}
\caption{\small{For the SECOM dataset, error (1-AUC) versus the number of rules when $p_0$ varies, estimated via 10-fold cross-validation (averaged over $10$ repetitions of the cross-validation).} Errors for CART and random forests are reported for comparisons.}
\label{figure_semicon_auc}
\end{center}
\end{figure}

Finally, the output of SIRUS may be displayed in the simple and interpretable form of Figure \ref{figure_semicon_R_print}, the output in the \texttt{R} console of the package \texttt{sirus} for the SECOM data.
\begin{figure}
\begin{center}
\setlength{\fboxrule}{1pt}
\fbox{\begin{minipage}{0.90\textwidth}
\includegraphics[scale = 1]{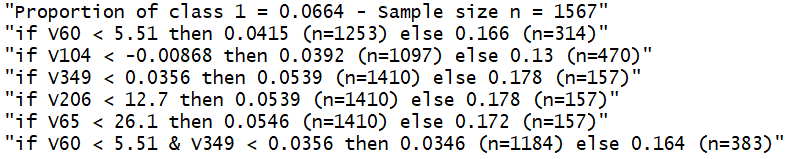}
\end{minipage}}
\caption{\small{List of rules output by our software \texttt{sirus} in the R console for the SECOM dataset.}}
\label{figure_semicon_R_print}
\end{center}
\end{figure}
Such a rule model enables to catch immediately how the most relevant variables impact failures. Among the $590$ variables, $5$ are enough to build a model as predictive as random forests, and such a selection is robust. Other rules alone may also be informative, but they do not add additional information to the model, since predictive accuracy is already minimal with the $6$ selected rules. Then, production engineers should first focus on those $6$ rules to investigate an improved setting of the production process. We insist that the stability of the output rule list is critical in practice. Indeed, the algorithm may be run multiple times during the analysis, eventually with an additional small new batch of data. The output rule list should be quite insensitive to such perturbation: domain experts are skeptical of unstable results, which are the symptoms of a partial and arbitrary modelling of the true phenomenon. SIRUS is stable, but it is not the case for decision trees or existing rule algorithms, as we show in the next subsection and illustrate in Appendix \ref{appendix_stability}.

\subsection{Improvement over Competitors} \label{subsec_competitors}

Overall, we observe that SIRUS provides a high improvement of stability compared to state-of-the-art rule algorithms, while preserving the other properties.
For the top competitors, experimental results are gathered in Table \ref{table_size} for model size, Table \ref{table_stability} for stability, and Table \ref{table_auc} for predictive accuracy. Experiments for additional competitors are provided in Appendix \ref{appendix_xp} in Tables \ref{table_size_appendix}, \ref{table_stability_appendix} and \ref{table_auc_appendix}. Standard deviations are made negligible by averaging metrics over $10$ repetitions of the cross-validation and are not displayed in the tables to increase readability.

Figure \ref{figure_credit_german} provides an example for the dataset ``Credit German'' of the dependence between predictivity and the number of rules when $p_0$ varies. In that case, the minimum of 1-AUC is about $0.25$ for SIRUS, $0.20$ for Breiman's forests, and $0.29$ for CART tree. 
For the chosen $p_0$, SIRUS returns a compact set of $22$ rules and its stability is $0.66$.
Figure \ref{figure_heart_statlog} provides another example of the good practical performance of SIRUS with the ``Heart Statlog'' dataset. Here, the predictivity of random forests is reached with $16$ rules, with a stability of $0.83$, i.e., about $13$ rules are consistent between two different models built in a $10$-fold cross-validation. 
Thus, the final models are simple, quite robust to data perturbation, and have a predictive accuracy close to random forests.
\begin{figure}
\begin{center}
\includegraphics[height=7cm,width=7cm]{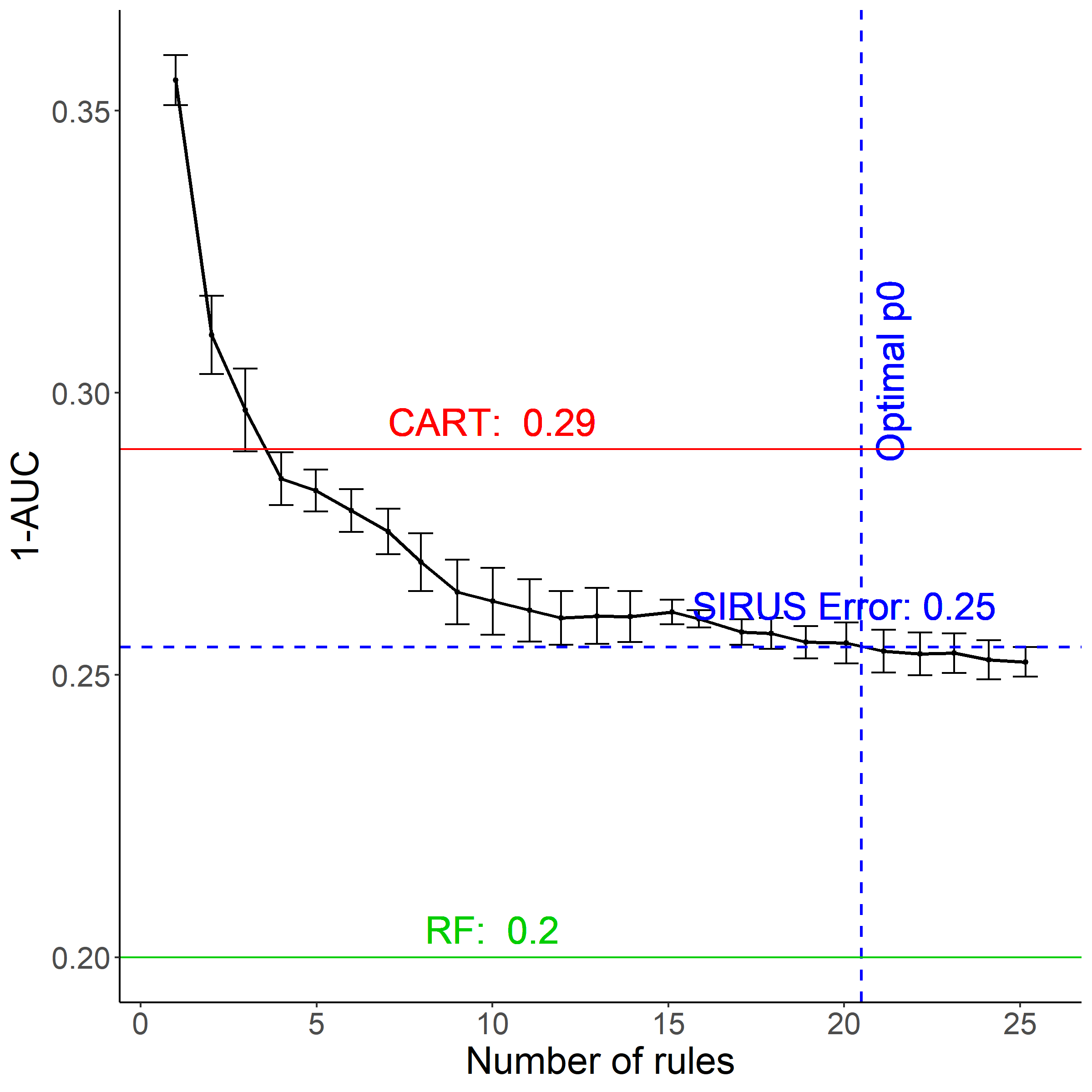}
\includegraphics[height=7cm,width=7cm]{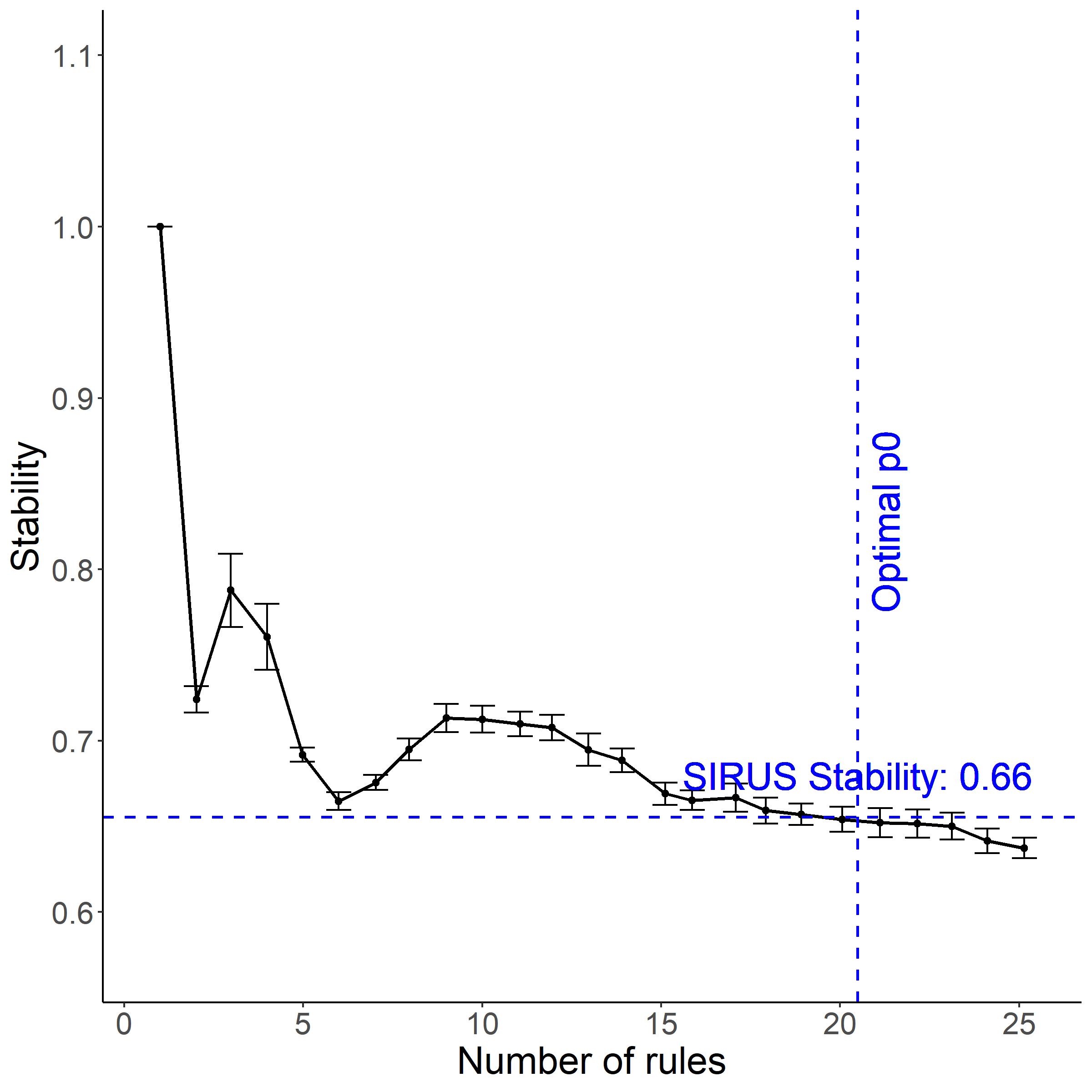}
\caption{\small{For the UCI dataset ``Credit German'', 1-AUC (on the left) and stability (on the right) versus the number of rules when $p_0$ varies, estimated via 10-fold cross-validation (results are averaged over $10$ repetitions).}}
\label{figure_credit_german}
\end{center}
\end{figure}
\begin{figure}
\begin{center}
\includegraphics[height=7cm,width=7cm]{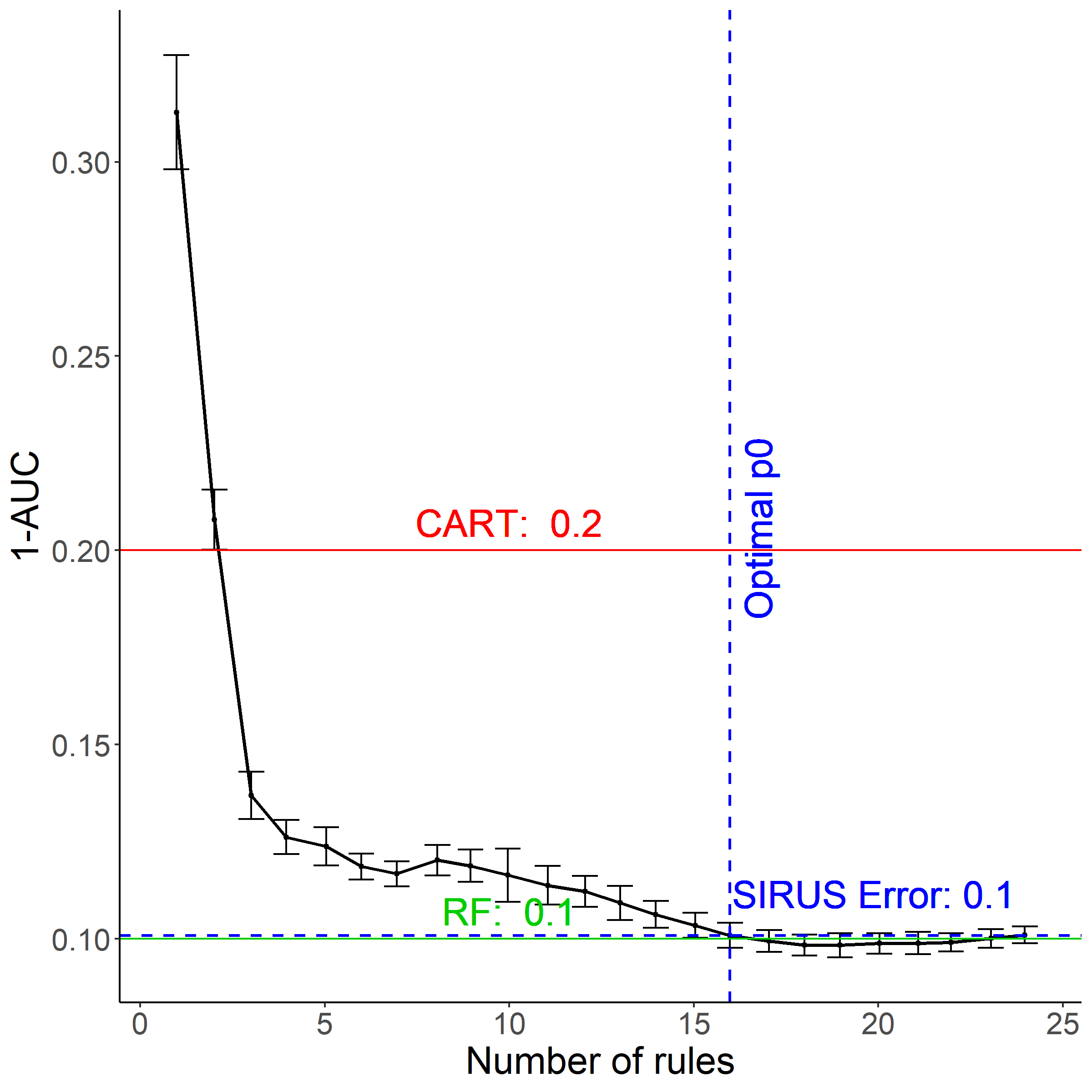}
\includegraphics[height=7cm,width=7cm]{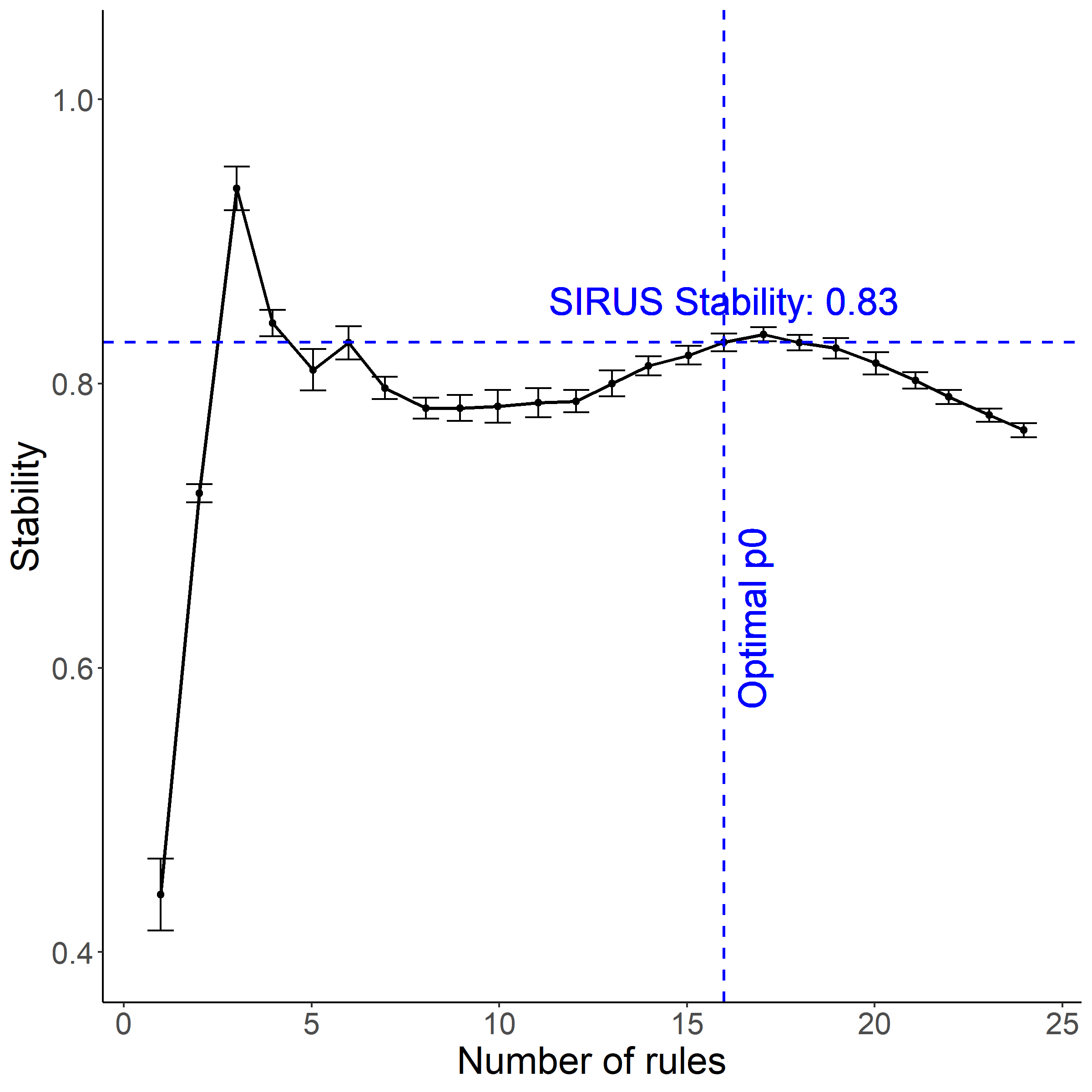}
\caption{\small{For the UCI dataset ``Heart Statlog'', 1-AUC (on the left) and stability (on the right) versus the number of rules when $p_0$ varies, estimated via 10-fold cross-validation (results are averaged over $10$ repetitions).}}
\label{figure_heart_statlog}
\end{center}
\end{figure}

We can draw the following conclusions from the experimental comparisons with competitors, displayed in Tables \ref{table_size}, \ref{table_stability}, and \ref{table_auc}.
SIRUS produces more stable and predictive rule lists than decision trees, for a comparable simplicity, but at the price of a higher computational complexity since many trees are grown.
SIRUS produces much more stable and shorter rule lists than RuleFit and Node harvest, for a comparable accuracy and computational complexity.
Classical rule algorithms exhibit similar properties as decision trees: a smaller computational complexity, but a high instability and a reduced predictivity.
Finally, algorithms based on frequent pattern mining exhibit quite good stability properties, higher than for the other types of competitors. On the other hand, their predictive accuracy is worse than decision trees. Experiments in Tables \ref{table_size}, \ref{table_stability}, and \ref{table_auc} show that SIRUS exhibits a high stability and predictivity improvement over these methods. Besides, simplicity varies across algorithms: CBA produces much longer rule lists than SIRUS, whereas BRL generates shorter models.
\begin{table}
\setlength{\tabcolsep}{2pt}
\centering
\begin{tabular}{|c | c | c | c | c | c | c | c|}
  \hline  \hline
   & \begin{tabular}{c}\textbf{Decision} \\ \textbf{tree}\end{tabular} & \begin{tabular}{c}\textbf{Classical} \\ \textbf{rule learning}\end{tabular} & \multicolumn{2}{|c|}{\begin{tabular}{c}\textbf{Frequent} \\ \textbf{pattern mining}\end{tabular}} & \multicolumn{3}{|c|}{\textbf{Tree ensemble}} \\
  \hline
  \textbf{Dataset} & \textbf{CART} & \textbf{RIPPER} & \textbf{CBA} & \textbf{BRL} & \textbf{RuleFit} & \begin{tabular}{c}\textbf{Node} \\ \textbf{harvest}\end{tabular} & \textbf{SIRUS} \\
  \hline
    Authentification & 21 & 7 & 7 & 17 & 49 & 30 & 13 \\
    Breast Wisconsin & 7 & 12 & 24 & 7 & 24 & 32 & 24 \\
    Credit Approval & 5 & 4 & 55 & 4 & 15 & 27 & 16 \\
    Credit German & 18 & 3 & 69 & 4 & 33 & 33 & 20 \\
    Diabetes & 13 & 3 & 17 & 6 & 26 & 31 & 8 \\
    Haberman & 2 & 1 & 2 & 2 & 3 & 17 & 5\\
    Heart C2 & 10 & 3 & 34 & 4 & 23 & 36 & 20 \\
    Heart H2 & 5 & 2 & 29 & 3 & 12 & 24 & 12 \\	
    Heart Statlog & 10 & 3 & 27 & 4 & 22 & 35 & 16 \\
    Hepatitis & 2 & 2 & 14 & 2 & 8 & 14 & 12 \\
    Ionosphere & 4 & 4 & 38 & 4 & 20 & 35 & 15 \\
    Kr vs Kp & 16 & 15 & 29 & 9 & 18 & 13 & 24 \\
    Liver Disorders & 15 & 3 & 2 & 3 & 19 & 33 & 17 \\
    Mushrooms & 4 & 8 & 25 & 11 & 10 & 22 & 23 \\
    Sonar & 6 & 4 & 33 & 2 & 32 & 83 & 19 \\
    Spambase & 14 & 16 & 126 & 16 & 68 & 60 & 21 \\
    Titanic & 13 & 4 & 4 & 3 & 19 & 23 & 6 \\
    Vote & 2 & 2 & 25 & NA & 12 & 10 & 7 \\
    Wilt & 9 & 5 & 3 & 10 & 31 & 19 & 24 \\
  \hline \hline
 \end{tabular}
\vspace*{1.5mm}
\caption{Mean model size over a $10$-fold cross-validation for UCI datasets. 
Results are averaged over $10$ repetitions of the cross-validation.}
\label{table_size}
\end{table}
\begin{table} 
\setlength{\tabcolsep}{2pt}
\centering
\begin{tabular}{| c | c | c | c | c | c | c | c |}
  \hline  \hline
  & \begin{tabular}{c}\textbf{Decision} \\ \textbf{tree}\end{tabular} & \begin{tabular}{c}\textbf{Classical} \\ \textbf{rule learning}\end{tabular} & \multicolumn{2}{|c|}{\begin{tabular}{c}\textbf{Frequent} \\ \textbf{pattern mining}\end{tabular}} & \multicolumn{3}{|c|}{\textbf{Tree ensemble}} \\
  \hline
  \textbf{Dataset} & \textbf{CART} & \textbf{RIPPER} & \textbf{CBA} & \textbf{BRL} & \textbf{RuleFit} & \begin{tabular}{c}\textbf{Node} \\ \textbf{harvest}\end{tabular} & \textbf{SIRUS} \\
  \hline
    Authentification & 0.41 & 0.36 & \B 0.87 & \B 0.86 & 0.48 & 0.59 & \B 0.81 \\
    Breast Wisconsin & 0.21 & 0.55 & \B 0.80 & \B 0.78 & 0.34 & 0.71 & 0.70 \\
    Credit Approval & 0.52 & 0.32 & 0.43 & 0.52 & 0.25 & 0.23 & \B 0.75 \\
    Credit German & 0.46 & 0.22 & 0.51 & 0.41 & 0.24 & 0.48 & \B 0.66 \\
    Diabetes & 0.29 & 0.21 & 0.46 & \B 0.73 & 0.39 & 0.45 & \B 0.81 \\
    Haberman & \B 0.83 & 0.09 & \B 0.79 & 0.50 & 0.46 & 0.52 & 0.65 \\
    Heart C2 & 0.25 & 0.35 & 0.38 & 0.60 & 0.39 & 0.49 & \B 0.71 \\
    Heart H2 & 0.46 & 0.27 & 0.52 & \B 0.73 & 0.29 & 0.29 & \B 0.65 \\
    Heart Statlog & 0.30 & 0.41 & 0.41 & \B 0.75 & 0.35 & 0.48 & \B 0.83 \\
    Hepatitis & 0.26 & 0.16 & 0.24 & 0.34 & 0.26 & 0.49 & \B 0.68 \\
    Ionosphere & \B 0.96 & 0.39 & 0.13 & 0.70 & 0.17 & 0.33 & 0.69 \\
    Kr vs Kp & 0.71 & 0.74 & \B 0.84 & \B 0.80 & 0.19 & 0.27 & \B 0.87 \\
    Liver Disorders & 0.23 & 0.10 & \B 0.91 & 0.50 & 0.24 & 0.31 & 0.58 \\
    Mushrooms & \B 1 & 0.84 & \B 0.98 & 0.80 & 0.69 & 0.48 & 0.86 \\
    Sonar & 0.34 & 0.04 & 0.09 & 0.19 & 0.09 & 0.20 & \B 0.55 \\
    Spambase & 0.49 & 0.10 & 0.46 & \B 0.86 & 0.28 & 0.66 & \B 0.78 \\
    Titanic & 0.55 & 0.42 & 0.69 & \B 0.88 & 0.37 & 0.36 & 0.76 \\
    Vote & \B 1 & 0.52 & 0.68 & NA & 0.21 & 0.30 & 0.75 \\
    Wilt & 0.36 & 0.32 & 0.72 & \B 0.94 & 0.47 & 0.64 & 0.73 \\
  \hline
    Average Rank & 4.2 & 5.9 & 3.3 & 2.8 & 5.6 & 4.3 & 1.9 \\
    p-values & \scriptsize{0.07} & \scriptsize{0.33} & \scriptsize{0.33} & \scriptsize{0.08} & \scriptsize{0.05} & \scriptsize{0.98} &  \\
    \textbf{Final Rank} & \textbf{4} & \textbf{6} & \textbf{2} & \textbf{2}  & \textbf{6} & \textbf{4} & \textbf{1} \\
  \hline \hline
\end{tabular}
\vspace*{1.5mm}
\caption{Mean stability over a $10$-fold cross-validation for UCI datasets.
Results are averaged over $10$ repetitions of the cross-validation. Values within 10\% of the maximum are displayed in bold. Algorithms are ranked with a Mann-Whitney-Wilcoxon test, the p-value with the previous performing algorithm determines the final rank (10\%-level test).}
\label{table_stability}
\end{table}
\begin{table} 
\setlength{\tabcolsep}{1pt}
\centering
\begin{tabular}{|c | c | c | c | c | c | c | c| c| c|}
  \hline \hline
  & \begin{tabular}{c}\textbf{Black} \\ \textbf{box}\end{tabular} & \begin{tabular}{c}\textbf{Decision} \\ \textbf{tree}\end{tabular} & \begin{tabular}{c}\textbf{Classical} \\ \textbf{rule} \\ \textbf{learning}\end{tabular} & \multicolumn{2}{|c|}{\begin{tabular}{c}\textbf{Frequent} \\ \textbf{pattern mining}\end{tabular}} & \multicolumn{3}{|c|}{\textbf{Tree ensemble}} \\
  \hline
  \textbf{Dataset} & \begin{tabular}{c}\textbf{Random} \\ \textbf{Forest} \end{tabular} & \textbf{CART} & \textbf{RIPPER} & \textbf{CBA}  & \textbf{BRL} & \textbf{RuleFit} & \begin{tabular}{c}\textbf{Node} \\ \textbf{harvest}\end{tabular} & \textbf{SIRUS} \\
  \hline
    Authentification & $10^{-4}$ & 0.02 & 0.02 & 0.14 & 0.009 & $\pmb{9.10^{-4}}$ & 0.02 & 0.03 \\
    Breast Wisconsin & 0.009 & 0.06 & 0.07 & 0.05 & 0.02 & \B 0.01 & \B 0.01 & \B  0.01 \\
    Credit Approval & 0.07 & 0.14 & 0.15 & 0.14 & 0.11 & \B  0.08 & \B 0.07 & 0.09 \\
    Credit German & 0.20 & 0.29 & 0.38 & 0.40 & 0.33 & \B 0.23 & \B 0.26 & \B 0.25 \\
    Diabetes & 0.17 & 0.25 & 0.29 & 0.30 & 0.25 & \B 0.18 & \B 0.19 & \B 0.19 \\	
    Haberman & 0.31 & 0.48 & 0.39 & 0.50 & 0.43 & \B 0.37 & \B 0.34 & \B 0.35 \\	
    Heart C2 & 0.10 & 0.19 & 0.23 & 0.17 & 0.23 & 0.12 & 0.12 & \B 0.10 \\	
    Heart H2 & 0.11 & 0.23 & 0.24 & 0.24 & 0.16 & \B 0.11 & \B 0.11 & \B 0.12 \\	
    Heart Statlog & 0.10 & 0.20 & 0.21 & 0.17 & 0.22 & 0.12 & 0.12 & \B 0.10 \\	
    Hepatitis & 0.12 & 0.48 & 0.39 & 0.36 & 0.33 & 0.20 & 0.23 & \B 0.17 \\
    Ionosphere & 0.02 & 0.11 & 0.12 & 0.13 & 0.10 & \B  0.04 & 0.07 & 0.07 \\
    Kr vs Kp & $9.10^{-4}$ & 0.02 & \B 0.009 & 0.05 & 0.01 & \B 0.009 & 0.04 & 0.04 \\
    Liver Disorders & 0.23 & 0.33 & 0.35 & 0.48 & 0.44 & \B 0.27 & 0.30 & 0.35 \\
    Mushrooms & 0 & 0.007 & $3.10^{-5}$ & $5.10^{-4}$ & $\pmb{2.10^{-5}}$ & \B $5.10^{-4}$ & 0.002 & \B $6.10^{-4}$ \\
    Sonar & 0.07 & 0.27 & 0.26 & 0.25 & 0.44 & \B 0.12 & 0.16 & 0.2 \\
    Spambase & 0.01 & 0.11 & 0.08 & 0.12 & 0.05 & \B 0.02 & 0.04 & 0.07 \\
    Titanic & 0.13 & 0.19 & 0.21 & 0.27 & 0.21 & \B 0.14 & 0.16 & 0.17 \\
    Vote & 0.01 & 0.06 & 0.04 & 0.06 & NA & \B 0.02 & \B 0.02 & \B 0.02 \\
    Wilt & 0.007 & 0.18 & 0.13 & 0.48 & 0.07 & \B 0.02 & 0.08 & 0.11 \\
  \hline
    Average Rank & & 5 & 4.9 & 5.8 & 4.4 & 1.4 & 2.4 & 2.8 \\
    p-values & & \scriptsize{0.22} & \scriptsize{0.24} & \scriptsize{0.01} &  \scriptsize{$6.10^{-3}$} & &  \scriptsize{0.01} & \scriptsize{0.34}\\
    \textbf{Final Rank} & & \textbf{4} & \textbf{4} & \textbf{7} & \textbf{4} & \textbf{1} &  \textbf{2} & \textbf{2} \\
  \hline \hline
\end{tabular}
\vspace*{1.5mm}
\caption{Model error (1-AUC) over a $10$-fold cross-validation for UCI datasets. 
Results are averaged over $10$ repetitions of the cross-validation. Values within 10\% of the minimum are displayed in bold, random forest is put aside. Algorithms are ranked with a Mann-Whitney-Wilcoxon test, the p-value with the previous performing algorithm determines the final rank (10\%-level test).}
\label{table_auc}
\end{table}

\subsection{SIRUS Parameters} \label{subsec_parameters}

SIRUS relies on a single tuning hyperparameter: the selection threshold $p_0$ involved in the definition of $\pathset_{M,n,p_{0}}$ to filter the most important rules, which therefore controls the simplicity of the model, and consequently also its accuracy and stability. 
On the other hand, SIRUS is not very sensitive to the other parameters: the number of trees, the number of quantiles, and the tree depth. Therefore, they do not require fine tuning, and we simply set efficient default values as explained below.

\paragraph{Tuning of SIRUS.} 
This parameter $p_0$ should be set to optimize a tradeoff between the number of rules, stability, and accuracy. 
In practice, it is difficult to settle such a criterion, and we choose to optimize $p_0$ to maximize the predictive accuracy with the smallest possible set of rules.
To achieve this goal, we proceed as follows. The error 1-AUC is estimated by $10$-fold cross-validation for a fine grid of $p_0$ values, defined such that $|\pathset_{M,n,p_{0}}|$ varies from $1$ to $25$ rules. (We let $25$ be an arbitrary upper bound on the maximum number of rules, considering that a bigger set is not readable anymore.)
The randomization introduced by the partition of the dataset in the $10$ folds of the cross-validation process has a significant impact on the variability of the size of the final model. Therefore, in order to get a robust estimation of $p_0$, the cross-validation is repeated multiple times (typically $10$) and results are averaged.
The standard deviation of the mean of 1-AUC is computed over these repetitions for each $p_0$ of the grid search. We consider that all models within $2$ standard deviations of the minimum of 1-AUC are not significantly less predictive than the optimal one. Thus, among these models, the one with the smallest number of rules is selected, i.e., the optimal $p_0$ is shifted towards higher values to reduce the model size without decreasing predictivity---see Figures \ref{figure_credit_german} and \ref{figure_heart_statlog} for examples.
This approach is very similar to the tuning procedure of the Lasso \citep{tibshirani1996regression}.

\paragraph{Number of trees.}
The accuracy, stability, and computational cost of SIRUS increase with the number of trees $M$. Thus, we simply design a stopping criterion to grow the minimum number of trees which ensures that accuracy and stability are higher than $95\%$ of their maximum asymptotic values with respect to $M$ and conditionally on $\Dn$.
We empirically observe that the stability requirement is met for a much higher number of trees than the accuracy requirement (about $10$ times). Therefore, the stopping criterion is only based on stability. More precisely, we require that $95\%$ of the rules are identical across two runs of SIRUS on a given dataset $\mathscr{D}_{n}$ in average. Formally, the mean stability $\mathbb{E} [ \hat{S}_{M,n,p_{0}} | \mathscr{D}_{n}]$ measures the expected proportion of rules shared by two fits of SIRUS on $\mathscr{D}_{n}$, for fixed $n$ (sample size), $p_0$ (threshold), and $M$ (number of trees).
Thus, the stopping criterion takes the form $1 - \mathbb{E} [ \hat{S}_{M,n,p_{0}} | \mathscr{D}_{n}] < \alpha$, with typically $\alpha = 0.05$.
 
There are two obstacles to operationalize this stopping criterion: its estimation and its dependence to $p_0$. We make two approximations to overcome these limitations and give empirical and theoretical evidence of their good practical behavior in Appendix \ref{appendix_M}.
First, Theorem \ref{stability_equivalent} in Appendix \ref{appendix_M_theory} provides an asymptotic equivalent with respect to $M$ of $1 - \mathbb{E} [ \hat{S}_{M,n,p_{0}} | \mathscr{D}_{n}]$, that we simply estimate by
\begin{align*}
\varepsilon_{M,n,p_0} = \frac{\sum_{\path \in \Pi} \Phi( M p_{0}, M, \hat{p}_{M, n}(\path) ) ( 1 - \Phi( M p_{0}, M,
 \hat{p}_{M, n}(\path)) )}{\sum_{\path \in \Pi} ( 1 - \Phi( M p_{0}, M,  \hat{p}_{M, n}(\path) ) ) },
\end{align*}
where $\Phi ( M p_{0}, M, p_n(\path))$ is the cdf of a binomial distribution with parameter 
$p_n(\path)$, $M$ trials, evaluated at $M p_{0}$.
Secondly, $\varepsilon_{M,n,p_0}$ depends on $p_0$, whose optimal value is unknown in the first step of SIRUS, when trees are grown.
It turns out however that $\varepsilon_{M,n,p_0}$ is not very sensitive to $p_0$, as shown by the experiments in Appendix \ref{appendix_M_xp}. Consequently, our strategy is to simply average $\varepsilon_{M,n,p_0}$ over a set $\hat{V}_{M,n}$ of many possible values of $p_0$ and use the resulting average as a gauge. These values are chosen to scan all possible path sets $\pathset_{M,n,p_0}$, of size ranging from $1$ to $50$ paths. When a set of $50$ paths is post-treated, its size reduces to around $25$ paths (as explained in the previous paragraph, $25$ is an arbitrarily threshold on the maximum number of rules above which a rule model is not readable anymore). In order to generate path sets of such sizes, values of $p_0$ are chosen halfway between two distinct consecutive $\hat{p}_{M,n}(\path), \path \in \Pi$, restricted to the highest $50$ values.
Thus, in the experiments, we utilize the following criterion to stop the growing of the forest, with typically $\alpha = 0.05$:
\begin{align} \label{criterion_M}
\underset{M}{\textrm{argmin}} \Big\{ \frac{1}{| \hat{V}_{M,n}|} \sum_{p_0 \in \hat{V}_{M,n}}
\varepsilon_{M,n,p_0} < \alpha \Big\}.
\end{align}

\paragraph{Quantile discretization.}
\label{simu_RF_bin}
In the modified random forest grown in the first step of SIRUS, the split at each tree node is limited to the empirical $q$-quantiles of each component of $\bX$, as described in Section \ref{sec_algo}.
Thus, we check that this modification alone of the forest has little impact on its accuracy. 
Using the R package \texttt{ranger}, 1-AUC is estimated for each dataset with 10-fold cross-validation for $q \in \{2, 5, 10, 20\}$.
We leave aside datasets with a majority of categorical variables, results are averaged over 10 repetitions of the cross-validation, and displayed in Table \ref{table_RF_bin}. Clearly, the decrease of accuracy generated by this discretization is small, and not very sensitive to $q$, provided that $q$ is not too small. Thus, $q = 10$ appears to be a good default choice from the experiments. In fact, the small impact of the discretization on the forest error is not surprising: with only $p=10$ input variables, the input space is split in a fine grid of $10^{10}$ hyperrectangles for $q = 10$ quantiles, providing a high flexibility to the modified random forest to identify local patterns.
\begin{table}
\centering
\begin{tabular}{|c | c | c | c | c | c |}
  \hline \hline
  \textbf{Dataset} & \textbf{Breiman's RF} & $\textbf{q=2}$ & $\textbf{q=5}$ & $\textbf{q=10}$ & $\textbf{q=20}$ \\
  \hline
    Authentification & 0.0002 & 0.08 & 0.002 & 0.0005 & 0.0004 \\
    Diabetes & 0.17 & 0.23 & 0.18 & 0.18 & 0.18 \\	
    Haberman & 0.32 & 0.35 & 0.30 & 0.32 & 0.30 \\		
    Heart Statlog & 0.10 & 0.10 & 0.10 & 0.10 & 0.10 \\	
    Hepatitis & 0.13 & 0.15 & 0.14 & 0.14 & 0.13 \\
    Ionosphere & 0.02 & 0.07 & 0.03 & 0.02 & 0.02 \\
    Liver Disorders & 0.23 & 0.32 & 0.27 & 0.25 & 0.24 \\
    Sonar & 0.07 & 0.09 & 0.07 & 0.07 & 0.07 \\
    Spambase & 0.01 & 0.14 & 0.03 & 0.02 & 0.01 \\
    Titanic & 0.13 & 0.15 & 0.14 & 0.14 & 0.13 \\
    Wilt & 0.007 & 0.15 & 0.03 & 0.02 & 0.02 \\
\hline \hline
\end{tabular}
\vspace*{1.5mm}
\caption{\small{Accuracy, measured by 1-AUC on UCI datasets, for two algorithms:
Breiman's random forests and random forests with splits limited to $q$-quantiles, for $q \in \{2,5,10,20\}$.}}
\label{table_RF_bin}
\end{table}

\paragraph{Tree depth.}
When SIRUS is fit using fully grown trees, the final set of rules $\pathset_{M,n,p_0}$ contains almost exclusively rules made of one or two splits, and rarely of three splits.
Although this may appear surprising at first glance, this phenomenon is in fact expected. 
Indeed, rules made of multiple splits are extracted from deeper tree levels and are thus more sensitive to data perturbation by construction. This results in much smaller values of $\hat{p}_{M,n}(\path)$ for rules with a high number of splits, and then deletion from the final set of path through the threshold $p_0$: $\smash{\pathset_{M,n,p_{0}}=\{ \path \in \Pi:\hat{p}_{M,n}(\path)>p_{0}\}}$. To illustrate this, let us consider the following typical example with $p = 100$ input variables and $q = 10$ quantiles. There are $q p = 100 \times 10 = 10^3$ possible splits at the root node of a tree, and then $2 pq = 2.10^3$ paths of one split. Since the left and right paths of one split at the root node are associated to the same rule, there are $q p = 10^3$ distinct rules of one split, about $(2 q p)^2 \approx 10^6$ distinct rules of two splits, and about $(2 q p)^3 \approx 10^{10}$ distinct rules of three splits. Using only rules of one split is too restrictive since it generates a small model class (a thousand rules for $100$ input variables) and does not handle variable interactions. On the other hand, rules of two splits are numerous (about one million) and thus provide a large flexibility to SIRUS. More importantly, since there are $10$ billion rules of three splits, a stable selection of a few of them is clearly a difficult task, and such complex rules are naturally discarded by SIRUS. 

In the software implementation \texttt{sirus}, the tree depth parameter \texttt{max.depth} is a modifiable input, set to $2$ by default to reduce the computational cost while leaving the output list of rules almost untouched as explained above. We conduct experiments where SIRUS is run with a tree depth of $1$, $2$, and $3$, and results are displayed in Table \ref{table_auc_d3}. 
Over the nineteen UCI datasets, rules of three splits appear in SIRUS rule list in only four cases, and a significant accuracy improvement over a tree depth of $2$ occurs only once, for the `Mushrooms' dataset. On the other hand, for all datasets except two, SIRUS outputs rules of two constraints, and predictivity is improved over a tree depth of $1$ for half of the datasets. The Titanic example shows how the rule list is drastically simplified by limiting tree depth to $1$, lowering the insights provided by SIRUS:
\begin{center}
\setlength{\fboxrule}{1pt}
\fbox{\begin{minipage}{0.71\textwidth}
\textbf{Average survival rate} $p_s = 39 \%$. \\[1em]
\begin{tabular}{c c c c c c}
\textbf{if} & \texttt{sex} is male & \textbf{then} & $p_s = 19\%$ & \textbf{else} & $p_s = 74\%$ \\[0.7em]
\textbf{if} & $1^{st}$ or $2^{nd}$ \texttt{class} & \textbf{then} & $p_s = 56\%$ & \textbf{else} & $p_s = 24\%$ \\[0.7em]
\end{tabular}
\end{minipage}}
\end{center}

This analysis of tree depth is not new. Indeed, both RuleFit \citep{friedman2008predictive} and Node harvest \citep{meinshausen2010node} articles discuss the optimal tree depth for the rule extraction from a tree ensemble in their experiments.
They both conclude that the optimal depth is $2$. Hence, the same hard limit of $2$ is used in Node harvest. RuleFit is slightly less restrictive: for each tree, its depth is randomly sampled with an exponential distribution concentrated on $2$, but allowing few trees of depth $1$, $3$, and $4$. We insist that they both reach such conclusion without considering stability issues, but only focusing on accuracy. Further considering stability properties consolidates that growing shallow trees is optimal for rule extraction from tree ensembles.
\begin{table}
\setlength{\tabcolsep}{2pt}
\centering
\begin{tabular}{|c | c | c | c |}
  \hline \hline
  \textbf{Dataset} & \textbf{SIRUS - depth = 1} & \textbf{SIRUS - depth = 2} & \textbf{SIRUS - depth = 3} \\
  \hline
    Authentification & 0.07 & \B 0.03 & \B 0.03 \\
    Breast Wisconsin & 0.01 & 0.01 & 0.01 \\
    Credit Approval	& 0.11 & \B 0.09 & \B 0.09 \\
    Credit German & 0.25 & 0.25 & 0.26 \\
    Diabetes & 0.19 & 0.19 & 0.19 \\	
    Haberman & 0.35 & 0.35 & 0.35 \\	
    Heart C2 & 0.11 & \B 0.10 & 0.11 \\	
    Heart H2 & 0.12 & 0.12  & 0.12 \\	
    Heart Statlog & 0.11 & \B 0.10 & \B 0.10 \\	
    Hepatitis & \B 0.15 & 0.17 & 0.18 \\
    Ionosphere & 0.07 & 0.07 & 0.07 \\
    Kr vs Kp & 0.05 & \B 0.04 & 0.06 \\
    Liver Disorders & 0.38 & \B 0.35 & \B 0.35 \\
    Mushrooms & $3.10^{-3}$ & $6.10^{-4}$ & $\pmb{3.10^{-4}}$ \\
    Sonar & 0.19 & 0.2 & 0.2 \\
    Spambase & 0.06 & 0.07 & 0.07 \\
    Titanic & 0.19 & \B 0.17 & \B 0.16 \\
    Vote & 0.02 & 0.02 & 0.02 \\
    Wilt & 0.19 & \B 0.11 & \B 0.11 \\
  \hline \hline
\end{tabular}
\vspace*{1.5mm}
\caption{\small{SIRUS error (1-AUC) over a $10$-fold cross-validation (averaged over $10$ repetitions) when tree depth is limited to $1$, $2$ or $3$. Values within 10\% of the minimum are displayed in bold, except for datasets with no significant variations.}}
\label{table_auc_d3}
\end{table}

\section{Conclusion} \label{sec_conclusion}
Interpretability of learning algorithms is required for applications involving critical decisions, for example the analysis of production processes in the manufacturing industry. Although interpretability does not have a precise definition, we argued that simplicity, stability, and predictivity are minimum requirements.
In particular, decision trees and rule algorithms both combine a simple structure and a good accuracy for nonlinear data, and are thus considered as state-of-the-art interpretable algorithms. However, these methods are unstable with respect to data perturbation, which is a strong operational limitation. Therefore, we proposed a new rule algorithm for classification, SIRUS (Stable and Interpretable RUle Set), which takes the form of a short list of rules. We proved that SIRUS considerably improves stability over state-of-the-art algorithms, while preserving simplicity, accuracy, and computational complexity of top competitors.
The principle of SIRUS is to extract rules from a random forest, based on their probability of occurrence in a random tree, and to stop the growing of the forest when the rule selection is converged. Thus, SIRUS inherits the computational complexity of random forests, and has only one tuning parameter. A software implementation, the \texttt{R/C++} package \texttt{sirus} \citep{Rsirus}, is available from \texttt{CRAN}.
Besides, we believe that the extension of SIRUS to regression is a promising future research direction: the main challenge is the construction of an appropriate rule aggregation framework to accurately estimate continuous outputs without hurting stability. Furthermore, although SIRUS has the ability to handle high-dimensional data, as illustrated with the SECOM dataset (590 inputs), specific variable selection strategies could be used to reduce the number of possible rules and then improve SIRUS performance.

\bibliography{biblio}

\newpage

\appendix


\section{Additional Experiments} \label{appendix_add_xp}

\subsection{Robustness Illustration} \label{appendix_stability}

\FloatBarrier

For the SECOM dataset used in the experimental Section 5 of the article, only three rule algorithms achieve the same predictivity as random forests: RuleFit, Node harvest, and SIRUS (1-AUC of $0.30$, whereas CART and BRL are no better than the random classifier with an error of 1-AUC $= 0.5$). SIRUS produces a short and stable list of $6$ rules, while RuleFit and Node harvest generate complex, long, and unstable rule lists. 
Rule algorithms based on tree ensembles are stochastic since they rely on the tree randomness $\Theta_1, \hdots, \Theta_M$. Consequently, RuleFit and Node harvest output different rule lists when run multiple times on the same dataset. Such behavior is a strong limitation in practice, as domain experts become skeptical of the algorithm conclusions. On the other hand, SIRUS is built to have a robust rule extraction mechanism, and the same list of rules is output over multiple repetitions with the same data, as proved in Theorem \ref{stability_equivalent} in the next Section.

To illustrate this, we run each algorithm twice on the SECOM dataset, and display the output models in Figure \ref{figure_sirus_secom} for SIRUS, Figure \ref{figure_nodeharvest_secom} for Node harvest, and Figure \ref{figure_rulefit_secom} for RuleFit.
We set the regularization parameter of Node harvest and SIRUS as explained in Subsection $5.3$ of the article, to maximize accuracy with the smallest possible model: for Node harvest $\lambda = 4$, and for SIRUS $p_0 = 0.04$. RuleFit is tuned as defined in \citet{friedman2008predictive}.
Figures \ref{figure_nodeharvest_secom} and \ref{figure_rulefit_secom} show that the rule lists output by RuleFit and Node harvest are quite different across multiple runs with the exact same data, while SIRUS has the same output.

We also observe that for the same accuracy, RuleFit and Node harvest models are longer and more complex than SIRUS. In addition, rules are aggregated using weights to generate predictions. This is not the case for SIRUS, which simply averages the $6$ output rules.
Finally, we can also mention that manually increasing the regularization of Node harvest, to reduce the model size to $6$ rules as in SIRUS, strongly hurts accuracy, which drops to $0.39$.
\begin{figure}
\begin{center}
\setlength{\fboxrule}{1pt}
\fbox{\begin{minipage}{0.85\textwidth}
\includegraphics[scale = 0.9]{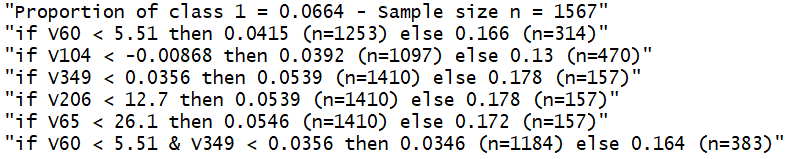}
\end{minipage}}
\fbox{\begin{minipage}{0.85\textwidth}
\includegraphics[scale = 0.9]{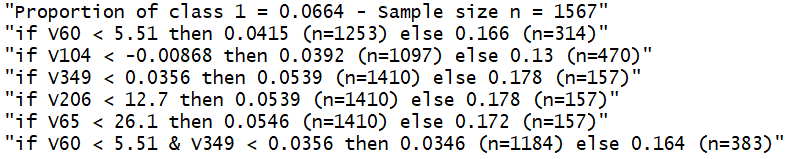}
\end{minipage}}
\caption{\small{The two lists of rules output by two runs of SIRUS for the SECOM dataset.}}
\label{figure_sirus_secom}
\end{center}
\end{figure}
\begin{figure}
\begin{center}
\setlength{\fboxrule}{1pt}
\fbox{\begin{minipage}{0.7\textwidth}
\includegraphics[scale = 0.83]{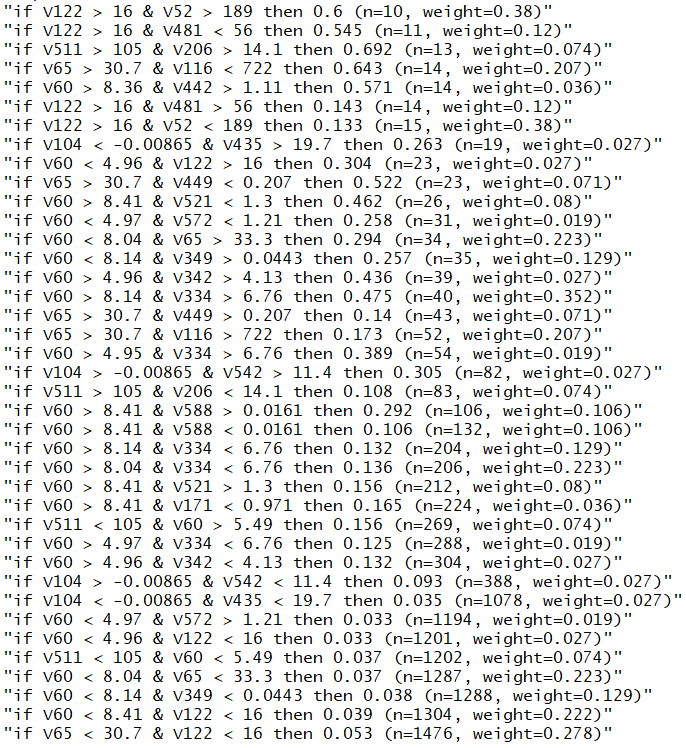}
\end{minipage}}
\fbox{\begin{minipage}{0.7\textwidth}
\includegraphics[scale = 0.83]{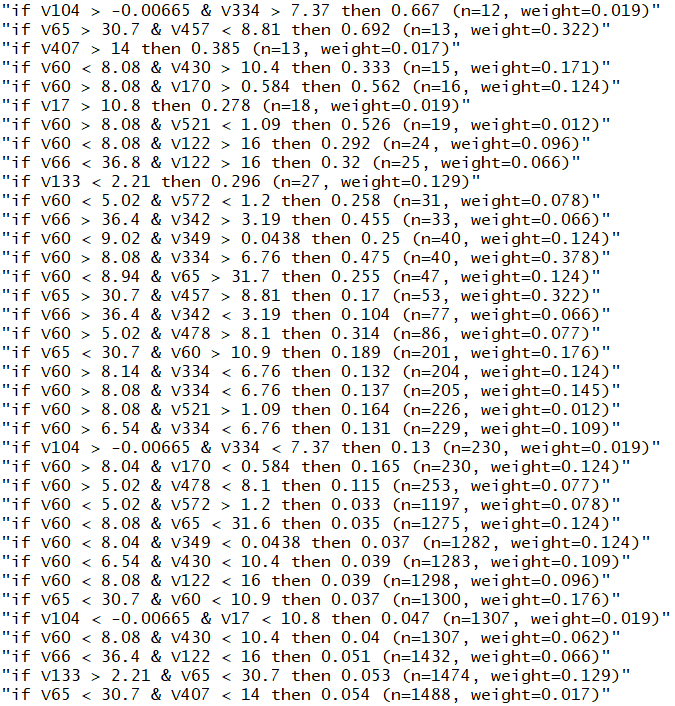}
\end{minipage}}
\caption{\small{The two lists of rules output by two runs of Node harvest for the SECOM dataset.}}
\label{figure_nodeharvest_secom}
\end{center}
\end{figure}
\begin{figure}
\begin{center}
\setlength{\fboxrule}{1pt}
\fbox{\begin{minipage}{0.9\textwidth}
\includegraphics[scale = 1]{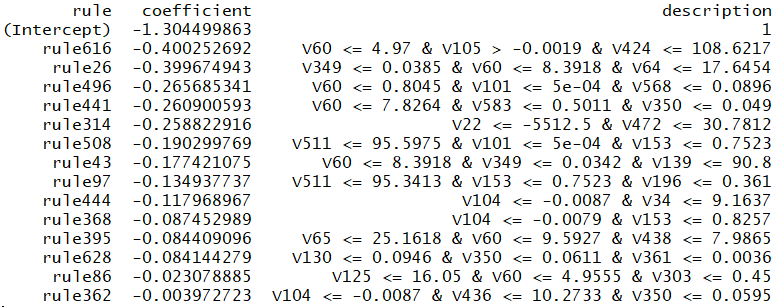}
\end{minipage}}
\fbox{\begin{minipage}{0.9\textwidth}
\includegraphics[scale = 1]{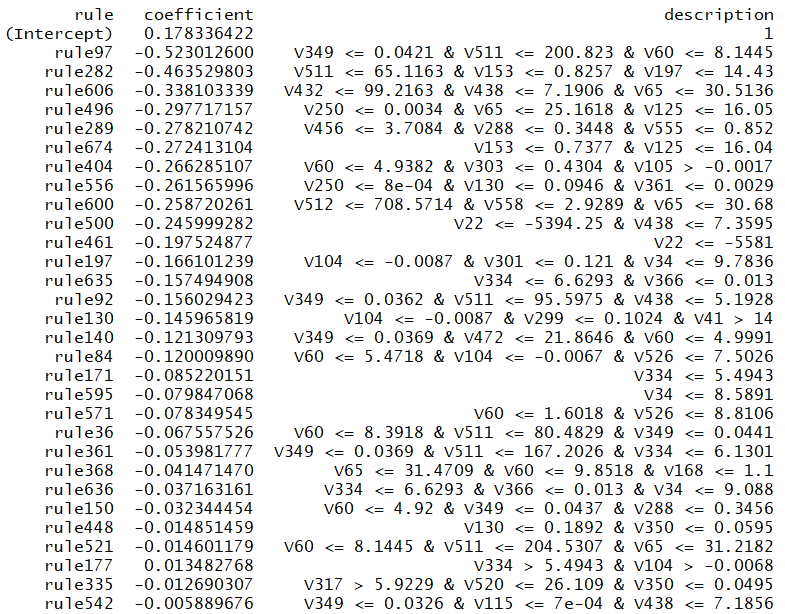}
\end{minipage}}
\caption{\small{The two lists of rules output by two runs of RuleFit for the SECOM dataset.}}
\label{figure_rulefit_secom}
\end{center}
\end{figure}

\subsection{Additional Competitors} \label{appendix_xp}

Additional experiments are provided to compare SIRUS to other competitors: C5.0 \citep{quinlan1992c4} (decision tree), PART \citep{frank1998generating}, and FOIL \citep{quinlan1995induction} (classical rule learning algorithms). Model size results are provided in Table \ref{table_size_appendix}, stability in Table \ref{table_stability_appendix}, and error in Table \ref{table_auc_appendix}. The stability and accuracy improvement of SIRUS is clear.

\begin{table}
\centering
\begin{tabular}{|c | c | c | c | c |}
  \hline \hline
  \textbf{Dataset} & \textbf{C5.0} & \textbf{PART} & \textbf{FOIL} & \textbf{SIRUS} \\
  \hline
    Authentification & 11 & 8 & 20 & 13 \\
    Breast Wisconsin & 5 & 10 & 41 & 24 \\
    Credit Approval & 9 & 32 &  40 & 16 \\
    Credit German & 22 & 68 & 101 & 22 \\
    Diabetes & 12 & 7 & 36  & 8 \\
    Haberman & 2 & 2 & 4 & 5\\
    Heart C2 & 10 & 20 & 31 & 20 \\
    Heart H2 & 4 & 15 & 29 & 12 \\	
    Heart Statlog & 10 & 18 & 28 & 15 \\
    Hepatitis & 7 & 8 & 14 & 12 \\
    Ionosphere & 9 & 6 & 28 & 15 \\
    Kr vs Kp & 11 & 21 & 24 & 24 \\
    Liver Disorders & 14 & 7 & 2 & 17 \\
    Mushrooms & 7 & 9 & 14 & 23 \\
    Sonar & 10 & 6 & 20 & 19 \\
    Spambase & 29 & 46 & 73 & 21 \\
    Titanic & 7 & 15 & 17 & 6 \\
    Vote & 5 & 7 & 19 & 7 \\
    Wilt & 10 & 8 & 10 & 24 \\
  \hline \hline
\end{tabular}
\vspace*{-2mm}
\caption{Mean model size over a $10$-fold cross-validation for UCI datasets (averaged over $10$ repetitions).}
\label{table_size_appendix}
\end{table}

\begin{table} 
\centering
\begin{tabular}{|c | c | c | c | c |}
  \hline \hline
  \textbf{Dataset} & \textbf{C5.0} & \textbf{PART} & \textbf{FOIL} & \textbf{SIRUS} \\
  \hline
    Authentification & 0.44 & 0.43 & \B 0.81 & \B 0.81 \\
    Breast Wisconsin & 0.17 & 0.49 & 0.36 & \B 0.70 \\
    Credit Approval & 0.18 & 0.31 & 0.17 & \B 0.75 \\
    Credit German & 0.03 & 0.16 & 0.11 & \B 0.65 \\
    Diabetes & 0.07 & 0.15 & 0.18 & \B 0.81 \\
    Haberman & 0.28 & 0.25 & \B 0.64 & \B 0.65 \\
    Heart C2 & 0.09 & 0.15 & 0.16 & \B 0.71 \\
    Heart H2 & 0.32 & 0.31 & 0.39 & \B 0.65 \\
    Heart Statlog & 0.11 & 0.15 & 0.15 & \B 0.82 \\
    Hepatitis & 0.10 & 0.15 & 0.05 & \B 0.68 \\
    Ionosphere & 0.24 & 0.13 & 0.07 & \B 0.69 \\
    Kr vs Kp & 0.65 & 0.51 & 0.85 & \B 0.87 \\
    Liver Disorders & 0.05 & 0.07 & \B 0.69 & 0.58 \\
    Mushrooms & 0.79 & 0.78 & \B 0.93 & \B 0.86 \\
    Sonar & 0.06 & 0.06 & 0.04 & \B 0.55 \\
    Spambase & 0.08 & 0.08 & 0.11 & \B 0.78 \\
    Titanic & 0.49 & 0.27 & \B 0.77 & \B 0.76 \\
    Vote & \B 0.67 & 0.40 & 0.39 & \B 0.75 \\
    Wilt & 0.34 & 0.37 & 0.48 & \B 0.73 \\
  \hline \hline
\end{tabular}
\vspace*{-2mm}
\caption{Mean stability over a $10$-fold cross-validation for UCI datasets (averaged over $10$ repetitions). Values within 10\% of the maximum are displayed in bold.}
\label{table_stability_appendix}
\end{table}

\begin{table} 
\centering
\begin{tabular}{|c | c | c | c | c | c | c |}
  \hline \hline
  \textbf{Dataset} & \textbf{C5.0} & \textbf{PART} & \textbf{FOIL} & \textbf{SIRUS} \\
  \hline
    Authentification & 0.02 & \B 0.01 & 0.08 & 0.03 \\
    Breast Wisconsin & 0.06 & 0.07 & 0.08 & \B  0.01 \\
    Credit Approval & 0.15 & 0.17 & 0.15 & \B 0.09 \\
    Credit German & 0.37 & 0.36 & 0.41 & \B 0.25 \\
    Diabetes & 0.28 & 0.30 & 0.28  & \B 0.19 \\	
    Haberman & 0.46 & 0.42 & 0.50 & \B 0.35 \\	
    Heart C2 & 0.20 & 0.23 & 0.19 & \B 0.10 \\	
    Heart H2 & 0.23 & 0.23 & 0.23 & \B 0.12 \\	
    Heart Statlog & 0.21 & 0.24 & 0.20 & \B 0.10 \\	
    Hepatitis & 0.34 & 0.34 & 0.39 & \B 0.17 \\
    Ionosphere & 0.10 & 0.10 & 0.13 & \B 0.07 \\
    Kr vs Kp & \B 0.006 & 0.008 & 0.02 & 0.04 \\
    Liver Disorders & \B 0.34 & \B 0.38 & 0.50 & \B 0.35 \\
    Mushrooms & 0.001 & 0 & $\pmb{6.10^{-5}}$ & $6.10^{-4}$ \\
    Sonar & 0.26 & 0.26 & 0.26 & \B 0.2 \\
    Spambase & \B 0.07 & \B 0.07 & 0.12 & \B 0.07 \\
    Titanic & 0.20 & 0.20 & 0.25 & \B 0.17 \\
    Vote & 0.04 & 0.05 & 0.05 & \B 0.02 \\
    Wilt & 0.15 & 0.17 & 0.46 & \B 0.11 \\
  \hline \hline
\end{tabular}
\vspace*{-2mm}
\caption{Model error (1-AUC) over a $10$-fold cross-validation for UCI datasets (averaged over $10$ repetitions). Values within 10\% of the minimum are displayed in bold.}
\label{table_auc_appendix}
\end{table}

\subsection{Rule Aggregation}
\label{simu_mean_logit}
In Section $3$ of the article, $\hat{\eta}_{M,n,p_{0}}(\bx)$ ($3.3$) is a simple average of the set of rules, defined as
\begin{align} \label{eq_eta_mean}
\hat{\eta}_{M,n,p_{0}}(\bx) = \frac{1}{| \hat{\mathscr{P}}_{M,n,p_{0}}|}
 \sum_{\path \in \pathset_{M,n,p_{0}}}  \hat{g}_{n, \path}( \bx).
\end{align}
To tackle our binary classification problem, a natural approach would be to use a logistic regression and define 
\begin{align} \label{eq_eta_logit}
\ln\Big(\frac{\hat{\eta}_{M,n,p_{0}}(\bx)}{1 - \hat{\eta}_{M,n,p_{0}}(\bx)}\Big)
 = \sum_{\path \in \pathset_{M,n,p_{0}}} \beta_{\path}   \hat{g}_{n, \path}( \bx),
\end{align}
where the coefficients $\beta_{\path}$ have to be estimated.
To illustrate the performance of the logistic regression (\ref{eq_eta_logit}), we consider again the UCI dataset, ``Credit German''. We augment the previous results from Figure $4$ (in Section $5$ of the article) with the logistic regression error in Figure \ref{figure_auc_logit_cg}.
\begin{figure}
\begin{center}
\includegraphics[height=6cm,width=9cm]{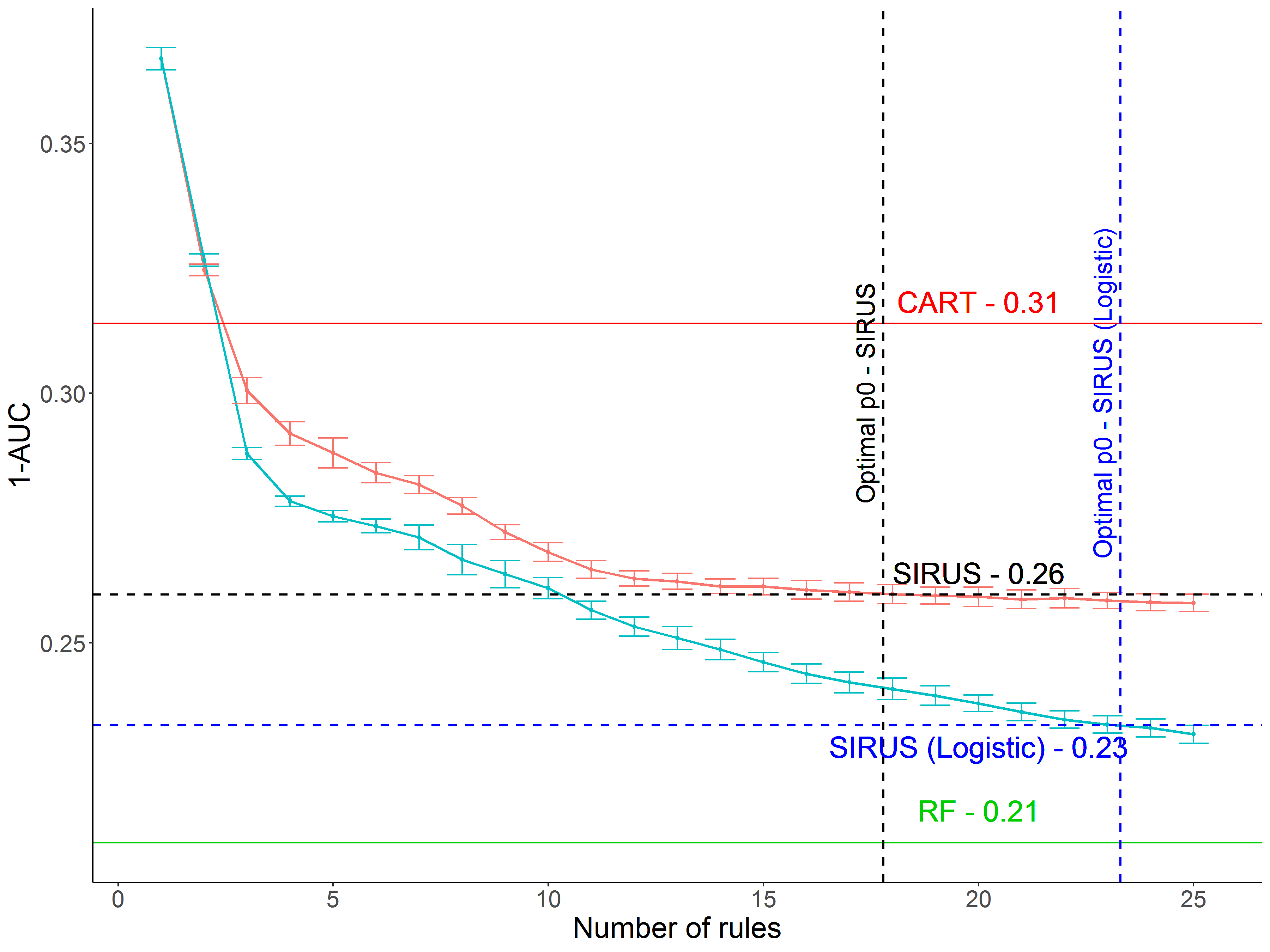}
\caption{\small{For the UCI dataset ``Credit German'', 1-AUC versus the number of rules when $p_0$ varies, estimated via 10-fold cross-validation (repeated $30$ times) for two different methods of rule aggregation: the rule average (\ref{eq_eta_mean}) in red and a logistic regression (\ref{eq_eta_logit}) in blue.}}
\label{figure_auc_logit_cg}
\end{center}
\end{figure}
One can observe that the predictive accuracy is slightly improved but it comes at the price of an additional set of coefficients that can be hard to interpret (some can be negative), and an increased computational cost. Notice that categorical variables are one-hot-encoded in this example.

\section{Stopping Criterion for the Number of Trees $M$} \label{appendix_M}

We recall that the definition of the stopping criterion ($5.1$) of the forest growing is provided in Section $5$ of the main article. First, we provide three groups of experiments to show its good empirical efficiency. In the second subsection, we provide theoretical properties of the stopping criterion.

\subsection{Experiments} \label{appendix_M_xp}
The following experiments on the UCI datasets show the good empirical performance of the stopping criterion ($5.1$).
Recall that the goal of this criterion is to determine the minimum number of trees $M$ ensuring that two independent fits of SIRUS on the same dataset result in two lists of rules with an overlap of $95$\% in average. This is checked with a first batch of experiments---see next paragraph. Secondly, the stopping criterion (5.1) does not consider the optimal $p_0$, unknown when trees are grown in the first step of SIRUS. Then, another batch of experiments is run to show that the stability approximation $1 - \varepsilon_{M,n,p_0}$ is quite insensitive to $p_0$. Finally, a last batch of experiments provides examples of the number of trees grown when SIRUS is fit.
Notice that for these experiments, categorical variables are one-hot-encoded.
\paragraph{Experiments 1.}
For each dataset, the following procedure is applied.
SIRUS is run a first time using criterion (5.1) to stop the number of trees. This initial run provides the optimal number of trees $M$ as well as the set $\hat{V}_{M,n}$ of possible $p_0$. Then, SIRUS is fit twice independently using the precomputed number of trees $M$. For each $p_0 \in \hat{V}_{M,n}$, the stability metric $\hat{S}_{M,n,p_{0}}$ (with $\Dn'=\Dn$) is computed over the two resulting lists of rules. Finally $\hat{S}_{M,n,p_{0}}$ is averaged across all $p_0$ values in $\hat{V}_{M,n}$. This procedure is repeated $10$ times: results are averaged and presented in Table \ref{table_epsilon}, with standard deviations in parentheses. Across the considered datasets, resulting values range from $0.941$ to $0.955$, and are thus close to $0.95$ as expected by construction of criterion (5.1).
\begin{table}
\centering
\begin{tabular}{|c | c |}
  \hline \hline
  \textbf{Dataset} & \textbf{Mean stability} \\
  \hline
Haberman & 0.950 (0.01) \\	
Diabetes & 0.950 (0.007) \\	
Heart Statlog & 0.954 (0.007)	 \\	
Liver Disorders	&   0.951 (0.006) \\	
Heart C2	& 0.955 (0.009) \\	
Heart H2	& 0.952 (0.009) \\	
Credit German	&0.950 (0.008)  \\	
Credit Approval	& 0.941 (0.02)  \\	
Ionosphere & 0.950 (0.009) \\	
  \hline \hline
\end{tabular}
\vspace*{1.5mm}
\caption{\small{Values of $\hat{S}_{M,n,p_{0}}$ averaged over $p_0 \in \hat{V}_{M,n}$ when the stopping criterion (5.1) is used to set $M$, for UCI datasets. 
Results are averaged over $10$ repetitions and standard deviations are displayed in parentheses.}}
 \label{table_epsilon}
\end{table}
\paragraph{Experiments 2.} The second type of experiments illustrates that $\varepsilon_{M,n,p_0}$ is quite insensitive to $p_0$ when $M$ is set with criterion (5.1). For the ``Credit German'' dataset, we fit SIRUS and then compute $1 - \varepsilon_{M,n,p_0}$ for each $p_0 \in \hat{V}_{M,n}$. Results are
displayed in Figure \ref{figure_epsilon}. $1 - \varepsilon_{M,n,p_0}$ ranges from $0.90$ to $1$, where the extreme values are reached for $p_0$ corresponding to very small number of rules, which are not of interest when $p_0$ is selected to maximize predictive accuracy. Thus, $1 - \varepsilon_{M,n,p_0}$ is quite concentrated around $0.95$ when $p_0$ varies.
\begin{figure}
\begin{center}
\includegraphics[height=8cm,width=12cm]{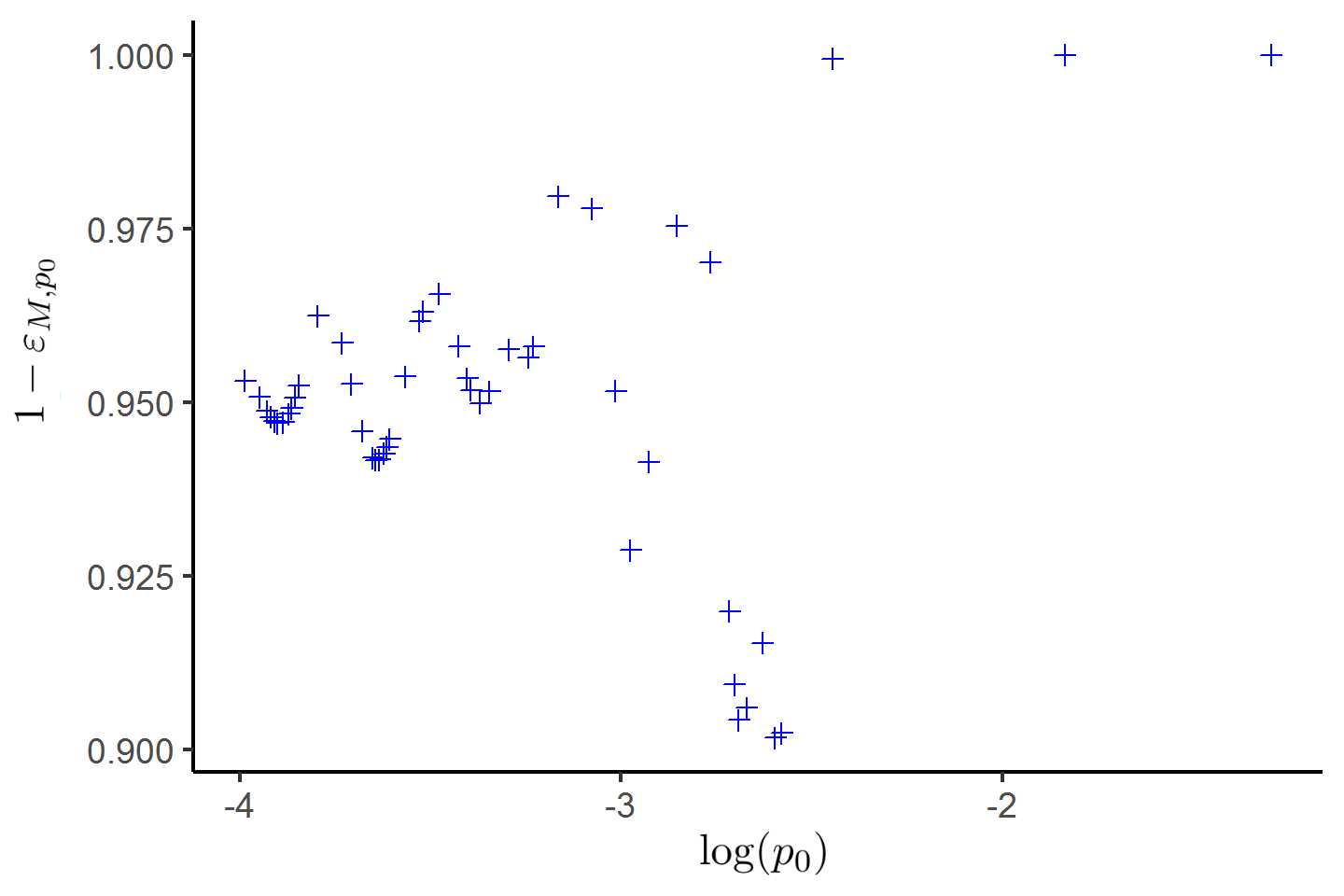}
\caption{\small{For the UCI dataset ``Credit German'', $1 - \varepsilon_{M,n,p_0}$ for a sequence of $p_0 \in \hat{V}_{M,p_0}$
 corresponding to final models ranging from $1$ to about $25$ rules.}}
\label{figure_epsilon}
\end{center}
\end{figure}
\paragraph{Experiments 3.} Finally, we display in Table \ref{table_M} the optimal number of trees when the growing of SIRUS is stopped using criterion (5.1). 
It ranges from $4220$ to $20\,650$ trees. In Breiman's forests, the number of trees above which the accuracy cannot be significantly improved  is typically $10$ times lower. However SIRUS grows shallow trees, and is thus not computationally more demanding than random forests overall.
\begin{table}
\centering
\begin{tabular}{|c | c |}
  \hline \hline
  \textbf{Dataset} & \textbf{Nb of trees (sd)} \\
  \hline
Haberman & 10\,920 (877) \\	
Diabetes & 18\,830 (1538) \\	
Heart Statlog &	7840 (994) \\	
Liver Disorders	& 14\,650 (1242)  \\	
Heart C2	& 6840 (1270) \\	
Heart H2	& 4220 (529) \\	
Credit German	&  7940 (672) \\	
Credit Approval	&  20\,650 (8460) \\	
Ionosphere & 7320 (487)  \\	
  \hline \hline
\end{tabular}
\vspace*{1.5mm}
\caption{\small{Number of trees $M$ determined by the stopping criterion (5.1) for UCI datasets. 
Results are averaged over $10$ repetitions and standard deviations are displayed in parentheses.}}
 \label{table_M}
\end{table}

\subsection{Theoretical Properties} \label{appendix_M_theory}
We emphasize that growing more trees does not improve predictive accuracy or stability with respect to data perturbation for a fixed sample size $n$. 
Indeed, the instability of the rule selection is generated by the variance of the estimates $\hat{p}_{M,n}(\path), \path \in \Pi$.
Upon noting that we have two sources of randomness---$\Theta$ and $\mathscr{D}_{n}$---, the law of total variance shows that  $\mathbb{V}[\hat{p}_{M,n}(\path)]$ can be broken down into two terms: the variance generated by the Monte Carlo randomness $\Theta$ on the one hand, and the sampling variance on the other hand. In fact, equation (\ref{proof_thm1_variance}) in the proof of Theorem \ref{theorem_consistency} below reveals that
\[
\mathbb{V} [ \hat{p}_{M, n} (\path)]
= \frac{1}{M} \E[p_n(\path)] ( 1 - \E[p_n(\path)]) + ( 1 - \frac{1}{M} ) \mathbb{V} [p_n(\path) ] . 
 \] 
The stopping criterion (5.1) ensures that the first term becomes negligible as $M \to \infty$, so that
 $\mathbb{V}[\hat{p}_{M,n}(\path)]$ reduces to the sampling variance $\mathbb{V}[p_n(\path)]$, which is independent of $M$. 
Therefore, stability with respect to data perturbation cannot be further improved by increasing the number of trees. 
Additionally, the trees are only involved in the selection of the paths. For a given set of paths $\pathset_{M,n,p_0}$, the construction of 
the final aggregated estimate $\hat{\eta}_{M,n,p_0}$ (see Section $3$ of the article) is independent of the forest. 
Thus, if further increasing the number of trees does not impact the path selection, neither it improves the predictive accuracy.

Next, Theorem~\ref{stability_equivalent} states that conditionally on $\mathscr{D}_{n}$ and with $\mathscr{D}_{n}' = \mathscr{D}_{n}$, $\hat{S}_{M,n,p_{0}}$ should be close to $1$, and also provides an asymptotic approximation of $\smash{\mathbb{E} [ \hat{S}_{M,n,p_{0}} | \mathscr{D}_{n}]}$ for large values of the number of trees $M$, which quantifies the influence of $M$ on  the mean stability, conditional on $\mathscr{D}_{n}$. 
We let $\smash{\setpemp \stackrel{\tiny{\rm def}}{=} \{ p_n(\path):  \path \in \Pi \}}$ be the empirical counterpart of $\mathcal{U}^{\star}$.
\setcounter{theorem}{1}
\begin{theorem}
\label{stability_equivalent}
If $p_{0}\in[0, 1] \setminus \setpemp$ and $\mathscr{D}'_{n} = \mathscr{D}_{n}$, then, conditional on $\mathscr{D}_{n}$, we have
\begin{align} \label{eq_stab_M}
\lim \limits_{M \to \infty} \hat{S}_{M,n,p_{0}} = 1 \quad \textrm{in probability.}
\end{align}
In addition, for all $p_{0} < \max~\setpemp$,
\begin{align*}
1 - & \mathbb{E} [ \hat{S}_{M,n,p_{0}} | \mathscr{D}_{n}] \\ & \underset{M \to \infty}{\sim} \sum_{\path \in \Pi}
\frac{\Phi( M p_{0}, M, p_n(\path)) ( 1 - \Phi\big( M p_{0}, M,
p_n(\path)))}
{\frac{1}{2} \sum_{\path' \in \Pi} \mathds{1}_{p_n(\path') > p_{0}} +
\mathds{1}_{p_n(\path') > p_{0} - \rho_{n}(\path, \path')
\frac{\sigma_{n}( \path')}{\sigma_{n}( \path) }(p_{0} - p_n(\path)) }},
\end{align*}
where $\Phi ( M p_{0}, M, p_n(\path))$ is the cdf of a binomial distribution with parameter 
$p_n(\path)$, $M$ trials, evaluated at $M p_{0}$, and, for all $\path, \path' \in \Pi$, 
\[
\sigma_{n}( \path) = \sqrt{p_n(\path) (1 -  p_n(\path))},
\]
and
\[
 \rho_{n}(\path, \path') = \frac{\emph{Cov}(\mathds{1}_{\path \in T(\Theta, \mathscr{D}_{n})}, 
\mathds{1}_{\path' \in T(\Theta, \mathscr{D}_{n})} | \mathscr{D}_{n} )}
{\sigma_{n}( \path ) \sigma_{n}( \path')}.
\]
\end{theorem}
The proof of Theorem~\ref{stability_equivalent} is to be found in Section \ref{sec_stability_equivalent}.
The equivalent provided in Theorem~\ref{stability_equivalent}  is defined when the sets of rules $\pathset_{M,n,p_{0}}$
and $\pathset_{M,n,p_{0}}'$ are not post-treated. It considerably simplifies the analysis of the asymptotic behavior of $\smash{\mathbb{E} [ \hat{S}_{M,n,p_{0}} | \mathscr{D}_{n}]}$. Since the post-treatment is deterministic, this operation is not an additional source of instability. Then, if the estimation of the rule set without post-treatment is stable, it is also the case when the post-treatment is added. Finally, despite its apparent complexity, the asymptotic approximation of $1 - \mathbb{E} [ \hat{S}_{M,n,p_{0}} | \mathscr{D}_{n}]$ can be easily estimated, and an efficient stopping criterion for the number of trees is therefore derived in (5.1).

\section{Proof of Theorem \ref{theorem_consistency}}
\label{sec_theorem_consistency}

We recall Assumptions (A1)-(A3) and Theorem \ref{theorem_consistency} for the sake of clarity.
\begin{enumerate}
\item[(A1)] The subsampling rate $a_n$ satisfies $\lim\limits_{n \to \infty} a_{n}=\infty$ and $\lim\limits_{n \to \infty} \frac{a_{n}}{n}=0$.
\item[(A2)] The number of trees $M_n$ satisfies $\lim\limits_{n \to \infty} M_n = \infty$.
\item[(A3)] $\bX$ has a strictly positive density $f$ with respect to the Lebesgue measure. Furthermore, for all $j\in \{1,\hdots,p\}$,
 the marginal density $f^{(j)}$ of $X^{(j)}$ is continuous, bounded, and strictly positive.
\end{enumerate}

\setcounter{theorem}{0}
\begin{theorem}
\label{theorem_consistency}
If Assumptions (A1)-(A3) are satisfied, then, for all $\path \in \Pi$, we have
\[
\lim\limits_{n \to \infty} \hat{p}_{M_{n},n}(\path) = p^{\star}(\path) \quad \textrm{in probability.}
\]
\end{theorem}

First, we prove Theorem \ref{theorem_consistency} for a path of one split. The proof is extended for a path of two splits in the next subsection and follows the same steps. Finally, the proof can be easily extended to a path of any depth $d \in \mathbb{N}^{\star}$ by recursion.

\subsection{Proof of Theorem \ref{theorem_consistency} for a path of one split}

We consider $\path_{1}=\{(j_{1},r_{1},s_{1})\}$ a path of one split, where $j_{1} \in\{1, \hdots, p\}$, $r_{1} \in\{1, \hdots, q - 1\}$, and $s_{1} \in\{L,R\}$.
We assume throughout that Assumptions (A1)-(A3) are satisfied. 

Before proving Theorem \ref{theorem_consistency}, we state five lemmas (Lemma \ref{convergence_quantile} to Lemma \ref{convergence_pn_theta_1}). Their proof can be found in the Subsection \ref{subsec_thm1_lemmas}.
Lemma \ref{convergence_quantile} is a preliminary technical result used to state both Lemmas \ref{criterion_residual_1} and \ref{criterion_distribution_1} - case (b).

\begin{lemme}
\label{convergence_quantile}
Let $\bX$ be a random variable distributed on $\R^p$ such that Assumptions (A1) and (A3) are satisfied. Then, for all $ j\in\{1, \hdots, p\}$ and all $r\in\{1, \hdots, q - 1\}$, we have
\[
\lim\limits_{n \to \infty} \sqrt{a_{n}}~ \mathbb{P}\big(q_{r}^{\star (j)}\leq X^{(j)}<\hat{q}_{n,r}^{(j)}\big)=0
\]
and
\[
\lim\limits_{n \to \infty} \sqrt{a_{n}}~ \mathbb{P}\big(\hat{q}_{n,r}^{(j)}\leq X^{(j)}<q_{r}^{\star (j)}\big)=0.
\]
\end{lemme}

Lemma \ref{criterion_residual_1} is used to prove both consistency (Lemma \ref{criterion_consistency_1}) and convergence rate (Lemma \ref{criterion_distribution_1})
 of the CART-splitting criterion when the root node of the tree is cut at an empirical quantile. 
Lemma \ref{convergence_pn_theta_1} is an intermediate result to prove Theorem \ref{theorem_consistency}.
\begin{lemme}
\label{criterion_residual_1}
If Assumptions (A1) and (A3) are satisfied, then for all $j\in\{1, \hdots, p\}$, all $r\in\{1, \hdots, q - 1\}$, and all $H\subseteq\mathds{R}^{p}$ such that
$\mathbb{P}(\bX \in H, X^{(j)} < q^{\star(j)}_{r})>0$ and $\mathbb{P}(\bX \in H, X^{(j)} \geq q^{\star(j)}_{r})>0$, we have
\[
\lim\limits_{n \to \infty} \sqrt{a_{n}}\big(L_{a_{n}}\big(H,\hat{q}_{n,r}^{(j)}\big)
-L_{a_{n}}\big(H,q_{r}^{\star(j)}\big)\big)=0 \quad \textrm{in probability.}
\]
\end{lemme}

\begin{lemme}
\label{criterion_consistency_1}
If Assumptions (A1) and (A3) are satisfied, then for all $j\in\{1, \hdots, p\}$, all $r\in\{1, \hdots, q - 1\}$, and all $H\subseteq\mathds{R}^{p}$ such that
$\mathbb{P}(\bX \in H, X^{(j)} < q^{\star(j)}_{r})>0$ and $\mathbb{P}(\bX \in H, X^{(j)} \geq q^{\star(j)}_{r})>0$, we have
\[
\lim\limits_{n \to \infty} L_{a_{n}}\big(H,\hat{q}_{n,r}^{(j)}\big)=L^{\star}\big(H,q_{r}^{\star(j)}\big) \quad \textrm{in probability.}
\]
\end{lemme}

When splitting a node, if the theoretical CART-splitting criterion has multiple maxima, one is randomly selected. This random selection follows a discrete probability law, which is not necessarily uniform and is based on $\P_{\bX, Y}$ as specified in Definition \ref{def_proba_equality_1}.
In order to derive the limit of the probability that a given split occurs in a $\Theta$-random tree in the empirical algorithm, one needs to assess the convergence rate of the empirical CART-splitting criterion when it has multiple maxima.

\begin{lemme} \label{criterion_distribution_1}
Consider that Assumptions (A1) and (A3) are satisfied.
Let $\bestcut_1 \subset \{1, \hdots, p\} \times \{1, \hdots, q-1\}$ be a set of splits of cardinality $c_1 \geq 2$, such that, for all $(j,r) \in \bestcut_1$,
$L^{\star}(\R^p,  q_{r}^{\star(j)}) \stackrel{\tiny{\rm def}}{=}L^{\star}_{\bestcut_1}$, i.e., the theoretical CART-splitting criterion is constant for all splits in $\bestcut_1$.
Let $(j_1, r_1) \in \bestcut_1$ and let $\textbf{L}_{n,\path_1}^{(\bestcut_1)}$ be a random vector where each component is the difference between the empirical
 CART-splitting criterion for the splits $(j, r) \in \bestcut_1 \setminus (j_1, r_1)$ and $(j_1, r_1)$, that is
\[
\textbf{L}_{n,\path_1}^{(\bestcut_1)} = \Big(
L_{a_{n}}\big(\mathbb{R}^{p},\hat{q}_{n,r}^{(j)}\big)-L_{a_{n}}\big(\mathbb{R}^{p},\hat{q}_{n,r_1}^{(j_1)}\big)\Big)_{(j, r) \in \bestcut_1 \setminus (j_1, r_1)}.
\]

\paragraph{(a)} If $L^{\star}_{\bestcut_1} > 0$, then we have
\[
\sqrt{a_{n}}~ \textbf{L}_{n,\path_1}^{(\bestcut_1)} \stackrel[n\rightarrow\infty]{\mathscr{D}}{\longrightarrow}\mathcal{N}(0, \Sigma),
\]
where, for all $(j,r), (j',r') \in \bestcut_1 \setminus (j_1, r_1)$, each element of the covariance matrix $\Sigma$ is defined by $
\Sigma_{(j,r), (j',r')} = \emph{Cov}[Z_{j,r}, Z_{j',r'}],
$ with
\begin{align*}
Z_{j,r}= \big(Y - & \E[Y | X^{(j_1)}<q_{r_1}^{\star(j_1)}]\mathds{1}_{X^{(j_1)}<q_{r_1}^{\star(j_1)}}
\\ & -\E[Y | X^{(j_1)} \geq q_{r_1}^{\star(j_1)}]\mathds{1}_{X^{(j_1)}\geq q_{r_1}^{\star(j_1)}}\big)^{2}\\
- \big(Y &- \E[Y | X^{(j)}<q_{r}^{\star(j)}]\mathds{1}_{X^{(j)}<q_{r}^{\star(j)}}
- \E[Y | X^{(j)} \geq q_{r}^{\star(j)}]\mathds{1}_{X^{(j)}\geq q_{r}^{\star(j)}}\big)^{2}.
\end{align*}
Besides, for all $(j,r) \in \bestcut_1$, $\mathds{V} [Z_{j,r}] >0$.
\paragraph{(b)} If $L^{\star}_{\bestcut_1} = 0$, then we have
\[
a_{n} \textbf{L}_{n,\path_1}^{(\bestcut_1)} \stackrel[n\rightarrow\infty]{\mathscr{D}}{\longrightarrow} h_{\path_1}(\mathbf{V}),
\]
where $\mathbf{V}$ is a Gaussian vector of covariance matrix $\emph{Cov}[\mathbf{Z}]$. If $\bestcut_1$ is explicitly written
$\bestcut_1 = \{(j_k, r_k)\}_{k=1,\hdots,c_1}$, $\mathbf{Z}$ is defined, for $k \in \{1,\hdots,c_1\}$, by
\begin{align*}
Z_{2k-1} = \frac{1}{\sqrt{p_{L,k}}} (Y - \E[Y]) \mathds{1}_{X^{(j_k)} < q_{r_k}^{\star (j_k)}} \\
Z_{2k} = \frac{1}{\sqrt{p_{R,k}}} (Y - \E[Y]) \mathds{1}_{X^{(j_k)} \geq q_{r_k}^{\star (j_k)}},
\end{align*}
where $p_{L,k} = \P (X^{(j_k)}<q_{r_k}^{\star (j_k)})$, $p_{R,k} = \P (X^{(j_k)} \geq q_{r_k}^{\star (j_k)})$, and $h_{\path_1}$ is a multivariate quadratic form defined as
\begin{align*}
h_{\path_1}: \left( \begin{array}{c} x_1 \\ \vdots \\ x_{2c_1} \end{array} \right) \rightarrow \left( \begin{array}{c}
x_{3}^2 + x_{4}^2 - x_1^2 - x_2^2\\ \vdots \\
x_{2k-1}^2 + x_{2k}^2 - x_{1}^2 - x_{2}^2 \\ \vdots \\
x_{2c_1-1}^2 + x_{2c_1}^2 - x_{1}^2 - x_{2}^2  \end{array}\right).
\end{align*}
Besides, the variance of each component of $h_{\path_1}(\mathbf{V})$ is strictly positive.
\end{lemme}

\begin{mydef}[Theoretical splitting procedure] \label{def_proba_equality_1}
Let $\theta_1^{(V)}$ be the set of eligible variables to split the root node of a theoretical random tree. 
The set of best theoretical cuts at the root node is defined as
\[
\bestcut_{1}^{\star}\big(\theta_1^{(V)}\big) = \underset{(j,r) \in \theta_{1}^{(V)} \times \{1, \hdots, q - 1\}}{\emph{argmax}} L^{\star}\big(\mathbb{R}^{p},q_{r}^{\star(j)}\big).
\]
If $\bestcut_{1}^{\star}(\theta_1^{(V)})$ has multiple elements, then $(j_1, r_1)$ is randomly drawn with probability 
\begin{align} \label{def_probability_equality_1} 
\mathbb{P}\big(\path_{1} \in T^{\star}(\Theta)|\Theta^{(V)}=\theta^{(V)}\big) 
= \Phi_{\theta_1^{(V)}, (j_1, r_1)} (\textbf{0}),
\end{align}
where $\Phi_{\theta_1^{(V)}, (j_1, r_1)}$ is the cdf of the limit law defined in Lemma \ref{criterion_distribution_1} for 
$\bestcut_1 = \bestcut_{1}^{\star}(\theta_1^{(V)})$. 
This definition is extended for the second split in Definition \ref{def_proba_equality_2}.
\end{mydef}

Recall that the randomness in a tree can be decomposed as $\Theta = (\Theta^{(S)}, \Theta^{(V)})$, where $\Theta^{(S)}$ corresponds to the subsampling and $\Theta^{(V)}$ is related to the variable selection. $\Theta^{(V)}$ takes values in the finite set $\Omega^{(V)} = \{1,\hdots,p\}^{3 \times \texttt{mtry}}$.
\begin{lemme} \label{convergence_pn_theta_1}
If Assumptions (A1)-(A3) are satisfied, then for all $\theta^{(V)}\in\Omega^{(V)}$, we have
\[
\lim\limits_{n \to \infty} \mathbb{P}\big(\path_{1}\in T(\Theta,\mathscr{D}_{n})|\Theta^{(V)}
=\theta^{(V)}\big)=\mathbb{P}\big(\path_{1} \in T^{\star}(\Theta)|\Theta^{(V)}=\theta^{(V)}\big).
\]
\end{lemme}

We are now equipped to prove Theorem \ref{theorem_consistency} in the case of one single split. Recall that 
\begin{align}
    \E[\hat{p}_{M_n, n}(\path_1)] = \P (\path_1 \in T(\Theta, \mathscr{D}_{n})).
\end{align}
Since $\Theta^{(V)}$ takes values in the finite set
$\Omega^{(V)}$, according to Lemma \ref{convergence_pn_theta_1}, we have
\begin{align*}
& \lim \limits_{n \to \infty} \P (\path_1 \in T(\Theta, \mathscr{D}_{n})) \\
& \quad = 
\lim \limits_{n \to \infty} \sum_{\theta^{(V)}\in\Omega^{(V)}}\mathbb{P}\big(\path_{1}\in T(\Theta,\mathscr{D}_{n})|\Theta^{(V)}=\theta^{(V)}\big)\mathbb{P}_{\Theta^{(V)}}\big(\Theta^{(V)}=\theta^{(V)}\big)\\
& \quad = \sum_{\theta^{(V)}\in\Omega^{(V)}}\mathbb{P}\big(\path_{1} \in T^{\star}(\Theta)|\Theta^{(V)}=\theta^{(V)}\big)\mathbb{P}_{\Theta^{(V)}}\big(\Theta^{(V)}=\theta^{(V)}\big)\\
& \quad = \P(\path_1 \in T^{\star}(\Theta)).
\end{align*}
Therefore,
\[
\lim_{n \to \infty} \E[\hat{p}_{M_n, n}(\path_1)] = p^{\star}(\path_{1}).
\]

To finish the proof, we just have to show that 
$\lim\limits_{n \to \infty} \mathds{V}[\hat{p}_{M_{n},n}(\path_{1})] = 0$. 

The law of total variance gives
\begin{align}
\mathbb{V} [ \hat{p}_{M_{n}, n} (\path_1)] & = \E \big[ \mathbb{V} [ \hat{p}_{M_{n}, n}(\path_1) |  \mathscr{D}_{n} ] \big] 
+  \mathbb{V} \big[ \E [ \hat{p}_{M_{n}, n} (\path_1) | \mathscr{D}_{n} ] \big] \nonumber \\
& =  \E \Big[ \mathbb{V} \Big[ \frac{1}{M_{n}}\sum_{\ell=1}^{M_{n}}\mathds{1}_{\path_{1} \in T(\Theta_{\ell},\mathscr{D}_{n})}| 
 \mathscr{D}_{n} \Big] \Big]  + \mathbb{V} [ p_n (\path_1)  ] \nonumber \\
 & =  \frac{1}{M_{n}} \E \big[ \mathbb{V} [ \mathds{1}_{\path_{1} \in T(\Theta_1,\mathscr{D}_{n})}| 
 \mathscr{D}_{n} ] \big] + \mathbb{V} [ p_n(\path_1)  ] \nonumber \\
 & =  \frac{1}{M_{n}} \E \big[p_n(\path_1) -p_n(\path_1)^{2} \big] + \mathbb{V} [p_n(\path_1)], \nonumber \\
& = \frac{1}{M_{n}} \E[p_n(\path_1)] ( 1 - \E[p_n(\path_1)] ) +  \Big( 1 - \frac{1}{M_{n}} \Big) \mathbb{V} [p_n(\path_1) ] . \label{proof_thm1_variance}
 \end{align}
Following the approach of \cite{mentchquantifying2016}, $p_n(\path_{1})$ 
is a complete infinite order U-statistic with the kernel $
\E [ \mathds{1}_{\path_{1} \in T ( \Theta, \mathscr{D}_{n} )} | \Theta^{( S )},  \mathscr{D}_{n} ] $. From  \cite{Hoeffding1948},
\[
\mathbb{V} [p_n(\path_1)] \leq \frac{a_{n}}{n} \xi_{a_{n}, a_{n}},
\]
where
$\xi_{a_{n}, a_{n}} = \mathbb{V} [ \E [ \mathds{1}_{\path_{1} \in T ( \Theta, \mathscr{D}_{n} )} 
| \Theta^{( S)}, \mathscr{D}_{n} ] | \Theta^{( S )} ]$. 
Since $\xi_{a_{n}, a_{n}}$ is bounded and $\lim \limits_{n \to \infty} \frac{a_{n}}{n} = 0$,
\[
\lim \limits_{n \to \infty}\mathbb{V} [p_n(\path_1)]  = 0.
\]
Using equality (\ref{proof_thm1_variance}), since $ p_n(\path_1)$ is bounded and $\lim \limits_{n \to \infty} M_{n} = \infty$,
\[
\lim \limits_{n \to \infty}\mathbb{V} [ p_{M{_n}, n}(\path_1)]  = 0.
\]
Finally, 
\begin{align*}
\lim \limits_{n \to \infty} \E\big[&(\hat{p}_{M_{n}, n} (\path_1) - p^{\star}(\path_1) )^2\big] \\ & = \lim \limits_{n \to \infty} \mathbb{V}[\hat{p}_{M_{n}, n} (\path_1)] + \big(\E[\hat{p}_{M_{n}, n} (\path_1)] - p^{\star}(\path_1)\big)^2 = 0,
\end{align*}
which concludes the proof.

\subsection{Proof of Theorem \ref{theorem_consistency} for a path of two split}

The proof of Theorem \ref{theorem_consistency} is extended for a path of two splits. We consider $\path_{1}=\{(j_{1},r_{1},s_{1})\}$ a path of one split and
$\path_{2}=\{(j_{k},r_{k},s_{k}), k=1,2\}$ a path of
two splits, where $j_{1}, j_{2}\in\{1, ..., p\}$, $r_{1},
r_{2}\in\{1, ..., q - 1\}$ and $s_{1}, s_{2} \in\{L,R\}$.
We assume assumptions (A1)-(A3) are satisfied.

The path $\path_2 = \{ (j_1, r_1, s_1), (j_2, r_2, s_2) \}$ can occur in trees where the split at the root node is $(j_1, r_1)$
and the split of one of the child node is $(j_2, r_2)$, and in trees where the splits are made in the reversed order, $(j_2, r_2)$
at the root node and $(j_1, r_1)$ at one of the child node. Since these two events are disjoint, $\P\big(\path_{2} \in T(\Theta,\mathscr{D}_{n})\big)$
is the sum of the probability of these two events. Without loss of generality, we will consider in the entire proof that the split at the root node is
$(j_1, r_1)$. Lemmas \ref{criterion_residual_2} - \ref{convergence_pn_theta_2} below extend Lemmas \ref{criterion_residual_1} - \ref{convergence_pn_theta_1} to the case where the tree path contains two splits. 

We need to introduce additional notations, first, the theoretical hyperrectangle based on a path $\path$ by
\[
H^{\star}(\path)=\left\{ \bx \in\mathds{R}^{p}:\begin{cases}
x^{(j_{k})}<q_{r_{k}}^{\star(j_{k})} & \textrm{if }s_{k}=L\\
x^{(j_{k})}\geq q_{r_{k}}^{\star(j_{k})} & \textrm{if }s_{k}=R
\end{cases}\;,k\in 1,\hdots,d \right\},
\]
 with $d \in \{1,2\}$, the empirical counterpart of $\hat{H}_n (\path)$ defined in (2.3).
Furthermore, since from assumption (A3), $\bX$ has a strictly positive density, then for $j \in \{1, ..., p\} \setminus j_1$, and $r \in \{1, ..., q - 1\}$,
$\mathbb{P}\big(\bX \in H^{\star}(\path_{1}), X^{(j)} < q^{\star(j)}_{r}\big)>0$ and $\mathbb{P}\big(\bX \in H^{\star}(\path_{1}), X^{(j)} \geq q^{\star(j)}_{r}\big)>0$. When $j = j_1$, the second cut is performed along the same direction as the first one. 
In that case, depending on the side $s_1$ of the first cut and the cut positions $r_1$ and $r$, one of the two child node can be empty with probability one.
For example, the hyperrectangle associated to the path $\{ (1, 2, L), (1, 3, R) \}$ is empty.
In SIRUS, such splits are not considered to find the best cut for a node at the second level of the tree.
Thus we define $\bestcut_{\path_1}$ the set of possible splits for the second cut
\begin{align*}
\bestcut_{\path_1} = \{&(j,r), j \in \{1, ..., p\} \setminus j_1, r \in \{1, ..., q - 1\} \} \\
&\cup \{(j_1,r), \textrm{  s.t.  } r < r_1 \textrm{  if  } s_1 = L, \textrm{  and  } r > r_1 \textrm{  if  } s_1 = R \},
\end{align*}
and $\bestcut_{\path_1}\big(\theta_2^{(V)}\big) = \big\{(j,r) \in \bestcut_{\path_1} \textrm{  s.t.  } j \in \theta_2^{(V)}\big\}$ when the split directions are restricted to $\theta_2^{(V)} \subset \{1,...,p\}$.

\begin{lemme}
\label{criterion_residual_2}
If Assumptions (A1) and (A3) are satisfied, then for all $(j, r) \in \bestcut_{\path_1}$, we have
\begin{align*}
\lim\limits_{n \to \infty}  \sqrt{a_{n}}\big(L_{a_{n}}\big(\hat{H}_{n}(\path_{1}),\hat{q}_{n,r}^{(j)}\big)
-L_{a_{n}}\big(H^{\star}(\path_{1}),q_{r}^{\star(j)}\big)\big)=0 \quad \textrm{in probability.}
\end{align*}
\end{lemme}

\begin{lemme}
\label{criterion_consistency_2}
If Assumptions (A1) and (A3) are satisfied, then for all $(j, r) \in \bestcut_{\path_1}$, we have
\begin{align*}
\lim\limits_{n \to \infty}  L_{a_{n}}\big(\hat{H}_{n}(\path_{1}),\hat{q}_{n,r}^{(j)}\big)=L^{\star}\big(H^{\star}(\path_{1}),q_{r}^{\star(j)}\big) \quad \textrm{in probability.}
\end{align*}
\end{lemme}

\begin{lemme} \label{criterion_distribution_2}
Consider that Assumptions (A1) and (A3) are satisfied.
Let $\bestcut_1 \subset \{1, ..., p\} \times \{1, ..., q-1\}$ and $\bestcut_2 \subset \bestcut_{\path_1}$ be two sets of splits of cardinality $c_1 \geq 1$ and $c_2 \geq 2$, such that the theoretical CART-splitting criterion is constant for all splits in $\bestcut_1$ on one hand, and in $\bestcut_2$ on the other hand, i.e.,
\begin{align*}
\forall l \in \{1,2\}, \quad \forall (j,r) \in \bestcut_l, \quad L^{\star}\big(H_l,  q_{r}^{\star(j)}\big) \stackrel{\tiny{\rm def}}{=}L^{\star}_{\bestcut_l},
\end{align*}
where $H_1 = \R^p$ and $H_2 = H^{\star}(\path_1)$. Let $(j_1, r_1) \in \bestcut_1$, $(j_2, r_2) \in \bestcut_2$, and let $\textbf{L}_{n,\path_2}^{(\bestcut_1, \bestcut_2)}$ a the random vector where each component is the difference between the empirical CART-splitting criterion for the splits $(j, r) \in \bestcut_1 \setminus (j_1, r_1)$ and $(j_1, r_1)$ for the first $c_1 - 1$ components, and for the splits $(j, r) \in \bestcut_2 \setminus (j_2, r_2)$ and $(j_2, r_2)$ for the remaining $c_2 - 1$ components,  that is
\begin{align*}
\textbf{L}_{n,\path_2}^{(\bestcut_1, \bestcut_2)} = \left(\begin{array}{c}
\big[L_{a_{n}}\big(\mathbb{R}^{p},\hat{q}_{n,r}^{(j)}\big)-L_{a_{n}}\big(\mathbb{R}^{p},\hat{q}_{n,r_1}^{(j_1)}\big)\big]_{(j, r) \in \bestcut_1 \setminus (j_1, r_1)} \\
\big[L_{a_{n}}\big(\hat{H}_n(\path_{1}),\hat{q}_{n,r}^{(j)}\big)-L_{a_{n}}\big(\hat{H}_n(\path_{1}),\hat{q}_{n,r_2}^{(j_2)}\big)\big]_{(j, r) \in \bestcut_2 \setminus (j_2, r_2)} \\
\end{array}\right).
\end{align*}

\paragraph{(a)} If $L^{\star}_{\bestcut_1} > 0$ and $L^{\star}_{\bestcut_2} > 0$, then we have
\begin{align*}
\sqrt{a_{n}} \textbf{L}_{n,\path_2}^{(\bestcut_1, \bestcut_2)} \stackrel[n\rightarrow\infty]{\mathscr{D}}{\longrightarrow}\mathcal{N}(0, \Sigma)
\end{align*}
where for $l, l' \in \{1,2\}$, for all $(j,r) \in \bestcut_l \setminus (j_l,r_l)$,
$(j', r') \in \bestcut_{l'} \setminus (j_{l'}, r_{l'})$, each element of the covariance matrix $\Sigma$ is defined by
$ \Sigma_{(j,r,l), (j',r',l')} = \textrm{Cov}[Z_{j,r,l}, Z_{j',r',l'}],$ with
\begin{align*}
Z_{j,r,l}= & \frac{1}{\P(\bX \in H_l)} \big(Y - \mu_{L,r_l}^{(j_l)}\mathds{1}_{X^{(j_{l})}<q_{r_{l}}^{\star(j_{l})}}
-\mu_{R,r_l}^{(j_l)}\mathds{1}_{X^{(j_{l})}\geq q_{r_{l}}^{\star(j_{l})}}\big)^{2} \mathds{1}_{\bX \in H_l}  \\
&- \frac{1}{\P(\bX \in H_l)} \big(Y - \mu_{L,r}^{(j)}\mathds{1}_{X^{(j)}<q_{r}^{\star(j)}}
-\mu_{R,r}^{(j)}\mathds{1}_{X^{(j)}\geq q_{r}^{\star(j)}}\big)^{2} \mathds{1}_{\bX \in H_l},
\end{align*}
$\mu_{L,r}^{(j)} = \E\big[Y | X^{(j)}<q_{r}^{\star(j)}, \bX \in H_l\big]$, 
$\mu_{R,r}^{(j)} = \E\big[Y | X^{(j)} \geq q_{r}^{\star(j)}, \bX \in H_l\big]$. 
Besides, for all $l \in \{1,2\}$ and for all $(j,r) \in C_l$, $\mathds{V}[Z_{j,r,l}] > 0$.

\paragraph{(b)} If $L^{\star}_{\bestcut_1} = L^{\star}_{\bestcut_2} = 0$, then we have
\begin{align*}
a_{n} \textbf{L}_{n,\path_2}^{(\bestcut_1, \bestcut_2)} \stackrel[n\rightarrow\infty]{\mathscr{D}}{\longrightarrow} h_{\path_2}(\mathbf{V}),
\end{align*}
where $\mathbf{V}$ is a gaussian vector of covariance matrix $\textrm{Cov}[\mathbf{Z}]$. If $\bestcut_1$ and $\bestcut_2$ are explicitly written
$\bestcut_1 = \{(j_k, r_k)\}_{k \in J_1}$, and $\bestcut_2 = \{(j_k, r_k)\}_{k \in J_2}$, with $J_1 = \{1,...,c_1+1\} \setminus 2$
and $J_2 = \{2\} \cup \{c_1+2,...,c_1+c_2\}$, $\mathbf{Z}$ is defined, for $l \in \{1,2\}$ and $k \in J_l$, by
\begin{align*}
Z_{2k-1} = \frac{1}{\sqrt{p_{L,k}\P(\bX \in H_l)}} (Y - \E[Y|\bX \in H_l]) \mathds{1}_{X^{(j_k)} < q_{r_k}^{\star (j_k)}}\mathds{1}_{\bX \in H_l} \\
Z_{2k} = \frac{1}{\sqrt{p_{R,k}\P(\bX \in H_l)}} (Y - \E[Y|\bX \in H_l]) \mathds{1}_{X^{(j_k)} \geq q_{r_k}^{\star (j_k)}}\mathds{1}_{\bX \in H_l},
\end{align*}
where $p_{L,k} = \P \big(X^{(j_k)}<q_{r_k}^{\star (j_k)}, \bX \in H_l\big)$, $p_{R,k} = \P \big(X^{(j_k)} \geq q_{r_k}^{\star (j_k)}, \bX \in H_l\big)$, and $h_{\path_2}$ is a multivariate quadratic form defined as
\begin{align*}
h_{\path_2}: \left( \begin{array}{c} x_1 \\ \vdots \\ x_{2(c_1 + c_2)}  \end{array} \right) \rightarrow \left( \begin{array}{c}
x_{5}^2 + x_{6}^2 - x_1^2 - x_2^2\\ \vdots \\
x_{2c_1+1}^2 + x_{2c_1+2}^2 - x_{1}^2 - x_{2}^2\\ 
x_{2c_1+3}^2 + x_{2c_1+4}^2 - x_{3}^2 - x_{4}^2\\ \vdots \\
x_{2(c_1+c_2)-1}^2 + x_{2(c_1+c_2)}^2 - x_{3}^2 - x_{4}^2  \end{array}\right).
\end{align*}
Besides, the variance of each component of $h_{\path_2}(\mathbf{V})$ is strictly positive.

\paragraph{(c)} If $L^{\star}_{\bestcut_1} > 0$ and $L^{\star}_{\bestcut_2} = 0$, then we have
\begin{align*}
a_{n} \textbf{L}_{n,\path_2}^{(\bestcut_1, \bestcut_2)} \stackrel[n\rightarrow\infty]{\mathscr{D}}{\longrightarrow} h'_{\path_2}(\mathbf{V}),
\end{align*}
where $\mathbf{V}$ is a gaussian vector of covariance matrix $\textrm{Cov}[\mathbf{Z}]$, and $\mathbf{Z}$ is defined as,
for $k \in J_1$,
\begin{align*}
Z_{2k - 1}&=  \big( Y - \E\big[Y | X^{(j_k)}<q_{r_k}^{\star(j_k)}\big] \big)^{2} \mathds{1}_{X^{(j_{k})}<q_{r_{k}}^{\star(j_{k})}} \\
Z_{2k }&=  \big( Y - \E\big[Y | X^{(j_k)} \geq q_{r_k}^{\star(j_k)}\big] \big)^{2} \mathds{1}_{X^{(j_{k})}\geq q_{r_{k}}^{\star(j_{k})}},
\end{align*}
for $k \in J_2$,
\begin{align*}
Z_{2k - 1}&= \frac{Y - \E[Y|\bX \in H^{\star}(\path_1)]}{\sqrt{p_{L,k}\P(\bX \in H^{\star}(\path_1))}} \mathds{1}_{X^{(j_k)} < q_{r_k}^{\star (j_k)}, \bX \in H^{\star}(\path_1)} \\
Z_{2k}&= \frac{Y - \E[Y|\bX \in H^{\star}(\path_1)]}{\sqrt{p_{R,k}\P(\bX \in H^{\star}(\path_1))}} \mathds{1}_{X^{(j_k)} \geq q_{r_k}^{\star (j_k)}, \bX \in H^{\star}(\path_1)}, 
\end{align*}
and $h'_{\path_2}$ is a multivariate quadratic form defined as
\begin{align*}
h'_{\path_2}: \left( \begin{array}{c} x_1 \\ \vdots \\ x_{2(c_1 + c_2)}  \end{array} \right) \rightarrow \left( \begin{array}{c}
x_1 + x_2 - x_5 - x_6\\ \vdots \\
x_1 + x_2 - x_{2c_1+1} - x_{2c_1 + 2}\\ 
x_{2c_1+3}^2 + x_{2c_1+4}^2 - x_{3}^2 - x_{4}^2\\ \vdots \\
x_{2(c_1+c_2)-1}^2 + x_{2(c_1+c_2)}^2 - x_{3}^2 - x_{4}^2  \end{array}\right).
\end{align*}
Besides, the variance of each component of $h'_{\path_2}(\mathbf{V})$ is strictly positive.

\paragraph{(d)} $L^{\star}_{\bestcut_1} = 0$ and $L^{\star}_{\bestcut_2} > 0$. Symmetric to case (c).
\end{lemme}

\begin{mydef}[Theoretical splitting procedure at children nodes] \label{def_proba_equality_2}
Let $\theta^{(V)} = (\theta_1^{(V)}, \theta_2^{(V)}, \cdot) \in \Omega^{(V)}$ be the sets of eligible variables to split the nodes of a theoretical random tree. The set of best theoretical cuts at the left children node along the variables in $\theta^{(V)}_2$ is defined as  
\begin{align*}
\bestcut_{2}^{\star}\big(\theta_{2}^{(V)}\big) = \underset{ (j,r) \in \bestcut_{\path_1}\big(\theta_2^{(V)}\big)}{\textrm{argmax}}
 L^{\star}\big(H^{\star}(\path_{1}),q_{r}^{\star(j)}\big).    
\end{align*}
If $\bestcut_{2}^{\star}\big(\theta_2^{(V)}\big)$ has multiple elements, then $(j_2, r_2)$ is randomly drawn with probability 
\begin{align} \label{def_probability_equality_2} 
\mathbb{P}\big(\path_{2} \in T^{\star}(\Theta)|\Theta^{(V)}=\theta^{(V)}\big) 
= \frac{\Phi_{\path_1, \theta^{(V)}, (j_2, r_2)} (\textbf{0})}
{\mathbb{P}\big(\path_{1} \in T^{\star}(\Theta)|\Theta^{(V)}=\theta^{(V)}\big)},
\end{align}
where $\mathbb{P}\big(\path_{1} \in T^{\star}(\Theta)|\Theta^{(V)}=\theta^{(V)}\big)$ is defined
from Definition \ref{def_proba_equality_1}, and $\Phi_{\path_1, \theta^{(V)}, (j_2, r_2)}$ is the cdf of the limit law defined in Lemma \ref{criterion_distribution_2} for $\bestcut_1 = \bestcut_{1}^{\star}\big(\theta_1^{(V)}\big)$ and
$\bestcut_2 = \bestcut_{2}^{\star}\big(\theta_2^{(V)}\big)$.
\end{mydef}

\begin{lemme}
\label{convergence_pn_theta_2}
If Assumptions (A1)-(A3) are satisfied, then for all $\theta^{(V)}\in\Omega^{(V)}$, we have
\begin{align*}
\lim\limits_{n \to \infty} \mathbb{P}\big(\path_{2} \in T(\Theta,\mathscr{D}_{n})|\Theta^{(V)}
=\theta^{(V)}\big)=\mathbb{P}\big(\path_{2} \in T^{\star}(\Theta)|\Theta^{(V)}=\theta^{(V)}\big)
\end{align*}
\end{lemme}
Finally, the proof of Theorem \ref{theorem_consistency} in the two-splits scenario is the same as in the single-split scenario.

\subsection{Proofs of intermediate lemmas} \label{subsec_thm1_lemmas}

\begin{proof}[Proof of Lemma \ref{convergence_quantile}]
\label{proof_lemma_convergence_quantile}

Set $j\in\{1, ..., p\} $, and $r\in\{1, ..., q - 1\}$. We define  the marginal cumulative distribution function $F^{(j)}$ of $X^{(j)}$,
 $F^{(j)}(x)=\mathbb{P}\big(X^{(j)}<x\big)$, and $F_{n}^{(j)}$ the empirical c.d.f.
\begin{align*}
F_{n}^{(j)}(x)=\frac{1}{n}\sum_{i=1}^{n}\mathds{1}_{X^{(j)}_{i}\leq x}.
\end{align*}
We adapt an inequality from \cite{serfling2009approximation} (section 2.3.2 page 75)
to bound the following conditional probability for all $\varepsilon>0$
\begin{align}
\P\big(q_{r}^{\star(j)}&\leq X_{1}^{(j)} < \hat{q}_{n,r}^{(j)} | X_1^{(j)} = q_{r}^{\star(j)} + \varepsilon \big) \nonumber \\
=\P\big(&q_{r}^{\star(j)} + \varepsilon < \hat{q}_{n,r}^{(j)} | X_1^{(j)} = q_{r}^{\star(j)} + \varepsilon \big) \nonumber \\
\leq \P\big(& F_{n}^{(j)}\big(q_{r}^{\star(j)} + \varepsilon \big)
\leq F_{n}^{(j)}\big(\hat{q}_{n,r}^{(j)}\big) | X_1^{(j)} = q_{r}^{\star(j)} + \varepsilon \big) \nonumber \\
\leq \mathbb{P}\Big(&1 + \sum_{i=2}^{n}\mathds{1}_{X_{i}^{(j)}\leq q_{r}^{\star(j)} + \varepsilon}
\leq\left\lceil \frac{n.r}{q}\right\rceil \Big) \nonumber \\
\leq \mathbb{P}\Big(&\sum_{i=2}^{n}\mathds{1}_{X_{i}^{(j)}\leq q_{r}^{\star(j)} + \varepsilon}
- (n - 1) F^{(j)}\big(q_{r}^{\star(j)} + \varepsilon \big) \\
&\leq\left\lceil \frac{n.r}{q}\right\rceil - 1 - (n - 1) F^{(j)}\big(q_{r}^{\star(j)} + \varepsilon\big)\Big) \label{eq_hoeff}
\end{align}

Since $f$ is continuous and strictly positive, there exists three constants $c_1, c_2, \eta >0$ such that for all $x \in [q_{r}^{\star(j)}, q_{r}^{\star(j)} + \eta]$, 
$c_1 \leq f^{(j)}(x) \leq c_2$. Thus, for all $\varepsilon < \eta$, we have
\begin{align*}
 F^{(j)}\big(q_{r}^{\star(j)} + \varepsilon \big) - F^{(j)}\big(q_{r}^{\star(j)} \big)  & =  \int_{q_{r}^{\star(j)} }^{q_{r}^{\star(j)} + \varepsilon} f^{(j)}(x) \textrm{d} x,
\end{align*}
which leads to
\begin{align*}
c_1 \varepsilon \leq F^{(j)}\big(q_{r}^{\star(j)} + \varepsilon \big) - F^{(j)}\big(q_{r}^{\star(j)} \big) \leq c_2 \varepsilon.
\end{align*}
Consequently,
\begin{align*}
\left\lceil \frac{n.r}{q}\right\rceil - 1 - &(n - 1) F^{(j)}\big(q_{r}^{\star(j)} + \varepsilon\big) \\
& \leq \left\lceil \frac{n.r}{q}\right\rceil - 1 - (n - 1) \big( c_1 \varepsilon + F^{(j)}\big(q_{r}^{\star(j)}\big) \big) \\
& \leq \left\lceil \frac{n.r}{q}\right\rceil - 1 - (n - 1) c_1 \varepsilon -  \frac{(n-1).r}{q} \\
& \leq 1 - (n - 1) c_1 \varepsilon.
\end{align*}
For $n > 1 + \frac{1}{c_1 \varepsilon}$, we can apply Hoeffding inequality to \ref{eq_hoeff}, 
\begin{align}
\P\big(&q_{r}^{\star(j)}\leq X_{1}^{(j)} < \hat{q}_{n,r}^{(j)} | X_1^{(j)} = q_{r}^{\star(j)} + \varepsilon \big) \nonumber \\
& \leq \mathbb{P}\big(\sum_{i=2}^{n}\mathds{1}_{X_{i}^{(j)}\leq q_{r}^{\star(j)} + \varepsilon}
- (n - 1) F^{(j)}\big(q_{r}^{\star(j)} + \varepsilon \big)
\leq 1 - (n - 1) c_1 \varepsilon \big) \nonumber \\
& \leq e^{-\frac{2}{n}\big(1 - (n - 1) c_1  \varepsilon \big)^{2}} \nonumber \\
& \leq C e^{- 2 n c_1^{2}  \varepsilon^{2}}, \label{bound_qtl}
\end{align}
where $C = e^{2 c_1 \eta (1 + 2 c_1 \eta)}$. By definition, we have
\begin{align*}
\P\big(q_{r}^{\star(j)}\leq X_{1}^{(j)} < \hat{q}_{n,r}^{(j)})
= \int_{]0,\infty[}\P\big(&q_{r}^{\star(j)}\leq X_{1}^{(j)} < \hat{q}_{n,r}^{(j)} | X_1^{(j)} = q_{r}^{\star(j)} + \varepsilon \big) \\
& \times f^{(j)}\big(q_{r}^{\star(j)} + \varepsilon\big) d\varepsilon.
\end{align*}
To bound the previous integral, we break it down in three parts. Since $f^{(j)}$ is bounded by $c_2$ on $[q_{r}^{\star(j)}, q_{r}^{\star(j)} + \eta]$,
 for $n > 1 + \frac{1}{c_1 \eta}$ we use inequality \ref{bound_qtl} to get
\begin{align*}
\P\big(q_{r}^{\star(j)}\leq X_{1}^{(j)} < \hat{q}_{n,r}^{(j)}\big)
\leq &\int_{]0,\frac{1}{(n - 1) c_1}]} c_2 d\varepsilon \\
&+ \int_{]\frac{1}{(n - 1) c_1}, \eta[} c_2 C e^{- 2 n c_1^{2}  \varepsilon^{2}}  d\varepsilon \\
&+ \int_{[\eta, \infty[}C e^{- 2 n c_1^{2}  \eta^{2}} f^{(j)}\big(q_{r}^{\star(j)} + \varepsilon\big)  d\varepsilon. \\
\end{align*}
In the second integral, we introduce the following change of variable $u =\sqrt{2n}c_1\varepsilon$
\begin{align*}
\int_{]\frac{1}{(n - 1) c_1}, \eta[} c_2 C e^{- 2 n c_1^{2}  \varepsilon^{2}}  d\varepsilon 
&= \frac{c_2 C}{c_1 \sqrt{2n}} \int_{]\frac{\sqrt{2n}}{(n - 1)}, \sqrt{2n} c_ 1 \eta[} e^{- u^2}  du \\
&\leq \frac{c_2 C}{c_1 \sqrt{2n}} \int_{]0,\infty[} e^{- u^2}  du 
\leq \frac{\sqrt{\pi} c_2 C}{2 c_1 \sqrt{2n}},
 \end{align*}
and therefore we can write
\begin{align*}
\sqrt{a_n} \P\big(q_{r}^{\star(j)}\leq X_{1}^{(j)} < \hat{q}_{n,r}^{(j)}\big)
\leq \frac{c_2 \sqrt{a_n} }{(n - 1) c_1} + \frac{\sqrt{\pi a_n } c_2 C}{2 c_1 \sqrt{2n}}
+ C \sqrt{a_n}  e^{- 2 n c_1^{2}  \eta^{2}}
\end{align*}

From Assumption (A1), $\lim \limits_{n \to \infty} \frac{a_n}{n} = 0$, and then
$$
\lim_{n \to \infty} \sqrt{a_n}~ \mathbb{P}\big(q_{r}^{\star(j)}\leq X_{1}^{(j)}<\hat{q}_{n,r}^{(j)}\big)  = 0.
$$
The case $\lim \limits_{n \to \infty} \sqrt{a_{n}} \mathbb{P}\big(\hat{q}_{n,r}^{(j)}\leq X_{1}^{(j)}
<q_{r}^{\star(j)}\big) = 0$ is similar.
\end{proof}

\subsubsection{Case 1: $\path_{1}$}

\begin{proof}[Proof of Lemma \ref{criterion_residual_1}]
Let $j\in\{1, ..., p\} $, $r\in\{1, ..., q - 1\} $,
and $H\subseteq\mathds{R}^{p}$ such that $\mathbb{P}\big(\bX \in H, X^{(j)} < q^{\star(j)}_{r}\big)>0$
 and $\mathbb{P}\big(\bX \in H, X^{(j)} \geq q^{\star(j)}_{r}\big)>0$. Let
 
\begin{align*}
 \Delta_{n,r}^{(j)} & =\sqrt{a_{n}}\big(L_{a_{n}}\big(H,\hat{q}_{n,r}^{(j)}\big)-L_{a_{n}}\big(H,q_{r}^{\star(j)}\big)\big) 
\end{align*}
that is
\begin{align*}
 \Delta_{n,r}^{(j)}  & =-\frac{\sqrt{a_{n}}}{N_{n}(H)}\big[\sum_{i=1}^{a_{n}}\big(Y_{i}-\overline{Y}_{H_{L}}\mathds{1}_{X_{i}^{(j)}<\hat{q}_{n,r}^{(j)}}-\overline{Y}_{H_{R}}\mathds{1}_{X_{i}^{(j)}\geq\hat{q}_{n,r}^{(j)}}\big)^{2}\mathds{1}_{\bX_{i}\in H}\\
& \quad -\sum_{i=1}^{a_{n}}\big(Y_{i}-\overline{Y}_{H_{L}^{\star}}\mathds{1}_{X_{i}^{(j)}<q_{r}^{\star(j)}}-\overline{Y}_{H_{R}^{\star}}\mathds{1}_{X_{i}^{(j)}\geq q_{r}^{\star(j)}}\big)^{2}\mathds{1}_{\bX_{i}\in H}\big]
\end{align*}
where, for a generic hyperrectangle $H$, we define $N_n(H) = \sum_{i=1}^{a_n} \mathds{1}_{\bX_i \in H}$,  and
\begin{align*}
H_{L}=\big\{ \bx \in H:x^{(j)}<\hat{q}_{n,r}^{(j)}\big\} \quad \textrm{and} \quad  \overline{Y}_{H_{L}}=\frac{1}{N_{n}(H_{L})}\sum_{i=1}^{a_{n}}Y_{i}\mathds{1}_{X_{i}^{(j)}<\hat{q}_{n,r}^{(j)}}\mathds{1}_{\bX_{i}\in H},
\end{align*}
with the convention $\overline{Y}_{H_L} = 0$ if $H_L$ is empty.
The theoretical quantities $H_L^{\star}$ and  $\overline{Y}_{H_{L}^{\star}}$ are
defined similarly by replacing the empirical quantile by its population version. We define symmetrically $H_R$, $H_R^{\star}$, $\overline{Y}_{H_{R}}$,
$\overline{Y}_{H_{R}^{\star}}$. 

Simple calculations show that \begin{align}
\Delta_{n,r}^{(j)}& =\frac{\sqrt{a_{n}}}{N_{n}(H)}\big(\overline{Y}_{H_{L}}^{2}N_{n}(H_{L})-\overline{Y}_{H_{L}^{\star}}^{2}N_{n}(H_{L}^{\star})\big) \nonumber \\
& \quad +\frac{\sqrt{a_{n}}}{N_{n}(H)}\big(\overline{Y}_{H_{R}}^{2}N_{n}(H_{R})-\overline{Y}_{H_{R}^{\star}}^{2}N_{n}(H_{R}^{\star})\big)
\label{proof_lemme2_eq1}
\end{align}

The first term in equation (\ref{proof_lemme2_eq1}) can be rewritten as
\begin{align*}
&\frac{\sqrt{a_{n}}}{N_{n}(H\big)}(\overline{Y}_{H_{L}}^{2}N_{n}(H_{L})-\overline{Y}_{H_{L}^{\star}}^{2}N_{n}(H_{L}^{\star})\big)\\
& \quad =\frac{\sqrt{a_{n}}}{N_{n}(H)N_{n}(H_{L})N_{n}(H_{L}^{\star})}\sum_{i,k,l=1}^{a_{n}}Y_{i}Y_{k} 
\mathds{1}_{\bX_{i}\in H,\bX_{k}\in H} \\ & \qquad \times \big(\mathds{1}_{X_{l}^{(j)}<q_{r}^{\star(j)}}\mathds{1}_{X_{i}^{(j)}<\hat{q}_{n,r}^{(j)}}\mathds{1}_{X_{k}^{(j)}<\hat{q}_{n,r}^{(j)}} - \mathds{1}_{X_{l}^{(j)}<\hat{q}_{n,r}^{(j)}}\mathds{1}_{X_{i}^{(j)}<q_{r}^{\star(j)}}\mathds{1}_{X_{k}^{(j)}<q_{r}^{\star(j)}}\big).
\end{align*}
Since $Y_{i}\in \{ 0,1\}$, we have the following bound

\begin{align*}
& \frac{\sqrt{a_{n}}}{N_{n}(H)}\big|\overline{Y}_{H_{L}}^{2}N_{n}(H_{L})-\overline{Y}_{H_{L}^{\star}}^{2}N_{n}(H_{L}^{\star})\big| \nonumber \\
&  \leq\frac{\sqrt{a_{n}}}{N_{n}(H)N_{n}(H_{L})N_{n}(H_{L}^{\star})}\sum_{i,k,l=1}^{a_{n}}\big|\mathds{1}_{X_{l}^{(j)}<q_{r}^{\star(j)}}\mathds{1}_{X_{i}^{(j)}<\hat{q}_{n,r}^{(j)}}\mathds{1}_{X_{k}^{(j)}<\hat{q}_{n,r}^{(j)}} \\
& \quad -\mathds{1}_{X_{l}^{(j)}<\hat{q}_{n,r}^{(j)}}\mathds{1}_{X_{i}^{(j)}<q_{r}^{\star(j)}}\mathds{1}_{X_{k}^{(j)}<q_{r}^{\star(j)}}\big|,
\end{align*}
and finally
\begin{align}
\label{delta_bound_1}
\frac{\sqrt{a_{n}}}{N_{n}(H)}\big|\overline{Y}_{H_{L}}^{2}N_{n}(H_{L})-\overline{Y}_{H_{L}^{\star}}^{2}N_{n}(H_{L}^{\star})\big|
& \leq\frac{a_{n}^{3}}{N_{n}(H)N_{n}(H_{L})N_{n}(H_{L}^{\star})}W_{n,r}^{(j)},
\end{align}
where
\begin{align} \label{proof_lemma2_eq2_W}
W_{n,r}^{(j)}=\frac{\sqrt{a_{n}}}{a_{n}^{3}}\sum_{i,k,l=1}^{a_{n}}\big|&\mathds{1}_{X_{l}^{(j)}<q_{r}^{\star(j)}}\mathds{1}_{X_{i}^{(j)}<\hat{q}_{n,r}^{(j)}}\mathds{1}_{X_{k}^{(j)}<\hat{q}_{n,r}^{(j)}}\\&-\mathds{1}_{X_{l}^{(j)}<\hat{q}_{n,r}^{(j)}}\mathds{1}_{X_{i}^{(j)}<q_{r}^{\star(j)}}\mathds{1}_{X_{k}^{(j)}<q_{r}^{\star(j)}}\big|. \nonumber
\end{align}
A close inspection of the terms inside the sum of (\ref{proof_lemma2_eq2_W}) reveals that 
\begin{align*}
\mathbb{E}\big[W_{n,r}^{(j)}\big] & \leq\frac{\sqrt{a_{n}}}{a_{n}^{3}}\sum_{i,k,l=1}^{a_{n}}\mathbb{P}\big(\hat{q}_{n,r}^{(j)}\leq X_{i}^{(j)}<q_{r}^{\star(j)}\big)+\mathbb{P}\big(\hat{q}_{n,r}^{(j)}\leq X_{k}^{(j)}<q_{r}^{\star(j)}\big)\\
& \quad +\mathbb{P}\big(q_{r}^{\star(j)}\leq X_{l}^{(j)}<\hat{q}_{n,r}^{(j)}\big)+\mathbb{P}\big(q_{r}^{\star(j)}\leq X_{i}^{(j)}<\hat{q}_{n,r}^{(j)}\big)\\
& \quad +\mathbb{P}\big(q_{r}^{\star(j)}\leq X_{k}^{(j)}<\hat{q}_{n,r}^{(j)}\big)+\mathbb{P}\big(\hat{q}_{n,r}^{(j)}\leq X_{l}^{(j)}<q_{r}^{\star(j)}\big)\\
& \leq 3\sqrt{a_{n}}~\mathbb{P}\big(\hat{q}_{n,r}^{(j)}\leq X_{1}^{(j)}<q_{r}^{\star(j)}\big)+3\sqrt{a_{n}}~\mathbb{P}\big(q_{r}^{\star(j)}\leq X_{1}^{(j)}<\hat{q}_{n,r}^{(j)}\big),
\end{align*}
which tends to zero, according to Lemma \ref{convergence_quantile}. Thus, in probability, 
\begin{align}
\lim\limits_{n \to \infty} W_{n,r}^{(j)} = 0. \label{proof_lemma2_eq:1}
\end{align}
Regarding the remaining terms in inequality (\ref{delta_bound_1}), by the law of large numbers, in probability,   
\begin{align}
\lim\limits_{n \to \infty} \frac{N_{n}(H)}{a_{n}} = \mathbb{P}\big(\bX\in H\big), \quad
\lim\limits_{n \to \infty} \frac{N_{n}(H_{L}^{\star})}{a_{n}} = \mathbb{P}\big(\bX\in H_{L}^{\star}\big).
\label{proof_lemma2_eq:2}
\end{align}
Additionally, 
\begin{align*}
\mathbb{E}\big[\big|\frac{N_{n}(H_{L})}{a_{n}}-\frac{N_{n}(H_{L}^{\star})}{a_{n}}\big|\big] & \leq \mathbb{E}\big[ \frac{1}{a_{n}}\sum_{i=1}^{a_{n}}\mathds{1}_{X^{(j)}_{i}\in H}\big| \mathds{1}_{X^{(j)}_{i}\leq\hat{q}_{n,r}^{(j)}}-\mathds{1}_{X^{(j)}_{i}\leq q_{r}^{\star(j)}}\big|\big]\\
& \leq \mathbb{P}\big(\hat{q}_{n,r}^{(j)}\leq X_{1}^{(j)}<q_{r}^{\star(j)}\big)+\mathbb{P}\big(q_{r}^{\star(j)}\leq X_{1}^{(j)}<\hat{q}_{n,r}^{(j)}\big),
\end{align*}
which tends to zero, according to Lemma \ref{convergence_quantile}. Therefore, in probability, 
\begin{align}
\lim\limits_{n \to \infty} \frac{N_{n}(H_{L})}{a_{n}}-\frac{N_{n}(H_{L}^{\star})}{a_{n}} = 0.
\label{proof_lemma2_eq:3}    
\end{align}
Since $\mathbb{P}(\bX\in H) > 0$ and $\mathbb{P}(\bX\in H_{L}^{\star}) > 0$ by assumption, 
we can combine (\ref{proof_lemma2_eq:1})-(\ref{proof_lemma2_eq:3}) to obtain, in probability, 
\begin{align}
\lim\limits_{n \to \infty} \frac{a_{n}^{3}}{N_{n}(H)N_{n}(H_{L})N_{n}(H_{L}^{\star})}=\frac{1}{\mathbb{P}
(\bX\in H)\mathbb{P}(\bX\in H_{L}^{\star})^{2}}. \label{eq:2}
\end{align}
Using (\ref{proof_lemma2_eq:1}) and (\ref{eq:2}) and inequality (\ref{delta_bound_1}), we obtain, in probability, 
\begin{align*}
 \lim\limits_{n \to \infty} \frac{\sqrt{a_{n}}}{N_{n}(H)}\big|\overline{Y}_{H_{L}}^{2}N_{n}(H_{L})-\overline{Y}_{H_{L}^{\star}}^{2}N_{n}(H_{L}^{\star})\big| = 0.
 \end{align*}
Similar results can be derived for the other term in equation (\ref{proof_lemme2_eq1}), which allows us to conclude that, in probability, 
\begin{align*}
 \lim\limits_{n \to \infty} \sqrt{a_{n}} \big(L_{a_{n}}\big(H,\hat{q}_{n,r}^{(j)}\big)-L_{a_{n}}\big(H,q_{r}^{\star(j)}\big)\big) = 0.
\end{align*}
\end{proof}

\begin{proof}[Proof of Lemma \ref{criterion_consistency_1}]

Let $j\in\{1, ..., p\}$, $r\in\{1, ..., q - 1\}$ and $H\subseteq\mathds{R}^{p}$ such that 
$\mathbb{P}\big(\bX \in H, X^{(j)} < q^{\star(j)}_{r}\big)>0$ and $\mathbb{P}\big(\bX \in H, X^{(j)} \geq q^{\star(j)}_{r}\big)>0$.
\begin{align*}
L_{a_{n}}\big(H,\hat{q}_{n,r}^{(j)}\big) = L_{a_{n}}\big(H,q_{r}^{\star(j)}\big) +
\big(L_{a_{n}}\big(H,\hat{q}_{n,r}^{(j)}\big) - L_{a_{n}}\big(H,q_{r}^{\star(j)}\big)\big)
\end{align*}
From the law of large number, in probability, 
\begin{align*}
\lim \limits_{n \to \infty} L_{a_{n}}\big(H,q_{r}^{\star(j)}\big)
= L^{\star}\big(H,q_{r}^{\star(j)}\big).
\end{align*}
Thus, according to Lemma \ref{criterion_residual_1}, in probability, 
\begin{align*}
\lim \limits_{n \to \infty} L_{a_{n}}\big(H,\hat{q}_{n,r}^{(j)}\big)
= L^{\star}\big(H,q_{r}^{\star(j)}\big).
\end{align*}
\end{proof}

\begin{proof}[Proof of Lemma \ref{criterion_distribution_1}]
We consider $\bestcut_1$, a set of splits of cardinality $c_1 \geq 2$ satisfying, for all $(j,r) \in \bestcut_1$,
$L^{\star}\big(\R^p,  q_{r}^{\star(j)}\big) \stackrel{\tiny{\rm def}}{=}L^{\star}_{\bestcut_1}$.
Fix $(j_1, r_1) \in \bestcut_1$, we recall that
\begin{align*}
\textbf{L}_{n,\path_1}^{(\bestcut_1)} = \left(\begin{array}{c}
L_{a_{n}}\big(\mathbb{R}^{p},\hat{q}_{n,r}^{(j)}\big)-L_{a_{n}}\big(\mathbb{R}^{p},\hat{q}_{n,r_1}^{(j_1)}\big)
\end{array}\right)_{(j, r) \in \bestcut_1 \setminus (j_1, r_1)}.
\end{align*}

\paragraph{Case (a): $L^{\star}_{\bestcut_1} > 0$} 
We first consider the following decomposition for $(j,r) \in \bestcut_1$,
\begin{align*}
L_{a_{n}}\big(&\R^p,\hat{q}_{n,r}^{(j)}\big) = L_{a_{n}}\big(\R^p,q_{r}^{\star(j)}\big) +
\big(L_{a_{n}}\big(\R^p,\hat{q}_{n,r}^{(j)}\big) - L_{a_{n}}\big(\R^p,q_{r}^{\star(j)}\big)\big) \\
=& \frac{1}{a_{n}}\sum_{i=1}^{a_n}(Y_{i}-\overline{Y})^{2}
-\frac{1}{a_{n}}\sum_{i=1}^{a_n}\big(Y_{i}-\overline{Y}^{\star}_{L}\mathds{1}_{X_{i}^{(j)}<q_{r}^{\star(j)}}
-\overline{Y}^{\star}_{R}\mathds{1}_{X_{i}^{(j)}\geq q_{r}^{\star(j)}}\big)^{2} \\
&+ L_{a_{n}}\big(\R^p,\hat{q}_{n,r}^{(j)}\big) - L_{a_{n}}\big(\R^p,q_{r}^{\star(j)}\big), 
\end{align*}
where 
\begin{align*}
N_{n,L}^{\star} = \sum_{i=1}^{a_{n}} \mathds{1}_{X_{i}^{(j)}<q_{r}^{\star(j)}}    \quad \textrm{and} \quad \overline{Y}_{L}^{\star}=\frac{1}{N_{n,L}^{\star}} \sum_{i=1}^{a_{n}}Y_{i}\mathds{1}_{X_{i}^{(j)}<q_{r}^{\star(j)}}
\end{align*}
($\overline{Y}_{R}^{\star}$, $N_{n,R}^{\star}$ are defined symmetrically).
Letting $\mu_{L,r}^{(j)} = \E\big[Y | X^{(j)}<q_{r}^{\star(j)}\big]$ (and $\mu_{R,r}^{(j)}$ symmetrically), the first two terms of the last decomposition are standard variance estimates and we can write
\begin{align}
L_{a_{n}}\big(\R^p,\hat{q}_{n,r}^{(j)}\big) 
= &\frac{1}{a_{n}}\sum_{i=1}^{a_n}(Y_{i}-\overline{Y})^{2} \\
&-\frac{1}{a_{n}}\sum_{i=1}^{a_n}\big(Y_{i}- \mu_{L,r}^{(j)}\mathds{1}_{X_{i}^{(j)}<q_{r}^{\star(j)}}
-\mu_{R,r}^{(j)}\mathds{1}_{X_{i}^{(j)}\geq q_{r}^{\star(j)}}\big)^{2} + R_{n,r}^{(j)}, \label{eq_lemma4_1}
\end{align}
where
\begin{align}  \label{eq_lemma4}
R_{n,L}^{(j)} =& \frac{N_{n,L}^{\star}}{a_{n}} \big(\overline{Y}^{\star}_{L} - \mu_{L,r}^{(j)}\big)^2 
+ \frac{N_{n,R}^{\star}}{a_{n}}\big(\overline{Y}^{\star}_{R} - \mu_{L,r}^{(j)}\big)^{2} \\ &+ L_{a_{n}}\big(\R^p,\hat{q}_{n,r}^{(j)}\big)
 - L_{a_{n}}\big(\R^p,q_{r}^{\star(j)}\big). \nonumber
\end{align}
Using the Central limit theorem, in probability,
\begin{align} \label{eq_proof_normal_1}
\lim \limits_{n \to \infty} \sqrt{a_n}~ \frac{N_{L,r}^{\star}}{a_{n}} \big(\overline{Y}^{\star}_{L,r} - \mu_{L,r}^{(j)}\big)^2 = 0.
\end{align}
The same result holds for the second term of (\ref{eq_lemma4}), and using Lemma \ref{criterion_residual_1} for the third term of (\ref{eq_lemma4}),
we get that, in probability,
\begin{align*}
\lim \limits_{n \to \infty} \sqrt{a_n} \big(L_{a_{n}}\big(\R^p,\hat{q}_{n,r}^{(j)}\big) - L_{a_{n}}\big(\R^p,q_{r}^{\star(j)}\big) \big) = 0.
\end{align*}
Finally,
\begin{align*}
\lim \limits_{n \to \infty} \sqrt{a_n} R_{n,r}^{(j)} = 0, \quad \textrm{in probability}.
\end{align*}
Using Equation (\ref{eq_lemma4_1}), each component of $\textbf{L}_{n,\path_1}^{(\bestcut_1)}$ writes, with $(j,r) \in \bestcut_1 \setminus (j_1,r_1)$,
\begin{align*}
L_{a_{n}}\big(\R^p,\hat{q}_{n,r}^{(j)}\big)  &- L_{a_{n}}\big(\R^p,\hat{q}_{n,r_1}^{(j_1)}\big) \\
=& \frac{1}{a_{n}}\sum_{i=1}^{a_n}\big(Y_{i} - \mu_{L,r_1}^{(j_1)}\mathds{1}_{X_{i}^{(j_1)}<q_{r_1}^{\star(j_1)}}
-\mu_{R,r_1}^{(j_1)}\mathds{1}_{X_{i}^{(j_1)}\geq q_{r_1}^{\star(j_1)}}\big)^{2} \\
&- \big(Y_{i} - \mu_{L,r}^{(j)}\mathds{1}_{X_{i}^{(j)}<q_{r}^{\star(j)}}
-\mu_{R,r}^{(j)}\mathds{1}_{X_{i}^{(j)}\geq q_{r}^{\star(j)}}\big)^{2} \\ &+ R_{n,r}^{(j)} - R_{n,r_1}^{(j_1)}
 \end{align*}
We can apply the multivariate Central limit theorem and Slutsky's theorem to obtain, 
\begin{align*}
\sqrt{a_{n}}~ \textbf{L}_{n,\path_1}^{(\bestcut_1)} \stackrel[n\rightarrow\infty]{\mathscr{D}}{\longrightarrow}\mathcal{N}\big(0, \Sigma\big)
\end{align*}
where for all $(j,r), (j',r') \in \bestcut_1 \setminus (j_1, r_1)$, each element of the covariance matrix $\Sigma$ is defined by $
\Sigma_{(j,r), (j',r')} = \textrm{Cov}[Z_{j,r}, Z_{j',r'}],
$ with
\begin{align*}
Z_{j,r}= & \big(Y - \mu_{L,r_1}^{(j_1)}\mathds{1}_{X^{(j_1)}<q_{r_1}^{\star(j_1)}}
-\mu_{R,r_1}^{(j_1)}\mathds{1}_{X^{(j_1)}\geq q_{r_1}^{\star(j_1)}}\big)^{2}\\
&- \big(Y - \mu_{L,r}^{(j)}\mathds{1}_{X^{(j)}<q_{r}^{\star(j)}}
-\mu_{R,r}^{(j)}\mathds{1}_{X^{(j)}\geq q_{r}^{\star(j)}}\big)^{2}.
\end{align*}
Since $L^{\star}_{\bestcut_1} > 0$, we have for all $(j,r) \in \bestcut_1$, $\mu_{L,r}^{(j)} \neq \mu_{R,r}^{(j)}$. Besides, according to assumption (A3), $\bX$ has a strictly positive density. Consequently, the variance of $Z_{j,r}$ is strictly positive. This concludes the first case.

\paragraph{Case (b): $L^{\star}_{\bestcut_1} = 0$}
Fix $(j,r) \in \bestcut_1$. Since $L^{\star}\big(\R^p, q_{r}^{\star(j)}\big) = 0$, we have 
\begin{align*}
\E[Y] = \E\big[Y|X^{(j)} < q^{\star (j)}_{r}\big] = \E\big[Y|X^{(j)} \geq q^{\star (j)}_{r}\big] \stackrel{\tiny{\rm def}}{=} \mu. 
\end{align*}
Then, simple calculations show that $L_{a_{n}}\big(\mathbb{R}^{p},\hat{q}_{n,r}^{(j)}\big)$ writes
\begin{align*}
L_{a_{n}}\big(\mathbb{R}^{p},\hat{q}_{n,r}^{(j)}\big) = - (\overline{Y} - \mu)^2 + \underbrace{\frac{N_{n,L}}{a_n}(\overline{Y}_{L} - \mu)^2}_{\delta_{L}} 
 + \underbrace{\frac{N_{n,R}}{a_n}(\overline{Y}_{R} - \mu)^2}_{\delta_{R}},
\end{align*}
where 
\begin{align*}
N_{n,L} = \sum_{i=1}^{a_n} \mathds{1}_{X_i^{(j)} < \hat{q}_{n,r}^{(j)}}  
\quad \textrm{and}
\quad 
\overline{Y}_{L}=\frac{1}{N_{n,L}}\sum_{i=1}^{a_{n}}Y_{i}\mathds{1}_{X_{i}^{(j)}<\hat{q}_{n,r}^{(j)}}
\end{align*}
($N_{n,R}$, $\overline{Y}_{R}$ are defined similarly for the other cell).
Letting $p_{L,r}^{(j)} = \P \big(X^{(j)}<q_{r}^{\star (j)}\big)$ and $p_{R,r}^{(j)} = \P \big(X^{(j)} \geq q_{r}^{\star (j)}\big)$ with $p_{L,r}^{(j)}, p_{R,r}^{(j)} >0$, we have 
\begin{align*}
\delta_{L} &= \frac{N_{n,L}}{a_n}(\overline{Y}_{L} - \mu)^2 \\
&= \frac{N_{n,L}}{a_n}(\overline{Y}_{L}^{\star} - \mu)^2 - 2\frac{N_{n,L}}{a_n} (\overline{Y}_{L}^{\star} - \overline{Y}_{L})(\overline{Y}_{L}^{\star} - \mu)
+ \frac{N_{n,L}}{a_n} (\overline{Y}_{L}^{\star} - \overline{Y}_{L})^2 \\
&= \frac{1}{p_{L,r}^{(j)}} \big( \frac{1}{a_n} \sum_{i=1}^{a_n} (Y_i - \mu) \mathds{1}_{X_i^{(j)} < q^{\star(j)}_{r}} \big)^2 + R_{L,r}^{(j)},
\end{align*}
where
\begin{align*}
R_{L,r}^{(j)} = & \big(\frac{a_n N_{n,L}}{N_{n,L}^{\star 2}} - \frac{1}{p_{n,L}}\big)
 \big( \frac{1}{a_n} \sum_{i=1}^{a_n} (Y_i - \mu) \mathds{1}_{X_i^{(j)} < q^{\star(j)}_{r}} \big)^2 \\
& - 2\frac{N_{n,L}}{a_n} (\overline{Y}_{L}^{\star} - \overline{Y}_{L})(\overline{Y}_{L}^{\star} - \mu)
+ \frac{N_{n,L}}{a_n} (\overline{Y}_{L}^{\star} - \overline{Y}_{L})^2
\end{align*}
By the law of large numbers, $\lim \limits_{n \to \infty} \frac{N_{n,L}^{\star}}{a_n} = p_{L,r}^{(j)}$ in probability.
Using Equation (\ref{proof_lemma2_eq:3}) in the proof of Lemma \ref{criterion_residual_1}, 
it comes that, in probability, $\lim \limits_{n \to \infty} \frac{N_{n,L}}{a_n} = p_{L,r}^{(j)} $, 
and consequently $\lim \limits_{n \to \infty} \frac{a_n N_{n,L}}{N_{n,L}^{\star 2}}   = \frac{1}{p_{L,r}^{(j)}}$.
Since $\sqrt{a_n} \frac{1}{a_n} \sum_{i=1}^{a_n} (Y_i - \mu) \mathds{1}_{X_i^{(j)} < q^{\star(j)}_{r}}$ converges in distribution to a normal distribution by the Central limit theorem, 
\begin{align*}
\lim \limits_{n \to \infty} a_n \big(\frac{a_n N_{n,L}}{N_{n,L}^{\star 2}} - \frac{1}{p_{L,r}^{(j)}}\big)
 \big( \frac{1}{a_n} \sum_{i=1}^{a_n} (Y_i - \mu) \mathds{1}_{X_i^{(j)} < q^{\star(j)}_{r}} \big)^2 = 0, \quad  \textrm{in probability.}    
\end{align*}
Furthermore, as for Equation (\ref{proof_lemma2_eq2_W}) in the proof of Lemma \ref{criterion_residual_1}, 
\begin{align*}
\sqrt{a_n} & |\overline{Y}_{L}^{\star} - \overline{Y}_{L}|
\\ &\leq \frac{a_n^2}{N_{n,L} N_{n,L}^{\star}} \underbrace{\frac{\sqrt{a_n}}{a_n^2} \sum_{i=1,l=1}^{a_n} Y_i 
\big|\mathds{1}_{X_i^{(j)} < q^{\star(j)}_{r}}\mathds{1}_{X_l^{(j)} < \hat{q}^{(j)}_{r}}
 - \mathds{1}_{X_i^{(j)} < \hat{q}^{(j)}_{r}}\mathds{1}_{X_l^{(j)} < q^{\star(j)}_{r}}\big|}_{\varepsilon_{L}},
\end{align*}
and 
\begin{align*}
\E[\varepsilon_{L}] \leq 2 \sqrt{a_n} \P\big(\hat{q}^{(j)}_{r} \leq X^{(j)} < q^{\star(j)}_{r}\big)
+ 2 \sqrt{a_n} \P\big(q^{\star(j)}_{r} \leq X^{(j)} < \hat{q}^{(j)}_{r}\big).
\end{align*}
According to Lemma \ref{convergence_quantile}, the right hand side term converges to $0$. Then, in probability,  $\lim \limits_{n \to \infty} \varepsilon_L = 0$.
Additionally, $\lim \limits_{n \to \infty} \frac{a_n^2}{N_{n,L} N_{n,L}^{\star}} = \frac{1}{p_{L,r}^{(j) 2}}$, and then, in probability,
\begin{align} \label{eq_R_term_2}
\lim \limits_{n \to \infty} \sqrt{a_n} (\overline{Y}_{L}^{\star} - \overline{Y}_{L}) = 0.
\end{align}
The second term of $a_n R_{L,r}^{(j)}$ writes
\begin{align*}
- a_n \times 2\frac{N_{n,L}}{a_n} (\overline{Y}_{L}^{\star} &- \overline{Y}_{L})(\overline{Y}_{L}^{\star} - \mu)
\\ &= - 2\frac{N_{n,L}}{a_n} \times \sqrt{a_n} (\overline{Y}_{L}^{\star} - \overline{Y}_{L}) \times \sqrt{a_n} (\overline{Y}_{L}^{\star} - \mu),
\end{align*}
where in probability, $\lim \limits_{n \to \infty} 2\frac{N_{n,L}}{a_n} = p_{L,r}^{(j)}$, $\lim \limits_{n \to \infty} \sqrt{a_n} (\overline{Y}_{L}^{\star} - \overline{Y}_{L}) = 0$ according to equation \ref{eq_R_term_2}, and $\sqrt{a_n} (\overline{Y}_{L}^{\star} - \mu)$ converges to a normal random variable from the central limit theorem.
 By Slutsky theorem, in probability,
$\lim \limits_{n \to \infty} - a_n \times 2\frac{N_{n,L}}{a_n} (\overline{Y}_{L}^{\star} - \overline{Y}_{L})(\overline{Y}_{L}^{\star} - \mu) = 0$.
Finally for the third term of $a_n R_{L,r}^{(j)}$ we also use equation \ref{eq_R_term_2} to conclude that in probability
\begin{align*}
\lim \limits_{n \to \infty} a_n \times \frac{N_{n,L}}{a_n} (\overline{Y}_{L}^{\star} - \overline{Y}_{L})^2 
= \lim \limits_{n \to \infty} \frac{N_{n,L}}{a_n} [\sqrt{a_n} (\overline{Y}_{L}^{\star} - \overline{Y}_{L})]^2 = 0
\end{align*}
Consequently,
\begin{align*}
\lim \limits_{n \to \infty} a_n R_{L,r}^{(j)} = 0.
\end{align*}
Symmetrically, we also have
\begin{align*}
\delta_{R} = \frac{1}{p_{R}} \big( \frac{1}{a_n} \sum_{i=1}^{a_n} (Y_i - \mu) \mathds{1}_{X_i^{(j)} \geq q^{\star(j)}_{r}} \big)^2 + R_{R,r}^{(j)},
\end{align*}
with $\lim \limits_{n \to \infty} a_n R_{R,r}^{(j)} = 0$, in probability.

Each component of $\textbf{L}_{n,\path_1}^{(\bestcut_1)}$ writes, with $(j,r) \in \bestcut_1 \setminus (j_1,r_1)$,
\begin{align*}
L&_{a_{n}}\big(\R^p,\hat{q}_{n,r}^{(j)}\big)  - L_{a_{n}}\big(\R^p,\hat{q}_{n,r_1}^{(j_1)}\big)
= \frac{1}{p_{L,r}^{(j)}} \big( \frac{1}{a_n} \sum_{i=1}^{a_n} (Y_i - \mu) \mathds{1}_{X_i^{(j)} < q^{\star(j)}_{r}} \big)^2 \\
&+ \frac{1}{p_{R,r}^{(j)}} \big( \frac{1}{a_n} \sum_{i=1}^{a_n} (Y_i - \mu) \mathds{1}_{X_i^{(j)} \geq q^{\star(j)}_{r}} \big)^2  
- \frac{1}{p_{L,r_1}^{(j_1)}} \big( \frac{1}{a_n} \sum_{i=1}^{a_n} (Y_i - \mu) \mathds{1}_{X_i^{(j_{1})} < q^{\star(j_{1})}_{r_{1}}} \big)^2 \\
&- \frac{1}{p_{R,r_1}^{(j_1)}} \big( \frac{1}{a_n} \sum_{i=1}^{a_n} (Y_i - \mu) \mathds{1}_{X_i^{(j_{1})} \geq q^{\star(j_{1})}_{r_{1}}} \big)^2 + R_{L,r}^{(j)} + R_{R,r}^{(j)} - R_{L,r_1}^{(j_1)} - R_{R,r_1}^{(j_1)}.
 \end{align*}

We explicitly write $\bestcut_1 = \{(j_k, r_k)\}_{k=1,...,c_1}$. Then $\textbf{L}_{n,\path_1}^{(\bestcut_1)}$ can be decomposed as
\begin{align*}
a_n \textbf{L}_{n,\path_1}^{(\bestcut_1)} = h_{\path_1}(\mathbf{V}_n) + \mathbf{R}_{n, \path_1},
\end{align*}
where for $k \in \{1,...,c_1\}$,
\begin{align*}
V_{n,2k-1} = \sqrt{\frac{a_n}{p_{L,r_k}^{(j_k)}}} \frac{1}{a_n} \sum_{i=1}^{a_n} (Y_i - \mu) \mathds{1}_{X_i^{(j_k)} < q^{\star(j_k)}_{r_k}}, \\
V_{n,2k} = \sqrt{\frac{a_n}{p_{R,r_k}^{(j_k)}}} \frac{1}{a_n} \sum_{i=1}^{a_n} (Y_i - \mu) \mathds{1}_{X_i^{(j_k)} \geq q^{\star(j_k)}_{r_k}}.
\end{align*}
$h_{\path_1}$ is a multivariate quadratic form defined as
\begin{align*}
h_{\path_1}: \left( \begin{array}{c} x_1 \\ \vdots \\ x_{2c_1} \end{array} \right) \rightarrow \left( \begin{array}{c}
x_{3}^2 + x_{4}^2 - x_1^2 - x_2^2\\ \vdots \\
x_{2k-1}^2 + x_{2k}^2 - x_{1}^2 - x_{2}^2 \\ \vdots \\
x_{2c_1-1}^2 + x_{2c_1}^2 - x_{1}^2 - x_{2}^2  \end{array}\right).
\end{align*}
and $R_{n, \path_1, k} = R_{L,r_k}^{(j_k)} + R_{R,r_k}^{(j_k)} - R_{L,r_1}^{(j_1)} - R_{R,r_1}^{(j_1)}$.

From the multivariate central limit theorem, $\mathbf{V}_n \stackrel[n\rightarrow\infty]{\mathscr{D}}{\longrightarrow} \mathbf{V}$,
where $\mathbf{V}$ is a gaussian vector of covariance matrix $\textrm{Cov}[\mathbf{Z}]$, and $\mathbf{Z}$ is defined as, for $k \in \{1,...,c_1\}$,
\begin{align*}
Z_{2k-1} = \frac{1}{\sqrt{p_{L,k}}} (Y - \E[Y]) \mathds{1}_{X^{(j_k)} < q_{r_k}^{\star (j_k)}}, 
Z_{2k} = \frac{1}{\sqrt{p_{R,k}}} (Y - \E[Y]) \mathds{1}_{X^{(j_k)} \geq q_{r_k}^{\star (j_k)}},
\end{align*}
with the simplified notations $p_{L,k} = p_{L,r_k}^{(j_k)}$ and $p_{R,k} = p_{R,r_k}^{(j_k)}$.

Finally, since $\lim \limits_{n \to \infty} \mathbf{R}_{n, \path_1} = \mathbf{0}$ in probability, from Slutsky's theorem and the continuous mapping theorem,
$a_{n} \textbf{L}_{n,\path_1}^{(\bestcut_1)} \stackrel[n\rightarrow\infty]{\mathscr{D}}{\longrightarrow} h_{\path_1}(\mathbf{V})$.
Note that, since $\bX$ has a strictly positive density, each component of $h_{\path_1}(\mathbf{V})$ has a strictly positive variance.
\end{proof}

\begin{proof}[Proof of Lemma  \ref{convergence_pn_theta_1}]
Consider a path $\path = (j_1, r_1, \cdot)$. Set $\theta^{(V)} = (\theta_1^{(V)},\cdot,\cdot) \in \Omega^{(V)}$, a realization of the randomization of the split direction. Recalling that the best split in a random tree is the one maximizing the CART-splitting criterion, condition on $\Theta^{(V)}=\theta^{(V)}$,
\begin{align}
\label{pn_theta_P1}
\{ \path_{1} \in T(\Theta,\mathscr{D}_{n}) \} 
=\underset{\substack{(j,r) \in \theta_{1}^{(V)} \times \{1, ..., q - 1\} \\ \qquad \quad \setminus{(j_1, r_1)}}}
{\bigcap}\big\{ L_{a_{n}}\big(\mathds{R}^{p},\hat{q}_{n,r_{1}}^{(j_{1})}\big)>L_{a_{n}}\big(\mathds{R}^{p},\hat{q}_{n,r}^{(j)}\big)\big\}.
\end{align}
We recall that, given $\theta^{(V)}$, we define the set of best theoretical cuts along the variables in $\theta^{(V)}_1$ as  
$$
\bestcut_{1}^{\star}\big(\theta_1^{(V)}\big) = \underset{(j,r) \in \theta_{1}^{(V)} \times \{1, ..., q - 1\}}{\textrm{argmax}} L^{\star}\big(\mathbb{R}^{p},q_{r}^{\star(j)}\big).
$$
Obviously if $(j_1,r_1) \notin \theta_{1}^{(V)} \times \{1, ..., q - 1\}$, the probability to select $\path_1$ in the empirical and theoretical tree is null.
In the sequel, we assume that $(j_1,r_1) \in \theta_{1}^{(V)} \times \{1, ..., q - 1\}$ and distinguish between four cases:
$(j_1,r_1)$ is not among the best theoretical cuts $\bestcut_{1}^{\star}\big(\theta_1^{(V)}\big)$, is the only element in $\bestcut_{1}^{\star}\big(\theta_1^{(V)}\big)$, is one element of $\bestcut_{1}^{\star}\big(\theta_1^{(V)}\big)$ with a positive value of the theoretical CART-splitting criterion, or finally, is one element of $\bestcut_{1}^{\star}\big(\theta_1^{(V)}\big)$ that all have a null value of the theoretical CART-splitting criterion.
\paragraph{Case 1} We assume that 
$(j_1, r_1) \notin \bestcut_{1}^{\star}\big(\theta_1^{(V)}\big)$. By definition of the theoretical random forest,
\begin{align}
\mathbb{P}\big(\path_{1} \in T^{\star}(\Theta)|\Theta^{(V)}=\theta^{(V)}\big) = 0    
\end{align}
Let $\big(j^{\star}, r^{\star} \big) \in \bestcut_{1}^{\star} \big(\theta_1^{(V)}\big)$, thus 
\begin{align*}
\varepsilon =L^{\star}\big(\mathds{R}^{p},q_{r^{\star}}^{\star(j^{\star})}\big) - L^{\star}\big(\mathds{R}^{p},q_{r_{1}}^{\star (j_{1})}\big) >0.
\end{align*}
Using equation (\ref{pn_theta_P1}), we have:
\begin{align*}
\mathbb{P}\big(\path_{1} \in T&\big(\Theta,\mathscr{D}_{n}\big)|\Theta^{(V)}=\theta^{(V)}\big)
\\
\leq \mathbb{P}\big(L_{a_{n}}&\big(\mathds{R}^{p},\hat{q}_{n,r_{1}}^{(j_{1})}\big)>L_{a_{n}}\big(\mathds{R}^{p},\hat{q}_{n,r^{\star}}^{(j^{\star})}\big)\big) \\
\leq \mathbb{P}\big(L_{a_{n}}&\big(\mathds{R}^{p},\hat{q}_{n,r_{1}}^{(j_{1})}\big)
- L^{\star}\big(\mathds{R}^{p},q_{r_{1}}^{\star (j_{1})}\big) - \epsilon
>L_{a_{n}}\big(\mathds{R}^{p},\hat{q}_{n,r^{\star}}^{(j^{\star})}\big)
- L^{\star}\big(\mathds{R}^{p},q_{r^{\star}}^{\star(j^{\star})}\big)\big) \\
 \leq \mathbb{P}\big(  L_{a_{n}}&\big(\mathds{R}^{p},\hat{q}_{n,r_{1}}^{(j_{1})}\big)
- L^{\star}\big(\mathds{R}^{p},q_{r_{1}}^{\star (j_{1})}\big)    
 \\ & - \big( L_{a_{n}}\big(\mathds{R}^{p},\hat{q}_{n,r^{\star}}^{(j^{\star})}\big)
- L^{\star}\big(\mathds{R}^{p},q_{r^{\star}}^{\star(j^{\star})}\big) \big) > \epsilon \big)
\end{align*}
Therefore, according to Lemma \ref{criterion_consistency_1},
\begin{align*}
\lim \limits_{n \to \infty} \mathbb{P}\big(\path_{1} \in T(\Theta,\mathscr{D}_{n})|\Theta^{(V)}=\theta^{(V)}\big) = 0
= \mathbb{P}\big(\path_{1} \in T^{\star}(\Theta)|\Theta^{(V)}=\theta^{(V)}\big)
\end{align*}
\paragraph{Case 2} We assume that 
$\bestcut_{1}^{\star}\big(\theta_1^{(V)}\big) = \big\{ (j_1, r_1) \big\}$. By definition of the theoretical random forest,  
\begin{align}
\mathbb{P}\big(\path_{1} \in T^{\star}(\Theta)|\Theta^{(V)}=\theta^{(V)}\big) = 1.    
\end{align}
Conditional on $\Theta^{(V)}=\theta^{(V)}$,
\begin{align*}
\left\lbrace\path_{1} \in T(\Theta,\mathscr{D}_{n})\right\rbrace^c =
\underset{\substack{(j,r) \in \theta_{1}^{(V)} \times \{1, ..., q - 1\} \\ \qquad \quad  \setminus{(j_1, r_1)}}}
{\bigcup}\big\{ L_{a_{n}}\big(\mathds{R}^{p},\hat{q}_{n,r_{1}}^{(j_{1})}\big)
\leq L_{a_{n}}\big(\mathds{R}^{p},\hat{q}_{n,r}^{(j)}\big)\big\},   
\end{align*}
which leads to 
\begin{align}
1 - & \mathbb{P}\big(\path_{1} \in   T(\Theta,\mathscr{D}_{n})|\Theta^{(V)}=\theta^{(V)}\big) \nonumber \\
& \leq \sum_{(j,r) \in \theta_{1}^{(V)} \times \{1, ..., q - 1\} \setminus{(j_1, r_1)}}
\mathbb{P}\big( L_{a_{n}}\big(\mathds{R}^{p},\hat{q}_{n,r_{1}}^{(j_{1})}\big)
\leq L_{a_{n}}\big(\mathds{R}^{p},\hat{q}_{n,r}^{(j)}\big) \big). \label{proof_eq_decomp_prob}
\end{align}
From Lemma \ref{criterion_consistency_1}, for all $j\in\theta^{(V)}_0$, $r\in\{1, ..., q - 1\}$
such that $(j, r) \neq (j_1, r_1)$, in probability,
\begin{align}
\lim \limits_{n \to \infty} L_{a_{n}}\big(\mathds{R}^{p},\hat{q}_{n,r_{1}}^{(j_{1})}\big)
- L_{a_{n}}\big(\mathds{R}^{p},\hat{q}_{n,r}^{(j)}\big) 
=  L^{\star}\big(\mathbb{R}^{p}, q_{r_{1}}^{\star(j_{1})}\big)
-  L^{\star}\big(\mathbb{R}^{p}, q_{r}^{\star(j)}\big) > 0. \label{proof_eq_comp_crit}
\end{align}
Using inequality (\ref{proof_eq_decomp_prob}) and equation (\ref{proof_eq_comp_crit}), we finally obtain,
\begin{align*}
\lim \limits_{n \to \infty}
\mathbb{P}\big(\path_{1} \in T(\Theta,\mathscr{D}_{n})|\Theta^{(V)}=\theta^{(V)}\big)
= 1 = \mathbb{P}\big(\path_{1} \in T^{\star}(\Theta)|\Theta^{(V)}=\theta^{(V)}\big).
\end{align*}
\paragraph{Case 3} We assume that 
$(j_1, r_1) \in \bestcut_{1}^{\star}\big(\theta_1^{(V)}\big)$, $\big| \bestcut_{1}^{\star}\big(\theta_1^{(V)}\big) \big| > 1$,
and $L^{\star}\big(\mathbb{R}^{p}, q_{r_{1}}^{\star(j_{1})}\big) > 0$.
On one hand, conditional on  $\Theta^{(V)}=\theta^{(V)}$,
\begin{align*}
\left\lbrace \path_{1} \in T(\Theta,\mathscr{D}_{n}) \right\rbrace
\subset \underset{(j, r) \in \bestcut_{1}^{\star}(\theta_1^{(V)}) \setminus{(j_1, r_1)}}
{\bigcap}\big\{ L_{a_{n}}\big(\mathds{R}^{p},\hat{q}_{n,r_{1}}^{(j_{1})}\big)
>L_{a_{n}}\big(\mathds{R}^{p},\hat{q}_{n,r}^{(j)}\big)\big\}.
\end{align*}
On the other hand, conditional on $\Theta^{(V)}=\theta^{(V)}$,
\begin{align*}
\left\lbrace \path_{1} \in T(\Theta,\mathscr{D}_{n})\right\rbrace^c 
= \underset{(j, r) \in \bestcut_{1}^{\star}(\theta_1^{(V)}) \setminus{(j_1, r_1)}}
{\bigcup}\big\{ L_{a_{n}}\big(\mathds{R}^{p},\hat{q}_{n,r_{1}}^{(j_{1})}\big)
\leq L_{a_{n}}\big(\mathds{R}^{p},\hat{q}_{n,r}^{(j)}\big)\big\}  \\
 \underset{(j, r) \in \theta_1^{(V)} \times \{1, ..., q - 1\} \setminus \bestcut_{1}^{\star}{(\theta_1^{(V)})}}
{\bigcup}\big\{ L_{a_{n}}\big(\mathds{R}^{p},\hat{q}_{n,r_{1}}^{(j_{1})}\big)
\leq L_{a_{n}}\big(\mathds{R}^{p},\hat{q}_{n,r}^{(j)}\big)\big\}.
\end{align*}
Combining the two previous inclusions, 
\begin{align*}
0 & \leq \mathbb{P}\big(\underset{(j, r) \in \bestcut_{1}^{\star}(\theta_1^{(V)}) \setminus{(j_1, r_1)}}
{\bigcap}\big\{ L_{a_{n}}\big(\mathds{R}^{p},\hat{q}_{n,r_{1}}^{(j_{1})}\big)
>L_{a_{n}}\big(\mathds{R}^{p},\hat{q}_{n,r}^{(j)}\big)\big\} \big)\\
& \quad - \mathbb{P}\big(\path_{1} \in T(\Theta,\mathscr{D}_{n})|\Theta^{(V)}=\theta^{(V)}\big)  \\ 
& \leq \sum_{(j, r) \in \theta_1^{(V)} \times \{1, ..., q - 1\} \setminus \bestcut_{1}^{\star}{(\theta_1^{(V)})}}
\mathbb{P}\big( L_{a_{n}}\big(\mathds{R}^{p},\hat{q}_{n,r_{1}}^{(j_{1})}\big)
\leq L_{a_{n}}\big(\mathds{R}^{p},\hat{q}_{n,r}^{(j)}\big) \big).
\end{align*}
Using the same reasoning as in \textbf{Case 2}, we get
\begin{align*}
\lim \limits_{n \to \infty}  & \mathbb{P}\big(\underset{(j, r) \in \bestcut_{1}^{\star}(\theta_1^{(V)}) \setminus{(j_1, r_1)}}
{\bigcap}\big\{ L_{a_{n}}\big(\mathds{R}^{p},\hat{q}_{n,r_{1}}^{(j_{1})}\big)
>L_{a_{n}}\big(\mathds{R}^{p},\hat{q}_{n,r}^{(j)}\big)\big\} \big) \\
& \quad - \mathbb{P}\big(\path_{1} \in T(\Theta,\mathscr{D}_{n})|\Theta^{(V)}
=\theta^{(V)}\big) = 0.
\end{align*}
We define the random vector $\textbf{L}_{n,\path_1}^{(\bestcut_1^{\star})}$ where each component is the difference between the empirical
 CART-splitting criterion for the splits $(j, r) \in \bestcut_1^{\star} \setminus (j_1, r_1)$ and $(j_1, r_1)$,
\begin{align*}
\textbf{L}_{n,\path_1}^{(\bestcut_1^{\star})} = \left(\begin{array}{c}
L_{a_{n}}\big(\mathbb{R}^{p},\hat{q}_{n,r}^{(j)}\big)-L_{a_{n}}\big(\mathbb{R}^{p},\hat{q}_{n,r_1}^{(j_1)}\big)
\end{array}\right)_{(j, r) \in \bestcut_1^{\star} \setminus (j_1, r_1)},
\end{align*}
then
\begin{align*}
\mathbb{P}\big(\underset{(j, r) \in \bestcut_{1}^{\star}(\theta_1^{(V)}) \setminus{(j_1, r_1)}}
{\bigcap}\big\{ L_{a_{n}}\big(\mathds{R}^{p},\hat{q}_{n,r_{1}}^{(j_{1})}\big)
>L_{a_{n}}\big(\mathds{R}^{p},\hat{q}_{n,r}^{(j)}\big)\big\} \big) 
= \mathbb{P} \big( \textbf{L}_{n,\path_1}^{(\bestcut_1^{\star})} < \textbf{0} \big)
\end{align*}
From Lemma \ref{criterion_distribution_1} (case (a)),
\begin{align*}
\sqrt{a_{n}} \textbf{L}_{n,\path_1}^{(\bestcut_1^{\star})} \stackrel[n\rightarrow\infty]{\mathscr{D}}{\longrightarrow}\mathcal{N}\big(0, \Sigma\big).
\end{align*}
where for all $(j,r), (j',r') \in \bestcut_1^{\star} \setminus (j_1, r_1)$, each element of the covariance matrix $\Sigma$ is defined by
\begin{align*}
\Sigma_{(j,r), (j',r')} = \textrm{Cov}[Z_{j,r}, Z_{j',r'}],
\end{align*}
with
\begin{align*}
Z_{j,r}= & \big(Y - \mu_{L,r_1}^{(j_1)}\mathds{1}_{X^{(j_1)}<q_{r_1}^{\star(j_1)}}
-\mu_{R,r_1}^{(j_1)}\mathds{1}_{X^{(j_1)}\geq q_{r_1}^{\star(j_1)}}\big)^{2}\\
&- \big(Y - \mu_{L,r}^{(j)}\mathds{1}_{X^{(j)}<q_{r}^{\star(j)}}
-\mu_{R,r}^{(j)}\mathds{1}_{X^{(j)}\geq q_{r}^{\star(j)}}\big)^{2},
\end{align*}
$\mu_{L,r}^{(j)} = \E\big[Y | X^{(j)}<q_{r}^{\star(j)}\big]$, 
$\mu_{R,r}^{(j)} = \E\big[Y | X^{(j)} \geq q_{r}^{\star(j)}\big]$, 
and the variance of $Z_{j,r}$ is strictly positive.
If $\Phi_{\theta_1^{(V)}, (j_1, r_1) }$ is the c.d.f. of the multivariate normal distribution of covariance matrix $\Sigma$, we can conclude
\begin{align*}
\lim \limits_{n \to \infty} \mathbb{P}\big(\path_{1} \in T(\Theta,\mathscr{D}_{n})|\Theta^{(V)}
=\theta^{(V)}\big) & = \lim \limits_{n \to \infty} \mathbb{P} \big( \sqrt{a_{n}} \textbf{L}_{n,\path_1}^{(\bestcut_1^{\star})} < \textbf{0} \big)\\
& = \Phi_{\theta_1^{(V)}, (j_1, r_1)} (\textbf{0}),
\end{align*}
where  
\begin{align*}
\sum_{(j,r) \in \bestcut_{1}^{\star}(\theta_1^{(V)})} 
 \Phi_{\theta_1^{(V)}, (j, r)} (\textbf{0}) = 1.
\end{align*}
According to Definition \ref{def_proba_equality_1}, in the theoretical random forest, if $\bestcut_{1}^{\star}\big(\theta_1^{(V)}\big)$ has multiple elements, $(j_1, r_1)$ is randomly drawn with probability 
\begin{align*}
\mathbb{P}\big(\path_{1} \in T^{\star}(\Theta)|\Theta^{(V)}=\theta^{(V)}\big) 
= \Phi_{\theta_1^{(V)}, (j_1, r_1)} (\textbf{0}),
\end{align*}
that is 
\begin{align*}
\lim \limits_{n \to \infty} \mathbb{P}\big(\path_{1} \in T(\Theta,\mathscr{D}_{n})|\Theta^{(V)} = \theta^{(V)}\big)
&= \mathbb{P}\big(\path_{1} \in T^{\star}(\Theta)|\Theta^{(V)}=\theta^{(V)}\big) 
\\ &= \Phi_{\theta_1^{(V)}, (j_1, r_1)} (\textbf{0}).
\end{align*}
We can notice that, in the specific case where $\bestcut_{1}^{\star}\big(\theta_1^{(V)}\big)$ has two elements, they are both selected with equal probability $\frac{1}{2}$. For more than two elements, the weights are not necessary equal, it depends on the covariance matrix $\Sigma$.
\paragraph{Case 4} We assume that all candidate splits have a null value for the theoretical CART-splitting criterion, 
i.e. for $(j, r) \in \theta_{1}^{(V)} \times \{1, ..., q - 1\}$, $L^{\star}\big(\mathbb{R}^{p}, q_{r}^{\star(j)}\big) = 0$.
Consequently $\bestcut_{1}^{\star}(\theta_1^{(V)}) = \theta_{1}^{(V)} \times \{1, ..., q - 1\}$.
By definition
\begin{align*}
\mathbb{P}\big(\path_{1} \in T(\Theta,\mathscr{D}_{n})|\Theta^{(V)}=\theta^{(V)}\big) 
= \P \big(\textbf{L}_{n,\path_1}^{(\bestcut_{1}^{\star})} < \mathbf{0}\big).
\end{align*}
According to Lemma \ref{criterion_distribution_1} (case (b)),
\begin{align*}
a_{n} \textbf{L}_{n,\path_1}^{(\bestcut_1)} \stackrel[n\rightarrow\infty]{\mathscr{D}}{\longrightarrow} h_{\path_1}(\mathbf{V}),
\end{align*}
where $\mathbf{V}$ is a gaussian vector of covariance matrix $\textrm{Cov}[\mathbf{Z}]$. If $\bestcut_{1}^{\star}\big(\theta_1^{(V)}\big)$ is explicitly written
$\bestcut_{1}^{\star}\big(\theta_1^{(V)}\big) = \{(j_k, r_k)\}_{k=1,...,c_1}$, $\mathbf{Z}$ is defined as, for $k \in \{1,...,c_1\}$,
\begin{align*}
Z_{2k-1} = \frac{1}{\sqrt{p_{L,k}}} (Y - \E[Y]) \mathds{1}_{X^{(j_k)} < q_{r_k}^{\star (j_k)}} \\
Z_{2k} = \frac{1}{\sqrt{p_{R,k}}} (Y - \E[Y]) \mathds{1}_{X^{(j_k)} \geq q_{r_k}^{\star (j_k)}},
\end{align*}
$p_{L,k} = \P \big(X^{(j_k)}<q_{r_k}^{\star (j_k)}\big)$, $p_{R,k} = \P \big(X^{(j_k)} \geq q_{r_k}^{\star (j_k)}\big)$, and $h_{\path_1}$ is a multivariate quadratic form defined as
\begin{align*}
h_{\path_1}: \left( \begin{array}{c} x_1 \\ \vdots \\ x_{2c_1} \end{array} \right) \rightarrow \left( \begin{array}{c}
x_{3}^2 + x_{4}^2 - x_1^2 - x_2^2\\ \vdots \\
x_{2k-1}^2 + x_{2k}^2 - x_{1}^2 - x_{2}^2 \\ \vdots \\
x_{2c_1-1}^2 + x_{2c_1}^2 - x_{1}^2 - x_{2}^2  \end{array}\right).
\end{align*}
and the variance of each component of $h_{\path_1}(\mathbf{V})$ is strictly positive.
If $\Phi_{\theta_1^{(V)}, (j_1, r_1)}$ is the cdf of $h_{\path_1}(\mathbf{V})$, then as in \textbf{Case 3},
\begin{align*}
\lim \limits_{n \to \infty} \mathbb{P}\big(\path_{1} \in T(\Theta,\mathscr{D}_{n})|\Theta^{(V)} =\theta^{(V)}\big)
 &= \Phi_{\theta_1^{(V)}, (j_1, r_1)}(\mathbf{0})
\\ &=  \mathbb{P}\big(\path_{1} \in T^{\star}(\Theta)|\Theta^{(V)} =\theta^{(V)}\big).
\end{align*}
\end{proof}

\subsubsection{Case 2: $\path_{2}$}

\begin{proof}[Proof of Lemma \ref{criterion_residual_2}] 

Let $(j,r) \in \bestcut_{\path_1}$.
\begin{align*}
& \sqrt{a_{n}}\big(L_{a_{n}}\big(\hat{H}_{n}(\path_{1}),\hat{q}_{n,r}^{(j)}\big) - L_{a_{n}}\big(H^{\star}(\path_{1}),q_{r}^{\star(j)}\big)\big)\\
& =\sqrt{a_{n}}\big[L_{a_{n}}\big(H^{\star}(\path_{1}),\hat{q}_{n,r}^{(j)}\big)-L_{a_{n}}\big(H^{\star}(\path_{1}),q_{r}^{\star(j)}\big)\big]\\
& \qquad +\sqrt{a_{n}}\big[L_{a_{n}}\big(\hat{H}_{n}(\path_{1}),\hat{q}_{n,r}^{(j)}\big)-L_{a_{n}}\big(H^{\star}(\path_{1}),\hat{q}_{n,r}^{(j)}\big)\big].
\end{align*}
Since $(j,r) \in \bestcut_{\path_1}$, $\mathbb{P}\big(\bX \in H^{\star}(\path_{1}) | X^{(j)} < q_{r}^{\star(j)}\big)>0$
and $\mathbb{P}\big(\bX \in H^{\star}(\path_{1}) | X^{(j)} \geq q_{r}^{\star(j)}\big)>0$. Then, we can directly apply Lemma \ref{criterion_residual_1} to the first term of this decomposition, which shows that, in probability
\begin{align*}
 \lim\limits_{n \to \infty} \sqrt{a_{n}}\big(L_{a_{n}}\big(H^{\star}(\path_{1}),\hat{q}_{n,r}^{(j)}\big)-L_{a_{n}}\big(H^{\star}(\path_{1}),q_{r}^{\star(j)}\big)\big) = 0.
\end{align*}

We expand the second term
\begin{align*}
\sqrt{a_{n}}\big(L_{a_{n}}&\big(\hat{H}_{n}(\path_{1}),\hat{q}_{n,r}^{(j)}\big) - L_{a_{n}}\big(H^{\star}(\path_{1}),\hat{q}_{n,r}^{(j)}\big)\big) \\
&=\frac{\sqrt{a_{n}}}{N_{n}(\hat{H}_{n}(\path_{1}))}\sum_{i=1}^{a_{n}}\big(Y_{i}-\overline{Y}_{\hat{H}_{n}(\path_{1})}\big)^{2}\mathds{1}_{\bX_{i}\in\hat{H}_{n}(\path_{1})}\\
&- \frac{\sqrt{a_{n}}}{N_{n}(H^{\star}(\path_{1}))}\sum_{i=1}^{a_{n}}\big(Y_{i}-\overline{Y}_{H^{\star}(\path_{1})}\big)^{2}\mathds{1}_{\bX_{i}\in H^{\star}(\path_{1})}\\
&-\frac{\sqrt{a_{n}}}{N_{n}(\hat{H}_{n}(\path_{1}))}\sum_{i=1}^{a_{n}}\big(Y_{i}-\overline{Y}_{\hat{H}_{L}}\mathds{1}_{X_{i}^{(j)}<\hat{q}_{n,r}^{(j)}}-\overline{Y}_{\hat{H}_{R}}\mathds{1}_{X_{i}^{(j)}\geq\hat{q}_{n,r}^{(j)}}\big)^{2}\mathds{1}_{\bX_{i}\in\hat{H}_{n}(\path_{1})}\\
&+ \frac{\sqrt{a_{n}}}{N_{n}(H^{\star}(\path_{1}))}\sum_{i=1}^{a_{n}}\big(Y_{i}-\overline{Y}_{H_{L}^{\star}}\mathds{1}_{X_{i}^{(j)}<\hat{q}_{n,r}^{(j)}}-\overline{Y}_{H_{R}^{\star}}\mathds{1}_{X_{i}^{(j)}\geq\hat{q}_{n,r}^{(j)}}\big)^{2}\mathds{1}_{\bX_{i}\in H^{\star}(\path_{1})}
\end{align*}
with $\hat{H}_{L}=\big\{ \bx\in\hat{H}_{n}(\path_{1}):x^{(j)}<\hat{q}_{n,r}^{(j)}\big\}$, $H_{L}^{\star}=\big\{ \bx\in H^{\star}(\path_{1}):x^{(j)}<\hat{q}_{n,r}^{(j)}\big\} $, and for all $H \subseteq \R^{p}$
\begin{align*}
N_{n}(H) = \frac{1}{a_{n}} \sum_{i=1}^{a_{n}}\mathds{1}_{\bX_{i} \in H}, \quad \overline{Y}_{H} = \frac{1}{N_{n}(H)} \sum_{i=1}^{a_{n}} Y_{i}\mathds{1}_{\bX_{i} \in H}.
\end{align*} 
We define symmetrically $\hat{H}_{R}$ and $H_{R}^{\star}$. We obtain
\begin{align*}
 \sqrt{a_{n}}&\big(L_{a_{n}}\big(\hat{H}_{n}(\path_{1}),\hat{q}_{n,r}^{(j)}\big) - L_{a_{n}}\big(H^{\star}(\path_{1}),\hat{q}_{n,r}^{(j)}\big)\big) = \Delta_{n,1} + \Delta_{n,2} + \Delta_{n,3},
\end{align*}
where 
\begin{align*}
\Delta_{n,1} = \sqrt{a_{n}}\big(\overline{Y}_{H^{\star}(\path_{1})}^{2}-\overline{Y}_{\hat{H}_{n}(\path_{1})}^{2}\big),
\end{align*}
\begin{align*}
\Delta_{n,2} = \sqrt{a_{n}}\frac{\overline{Y}_{\hat{H}_{L}}^{2}N_{n}(\hat{H}_{L})N_{n}(H^{\star}(\path_{1}))-\overline{Y}_{H_{L}^{\star}}^{2}N_{n}(H_{L}^{\star})N_{n}(\hat{H}_{n}(\path_{1}))}{N_{n}(\hat{H}_{n}(\path_{1}))N_{n}(H^{\star}(\path_{1}))},
\end{align*} and 
\begin{align*}
\Delta_{n,3} = \sqrt{a_{n}}\frac{\overline{Y}_{\hat{H}_{R}}^{2}N_{n}(\hat{H}_{R})N_{n}(H^{\star}(\path_{1}))-\overline{Y}_{H_{R}^{\star}}^{2}N_{n}(H_{R}^{\star})N_{n}(\hat{H}_{n}(\path_{1}))}{N_{n}(\hat{H}_{n}(\path_{1}))N_{n}(H^{\star}(\path_{1}))}.
\end{align*}
We first consider $\Delta_{n,1}$. Simple calculations show that 
\begin{align*}
\Delta_{n,1}  =& \frac{\sqrt{a_{n}}}{N_{n}(H^{\star}(\path_{1}))^{2}N_{n}(\hat{H}_{n}(\path_{1}))^{2}} \\
&\times \sum_{i,k,l,m} Y_{i}Y_{k}\big[\mathds{1}_{\bX_{i}\in H^{\star}(\path_{1}),\bX_{k}\in H^{\star}(\path_{1}),\bX_{l}\in\hat{H}_{n}(\path_{1}),\bX_{m}\in\hat{H}_{n}(\path_{1})}\\
& \qquad -\mathds{1}_{\bX_{i}\in\hat{H}_{n}(\path_{1}),\bX_{k}\in\hat{H}_{n}(\path_{1}),\bX_{l}\in H^{\star}(\path_{1}),\bX_{m}\in H^{\star}(\path_{1})}\big]
\end{align*}
We consider the case $s_{1}=L$, ($s_{1}=R$ is similar). Since  $Y_{i}\in\big\{ 0,1\big\}$,
\begin{align*}
|\Delta_{n,1}| \leq&\frac{\sqrt{a_{n}}}{N_{n}(H^{\star}(\path_{1}))^{2}N_{n}(\hat{H}_{n}(\path_{1}))^{2}} \\
& \times \sum_{i,k,l,m}\big|\mathds{1}_{X_{i}^{(j_{1})}<q_{r_{1}}^{\star(j_{1})},X^{(j_{1})}_{k}<q_{r_{1}}^{\star(j_{1})},X^{(j_{1})}_{l}<\hat{q}_{n,r_{1}}^{(j_{1})},X^{(j_{1})}_{m}<\hat{q}_{n,r_{1}}^{(j_{1})}}\\
& \qquad -\mathds{1}_{X^{(j_{1})}_{i}<\hat{q}_{n,r_{1}}^{(j_{1})},X^{(j_{1})}_{k}<\hat{q}_{n,r_{1}}^{(j_{1})},X^{(j_{1})}_{l}<q_{r_{1}}^{\star(j_{1})},X^{(j_{1})}_{m}<q_{r_{1}}^{\star(j_{1})}}\big|
\end{align*}

As in the proof of Lemma  \ref{criterion_residual_1}, according to Lemma \ref{convergence_quantile},  $\lim\limits_{n \to \infty} \Delta_{n,1} = 0$, in probability. Since $\Delta_{n,2}$ and $\Delta_{n,3}$ are the same quantities computed on each of the two daughter nodes, we study $\Delta_{n,2}$ only. 
\begin{align*}
\Delta_{n,2} & =\frac{\sqrt{a_{n}} (N_{n}(\hat{H}_{L})N_{n}(H_{L}^{\star}))^{-1}}{N_{n}(\hat{H}_{n}(\path_{1}))N_{n}(H^{\star}(\path_{1}))} \\
& \times \sum_{i,k,l,m}Y_{i}Y_{k}\big[\mathds{1}_{\bX_{i}\in\hat{H}_{L},\bX_{k}\in\hat{H}_{L},\bX_{l}\in H_{L}^{\star},\bX_{m}\in H^{\star}(\path_{1})}\\
& \qquad \quad -\mathds{1}_{\bX_{i}\in H_{L}^{\star},\bX_{k}\in H_{L}^{\star},\bX_{l}\in\hat{H}_{L},\bX_{m}\in\hat{H}_{n}(\path_{1})}\big] \\
& =\frac{\sqrt{a_{n}} (N_{n}(\hat{H}_{L})N_{n}(H_{L}^{\star}))^{-1}}{N_{n}(\hat{H}_{n}(\path_{1}))N_{n}(H^{\star}(\path_{1}))} 
\sum_{i,k,l,m}Y_{i}Y_{k}
\mathds{1}_{X^{(j)}_{i}<\hat{q}_{n,r}^{(j)},X^{(j)}_{k}<\hat{q}_{n,r}^{(j)},X^{(j)}_{l}<\hat{q}_{n,r}^{(j)}}\\
& \quad \times \big[\mathds{1}_{X^{(j_{1})}_{i}<\hat{q}_{n,r_{1}}^{(j_{1})},X^{(j_{1})}_{k}<\hat{q}_{n,r_{1}}^{(j_{1})},X^{(j_{1})}_{l}<q_{r_{1}}^{\star(j_{1})},X^{(j_{1})}_{m}<q_{r_{1}}^{\star(j_{1})}} \\
& \qquad -\mathds{1}_{X^{(j_{1})}_{i}<q_{r_{1}}^{\star(j_{1})},X^{(j_{1})}_{k}<q_{r_{1}}^{\star(j_{1})},X^{(j_{1})}_{l}<\hat{q}_{n,r_{1}}^{(j_{1})},X^{(j_{1})}_{m}<\hat{q}_{n,r_{1}}^{(j_{1})}}\big].
\end{align*}
Therefore
\begin{align*}
|\Delta_{n,2}| & \leq \frac{\sqrt{a_{n}} (N_{n}(\hat{H}_{L})N_{n}(H_{L}^{\star}))^{-1}}{N_{n}(\hat{H}_{n}(\path_{1}))N_{n}(H^{\star}(\path_{1}))} \\
& \times \sum_{i,k,l,m}\big|\mathds{1}_{X^{(j_{1})}_{i}<\hat{q}_{n,r_{1}}^{(j_{1})},X^{(j_{1})}_{k}<\hat{q}_{n,r_{1}}^{(j_{1})},X^{(j_{1})}_{l}<q_{r_{1}}^{\star(j_{1})},X^{(j_{1})}_{m}<q_{r_{1}}^{\star(j_{1})}}\\
& \qquad -\mathds{1}_{X^{(j_{1})}_{i}<q_{r_{1}}^{\star(j_{1})},X^{(j_{1})}_{k}<q_{r_{1}}^{\star(j_{1})},X^{(j_{1})}_{l}<\hat{q}_{n,r_{1}}^{(j_{1})},X^{(j_{1})}_{m}<\hat{q}_{n,r_{1}}^{(j_{1})}}\big|.
\end{align*}
As in the proof of Lemma \ref{criterion_residual_1}, according to Lemma \ref{convergence_quantile}, $\lim\limits_{n \to \infty} \Delta_{n,2} = 0$, in probability, which concludes the proof, since $\Delta_{n,3}$ can be studied in the same manner.  
\end{proof}

\begin{proof}[Proof of Lemma \ref{criterion_consistency_2}]
Let $(j,r) \in \bestcut_{\path_1}$.
\begin{align}
L_{a_{n}}\big(\hat{H}_{n}(\path_{1}),\hat{q}_{n,r}^{(j)}\big)
=&L_{a_{n}}\big(H^{\star}(\path_{1}),q_{r}^{\star(j)}\big)\nonumber \\
& +\big[L_{a_{n}}\big(\hat{H}_{n}(\path_{1}),\hat{q}_{n,r}^{(j)}\big)
-L_{a_{n}}\big(H^{\star}(\path_{1}),q_{r}^{\star(j)}\big)\big] \label{proof_eq_decomp_path2}
\end{align}
According to Lemma \ref{criterion_residual_2}, the second term in equation (\ref{proof_eq_decomp_path2}) converges to $0$ in probability. From the law of large numbers, in probability, 
\begin{align*}
\lim \limits_{n \to \infty} L_{a_{n}}\big(H^{\star}(\path_{1}),q_{r}^{\star(j)}\big)
=L^{\star}\big(H^{\star}(\path_{1}),q_{r}^{\star(j)}\big),
\end{align*}
which concludes the proof. 
\end{proof}

\begin{proof}[Proof of Lemma \ref{criterion_distribution_2}]
Similar to the case with $\path_1$ (Lemma \ref{criterion_consistency_1}), where Lemma \ref{criterion_residual_2} is used instead of Lemma \ref{criterion_residual_1}.
\end{proof}

\begin{proof}[Proof of Lemma \ref{convergence_pn_theta_2}]

Consider a path $\path_2 = \{(j_1, r_1, L), (j_2, r_2, \cdot)\}$. Set $\theta^{(V)} = \big(\theta_1^{(V)}, \theta_2^{(V)}\big)$, a realization of  the randomization of the split directions at the root node and its left child node. Then, $\theta_1^{(V)}$ and $\theta_2^{(V)}$  denote the set of eligible variables for 
respectively the first and second split. We also consider $\bestcut_{\path_1}\big(\theta_2^{(V)}\big) \subset \bestcut_{\path_1}$ the set of eligible second splits.

Recalling that the best split in a random tree is the one maximizing the CART-splitting criterion, conditional on $\Theta^{(V)}=\theta^{(V)}$,
\begin{align*}
\left\lbrace\path_{2} \in T(\Theta,\mathscr{D}_{n})\right\rbrace 
 = \underset{\substack{(j, r) \in \theta_{1}^{(V)}  \times \{1, ..., q - 1\} \\ \qquad \quad \setminus (j_1, r_1)}}
{\bigcap}\big\{ L_{a_{n}}\big(\mathds{R}^{p},\hat{q}_{n,r_{1}}^{(j_{1})}\big)>L_{a_{n}}\big(\mathds{R}^{p},\hat{q}_{n,r}^{(j)}\big)\big\} \\
 \qquad \underset{(j, r) \in \bestcut_{\path_1}(\theta_2^{(V)}) \setminus (j_2, r_2)}
{\bigcap}\big\{ L_{a_{n}}\big(\hat{H}_{n}(\path_{1}),\hat{q}_{n,r_{2}}^{(j_{2})}\big)
>L_{a_{n}}\big(\hat{H}_{n}(\path_{1}),\hat{q}_{n,r}^{(j)}\big)\big\}
\end{align*}
Recall that $\bestcut_{1}^{\star}\big(\theta_{1}^{(V)}\big) = \underset{ (j,r) \in\theta_{1}^{(V)} \times \{1, ..., q - 1\}}{\textrm{argmax}} L^{\star}\big(\mathbb{R}^{p},q_{r}^{\star(j)}\big) $, and similarly
\begin{align*}
\bestcut_{2}^{\star}\big(\theta_{2}^{(V)}\big) = \underset{ (j,r) \in \bestcut_{\path_1}(\theta_2^{(V)})}{\textrm{argmax}}
 L^{\star}\big(H^{\star}(\path_{1}),q_{r}^{\star(j)}\big).    
\end{align*}

Obviously if $(j_1,r_1) \notin \theta_{1}^{(V)} \times \{1, ..., q - 1\}$ or 
$(j_2,r_2) \notin \bestcut_{\path_1}(\theta_2^{(V)}) $, the probability to select $\path_2$ in the empirical and theoretical tree is null.
In the sequel, we assume that $(j_1,r_1) \in \theta_{0}^{(V)} \times \{1, ..., q - 1\}$ and 
$(j_2,r_2) \in \bestcut_{\path_1}\big(\theta_2^{(V)}\big)$ and distinguish between cases,
 depending on whether $(j_1, r_1) \in \bestcut_{1}^{\star}\big(\theta_{1}^{(V)}\big)$
 or not and $(j_2, r_2) \in \bestcut_2^{\star}\big(\theta_2^{(V)}\big)$ or not, as well as the cardinality of $\bestcut_1^{\star}\big(\theta_1^{(V)}\big)$ and
$\bestcut_2^{\star}\big(\theta_2^{(V)}\big)$, and whether the maximum of the theoretical CART-splitting criterion is null or not.

\paragraph{Case 1} We assume that $(j_1, r_1) \notin \bestcut_{1}^{\star}\big(\theta_{1}^{(V)}\big)$. Hence, the theoretical decision tree satisfies 
$$\mathbb{P}\big(\path_{2} \in T^{\star}(\Theta)|\Theta^{(V)}=\theta^{(V)}\big) = \mathbb{P}\big(\path_{1} \in T^{\star}(\Theta)|\Theta^{(V)}=\theta^{(V)}\big) = 0.$$
According to Lemma \ref{convergence_pn_theta_1}, we have
\begin{align*}
\lim \limits_{n \to \infty} \mathbb{P}&\big(\path_{2} \in T(\Theta,\mathscr{D}_{n})|\Theta^{(V)}=\theta^{(V)}\big)\\
& \leq \lim \limits_{n \to \infty} \mathbb{P}\big(\path_{1} \in T(\Theta,\mathscr{D}_{n})|\Theta^{(V)}=\theta^{(V)}\big)\\
& = 0\\
& = \mathbb{P}\big(\path_{2} \in T^{\star}(\Theta)|\Theta^{(V)}=\theta^{(V)}\big).    
\end{align*}

\paragraph{Case 2} We assume that $\big(j_{2}, r_{2}\big) \notin \bestcut_{2}^{\star}\big(\theta_{2}^{(V)}\big)$. Again, for the theoretical decision tree, $$\mathbb{P}\big(\path_{2} \in T^{\star}(\Theta)|\Theta^{(V)}=\theta^{(V)}\big) = 0.$$
Letting $\big(j^{\star}, r^{\star} \big) \in \bestcut_{2}^{\star}\big(\theta_{2}^{(V)}\big)$, 
\begin{align*}
\varepsilon
=L^{\star}\big(H^{\star}(\path_{1}),q_{r^{\star}}^{\star(j^{\star})}\big) - L^{\star}\big(H^{\star}(\path_{1}),q_{r_{2}}^{\star (j_{2})}\big).
\end{align*}
Therefore, 
\begin{align*}
 \mathbb{P}\big(\path_{2}& \in T(\Theta,\mathscr{D}_{n})|\Theta^{(V)}=\theta^{(V)}\big) \\
& \leq \mathbb{P}\big(L_{a_{n}}\big(\hat{H}_{n}(\path_{1}),\hat{q}_{n,r_{2}}^{(j_{2})}\big)
>L_{a_{n}}\big(H^{\star}(\path_{1}),\hat{q}_{n,r^{\star}}^{(j^{\star})}\big)\big) \\
& \leq \mathbb{P}\big(L_{a_{n}}\big(\hat{H}_{n}(\path_{1}),\hat{q}_{n,r_{2}}^{(j_{2})}\big)
- L^{\star}\big(H^{\star}(\path_{1}),q_{r_{2}}^{\star (j_{2})}\big) - \epsilon  \\ & \qquad 
>L_{a_{n}}\big(\hat{H}_{n}(\path_{1}),\hat{q}_{n,r^{\star}}^{(j^{\star})}\big)
- L^{\star}\big(H^{\star}(\path_{1}),q_{r^{\star}}^{\star(j^{\star})}\big)\big) \\
& \leq \mathbb{P}\big( L_{a_{n}}\big(\hat{H}_{n}(\path_{1}),\hat{q}_{n,r_{2}}^{(j_{2})}\big)
- L^{\star}\big(H^{\star}(\path_{1}),q_{r_{2}}^{\star (j_{2})}\big)  \\ 
& \qquad -  \big( L_{a_{n}}\big(\hat{H}_{n}(\path_{1}),\hat{q}_{n,r^{\star}}^{(j^{\star})}\big)
- L^{\star}\big(H^{\star}(\path_{1}),q_{r^{\star}}^{\star(j^{\star})}\big) \big)
> \epsilon \big).
\end{align*}
Consequently, according to Lemma \ref{criterion_consistency_2}, 
\begin{align*}
\lim \limits_{n \to \infty} \mathbb{P}\big(\path_{2} \in T(\Theta,\mathscr{D}_{n})|\Theta^{(V)}=\theta^{(V)}\big) = 0
= \mathbb{P}\big(\path_{2} \in T^{\star}(\Theta)|\Theta^{(V)}=\theta^{(V)}\big).
\end{align*}

\paragraph{Case 3} We assume that
$(j_1, r_1) \in \bestcut_{1}^{\star}\big(\theta_{1}^{(V)}\big)$
and $\bestcut_{2}^{\star}\big(\theta_{2}^{(V)}\big) = \{ (j_{2}, r_{2}) \}$,
i.e. $(j_{2}, r_{2})$ is the unique maximum of the theoretical CART-splitting criterion for the cell $H^{\star}(\path_{1})$. By definition of the theoretical decision tree,
\begin{align*}  
\mathbb{P}\big(\path_{2} \in T^{\star}(\Theta)|\Theta^{(V)}=\theta^{(V)}\big)
= \mathbb{P}\big(\path_{1} \in T^{\star}(\Theta)|\Theta^{(V)}=\theta^{(V)}\big)
\end{align*}
Conditional on $\{\Theta^{(V)}=\theta^{(V)}\}$, 
\begin{align}
\{ \path_{2}& \in T(\Theta,\mathscr{D}_{n}) \} 
=  \big\{ \path_{1} \in T(\Theta,\mathscr{D}_{n}) \big\} \nonumber \\
& \qquad 
\underset{(j, r) \in \bestcut_{\path_1}(\theta_2^{(V)})  \setminus (j_2, r_2)}
{\bigcap}\big\{ L_{a_{n}}\big(\hat{H}_{n}(\path_{1}),\hat{q}_{n,r_{2}}^{(j_{2})}\big)
>L_{a_{n}}\big(\hat{H}_{n}(\path_{1}),\hat{q}_{n,r}^{(j)}\big)\big\}. \label{proof_twosplits_prob_ineq3}
\end{align}
Consequently, 
\begin{align}
& \mathbb{P}\big(\path_{2} \in T(\Theta,\mathscr{D}_{n})|\Theta^{(V)}=\theta^{(V)}\big) \nonumber \\
& \geq \mathbb{P} \big( \path_{1} \in T(\Theta,\mathscr{D}_{n})\big |\Theta^{(V)}=\theta^{(V)} \big) \nonumber \\
& - \sum_{(j, r) \in \bestcut_{\path_1}(\theta_2^{(V)})  \setminus (j_2, r_2)}
\mathbb{P}\big( L_{a_{n}}\big(\hat{H}_{n}(\path_{1}),\hat{q}_{n,r_{2}}^{(j_{2})}\big)
\leq L_{a_{n}}\big(\hat{H}_{n}(\path_{1}),\hat{q}_{n,r}^{(j)}\big) \big). \label{proof_twosplits_prob_ineq}
\end{align}
For $(j, r) \in \bestcut_{\path_1}\big(\theta_2^{(V)}\big)  \setminus (j_2, r_2)$,
\begin{align}
L^{\star}\big(H^{\star}(\path_{1}),q_{r_{2}}^{\star(j_{2})}\big)
-  L^{\star}\big(H^{\star}(\path_{1}),q_{r}^{\star(j)}\big)> 0. \label{proof_twosplits_prob_ineq2}
\end{align}
Thus, using inequalities (\ref{proof_twosplits_prob_ineq}) and (\ref{proof_twosplits_prob_ineq2}), and according to 
Lemma \ref{criterion_consistency_2}, 
\begin{align*}
\lim \limits_{n \to \infty} \mathbb{P}\big( L_{a_{n}}\big(\hat{H}_{n}(\path_{1}),\hat{q}_{n,r_{2}}^{(j_{2})}\big) 
\leq L_{a_{n}}\big(\hat{H}_{n}(\path_{1}),\hat{q}_{n,r}^{(j)}\big) \big) = 0,
\end{align*}
and thus, using (\ref{proof_twosplits_prob_ineq3}) and (\ref{proof_twosplits_prob_ineq}), 
\begin{align*}
\lim \limits_{n \to \infty}& \mathbb{P}\big(\path_{2} \in T(\Theta,\mathscr{D}_{n})|\Theta^{(V)}=\theta^{(V)}\big) \\
& = \lim \limits_{n \to \infty} \mathbb{P} \big( \path_{1} \in T(\Theta,\mathscr{D}_{n}) |\Theta^{(V)}=\theta^{(V)} \big)
 = \mathbb{P} \big( \path_{1} \in T^{\star}(\Theta)|\Theta^{(V)}=\theta^{(V)} \big) \\
& = \mathbb{P}\big(\path_{2} \in T^{\star}(\Theta)|\Theta^{(V)}=\theta^{(V)}\big),
\end{align*}
where the second inequality is a direct consequence of Lemma \ref{convergence_pn_theta_1}.

\paragraph{Case 4} For the first split, we assume
$(j_1, r_1) \in \bestcut_{1}^{\star}\big(\theta_{1}^{(V)}\big)$ with $L^{\star}\big(\R^p, q^{\star(j_1)}_{r_1}\big) > 0$, and for the second split,
$\big(j_{2}, r_{2}\big) \in \bestcut_{2}^{\star}\big(\theta_{2}^{(V)}\big)$ with $| \bestcut_{2}^{\star}\big(\theta_{2}^{(V)}\big) \big| > 1$
and $L^{\star}\big(H^{\star}(\path_1), q^{\star(j_2)}_{r_2}\big) > 0$.

On one hand, conditional on the event $\{ \Theta^{(V)} = \theta^{(V)} \}$,
\begin{align}
\left\lbrace \path_{2} \in T(\Theta,\mathscr{D}_{n}) \right\rbrace =
 \underset{\substack{(j, r) \in \theta_1^{(V)} \times \{1, ..., q - 1\} \\ \qquad \quad\setminus (j_1, r_1)}}
 {\bigcap}\big\{ L_{a_{n}}\big(\mathds{R}^{p},\hat{q}_{n,r_{1}}^{(j_{1})}\big)
>L_{a_{n}}\big(\mathds{R}^{p},\hat{q}_{n,r}^{(j)}\big)\big\} \nonumber \\
 \qquad \underset{(j, r) \in \bestcut_{\path_1}(\theta_2^{(V)}) \setminus (j_2, r_2)}
{\bigcap}\big\{ L_{a_{n}}\big(\hat{H}_{n}(\path_{1}),\hat{q}_{n,r_{2}}^{(j_{2})}\big)
>L_{a_{n}}\big(\hat{H}_{n}(\path_{1}),\hat{q}_{n,r}^{(j)}\big)\big\}. \label{proof_thm_twosplits_case4_eq1}
\end{align}
Using equation (\ref{proof_thm_twosplits_case4_eq1}) to find a subset and a superset of $\left\lbrace \path_{2} \in T(\Theta,\mathscr{D}_{n}) \right\rbrace$, we obtain
\begin{align*}
0 & \geq  \mathbb{P}\big(\path_{2} \in T(\Theta,\mathscr{D}_{n})|\Theta^{(V)}=\theta^{(V)}\big) \\
& \enskip  - \mathbb{P}\Bigg(\underset{(j, r) \in \bestcut_{1}^{\star}(\theta_{1}^{(V)}) \setminus (j_1, r_1)}
{\bigcap}\big\{ L_{a_{n}}\big(\mathds{R}^{p},\hat{q}_{n,r_{1}}^{(j_{1})}\big)
>L_{a_{n}}\big(\mathds{R}^{p},\hat{q}_{n,r}^{(j)}\big)\big\}  \\ 
& \qquad \underset{(j, r) \in \bestcut_{2}^{\star}(\theta_{2}^{(V)}) \setminus (j_2, r_2)}
{\bigcap}\big\{ L_{a_{n}}\big(\hat{H}_{n}(\path_{1}),\hat{q}_{n,r_{2}}^{(j_{2})}\big)
>L_{a_{n}}\big(\hat{H}_{n}(\path_{1}),\hat{q}_{n,r}^{(j)}\big)\big\} \Bigg)   \\
& \geq \sum_{(j,r) \in\theta_{1}^{(V)} \times \{1, ..., q - 1\} \setminus \bestcut_{1}^{\star}{(\theta_{1}^{(V)})}}
\mathbb{P}\big( L_{a_{n}}\big(\mathds{R}^{p},\hat{q}_{n,r_{1}}^{(j_{1})}\big)
\leq L_{a_{n}}\big(\mathds{R}^{p},\hat{q}_{n,r}^{(j)}\big) \big) \\
& \enskip + \sum_{(j,r) \in\theta_{2}^{(V)} \times \{1, ..., q - 1\} \setminus \bestcut_{2}^{\star}{(\theta_{2}^{(V)})}}
\mathbb{P}\big( L_{a_{n}}\big(\hat{H}_{n}(\path_{1}),\hat{q}_{n,r_{2}}^{(j_{2})}\big)
\leq L_{a_{n}}\big(\hat{H}_{n}(\path_{1}),\hat{q}_{n,r}^{(j)}\big) \big)
\end{align*}
We  proved in \textbf{Case 3} that the limit of the last two terms of the previous inequality is zero, in probability. Therefore, 
\begin{align}
& \lim \limits_{n \to \infty} \mathbb{P}\big(\path_{2} \in T(\Theta,\mathscr{D}_{n})|\Theta^{(V)}=\theta^{(V)}\big) \nonumber \\
 & = \lim \limits_{n \to \infty} \mathbb{P}\Bigg(\underset{(j, r) \in \bestcut_{1}^{\star}(\theta_{1}^{(V)}) \setminus (j_1, r_1)}
 {\bigcap}\big\{ L_{a_{n}}\big(\mathds{R}^{p},\hat{q}_{n,r_{1}}^{(j_{1})}\big)
>L_{a_{n}}\big(\mathds{R}^{p},\hat{q}_{n,r}^{(j)}\big)\big\}  \nonumber \\ 
& \qquad \underset{(j, r) \in \bestcut_{2}^{\star}(\theta_{2}^{(V)}) \setminus (j_2, r_2)}
{\bigcap}\big\{ L_{a_{n}}\big(\hat{H}_{n}(\path_{1}),\hat{q}_{n,r_{2}}^{(j_{2})}\big)
>L_{a_{n}}\big(\hat{H}_{n}(\path_{1}),\hat{q}_{n,r}^{(j)}\big)\big\} \Bigg). \label{proof_twosplits_eq2}
\end{align}

We define the random vector $\textbf{L}_{n,\path_2}^{(\bestcut_1^{\star}, \bestcut_2^{\star})}$ 
(we drop $\theta^{(V)}$ to lighten notations)
where each component is the difference between the empirical CART-splitting criterion for the splits $(j, r) \in \bestcut_1^{\star} \setminus (j_1, r_1)$ and $(j_1, r_1)$
 for the first $|\bestcut_1^{\star}| - 1$ components, and for the splits $(j, r) \in \bestcut_2^{\star} \setminus (j_2, r_2)$ and $(j_2, r_2)$ for the remaining $|\bestcut_2^{\star}| - 1$ components,  i.e.,
\begin{align*}
\textbf{L}_{n,\path_2}^{(\bestcut_1^{\star}, \bestcut_2^{\star})} = \left(\begin{array}{c}
\big[L_{a_{n}}\big(\mathbb{R}^{p},\hat{q}_{n,r}^{(j)}\big)-L_{a_{n}}\big(\mathbb{R}^{p},\hat{q}_{n,r_1}^{(j_1)}\big)\big]_{(j, r) \in \bestcut_1^{\star} \setminus (j_1, r_1)} \\
\big[L_{a_{n}}\big(\hat{H}_n(\path_{1}),\hat{q}_{n,r}^{(j)}\big)-L_{a_{n}}\big(\hat{H}_n(\path_{1}),\hat{q}_{n,r_2}^{(j_2)}\big)\big]_{(j, r) \in \bestcut_2^{\star} \setminus (j_2, r_2)} \\
\end{array}\right).
\end{align*}

Then, we can write
\begin{align}
& \mathbb{P}\Bigg(\underset{(j, r) \in \bestcut_{1}^{\star}(\theta_{1}^{(V)}) \setminus (j_1, r_1)}
{\bigcap}\big\{ L_{a_{n}}\big(\mathds{R}^{p},\hat{q}_{n,r_{1}}^{(j_{1})}\big)
>L_{a_{n}}\big(\mathds{R}^{p},\hat{q}_{n,r}^{(j)}\big)\big\} \nonumber \\ 
& \qquad 
\underset{(j, r) \in \bestcut_{2}^{\star}(\theta_{2}^{(V)}) \setminus (j_2, r_2)}{\bigcap}\big\{ L_{a_{n}}\big(\hat{H}_{n}(\path_{1}),\hat{q}_{n,r_{2}}^{(j_{2})}\big)
>L_{a_{n}}\big(\hat{H}_{n}(\path_{1}),\hat{q}_{n,r}^{(j)}\big)\big\} \Bigg) \nonumber\\
&= \mathbb{P} \big( \textbf{L}_{n,2}^{(\bestcut_{1}^{\star}, \bestcut_{2}^{\star})} < \textbf{0} \big) \label{proof_twosplits_eq3}
\end{align}

According to Lemma \ref{criterion_distribution_2},
\begin{align*}
\sqrt{a_{n}} \textbf{L}_{n,\path_2}^{(\bestcut_1^{\star}, \bestcut_2^{\star})} \stackrel[n\rightarrow\infty]{\mathscr{D}}{\longrightarrow}\mathcal{N}(0, \Sigma)
\end{align*}
where for $l, l' \in \{1,2\}$, for all $(j,r) \in \bestcut_l^{\star} \setminus (j_l,r_l)$,
$(j', r') \in \bestcut_{l'}^{\star} \setminus (j_{l'}, r_{l'})$, each element of the covariance matrix $\Sigma$ is defined by
$ \Sigma_{(j,r,l), (j',r',l')} = \textrm{Cov}[Z_{j,r,l}, Z_{j',r',l'}],$ with
\begin{align*}
Z_{j,r,l}= & \frac{1}{\P(\bX \in H_l)} \big(Y - \mu_{L,r_l}^{(j_l)}\mathds{1}_{X^{(j_{l})}<q_{r_{l}}^{\star(j_{l})}}
-\mu_{R,r_l}^{(j_l)}\mathds{1}_{X^{(j_{l})}\geq q_{r_{l}}^{\star(j_{l})}}\big)^{2} \mathds{1}_{\bX \in H_l}  \\
&- \frac{1}{\P(\bX \in H_l)} \big(Y - \mu_{L,r}^{(j)}\mathds{1}_{X^{(j)}<q_{r}^{\star(j)}}
-\mu_{R,r}^{(j)}\mathds{1}_{X^{(j)}\geq q_{r}^{\star(j)}}\big)^{2} \mathds{1}_{\bX \in H_l},
\end{align*}
$\mu_{L,r}^{(j)} = \E\big[Y | X^{(j)}<q_{r}^{\star(j)}, \bX \in H_l\big]$, 
$\mu_{R,r}^{(j)} = \E\big[Y | X^{(j)} \geq q_{r}^{\star(j)}, \bX \in H_l\big]$, and the variance of $Z_{j,r,l}$ is strictly positive.

Letting $\Phi_{\path_{1}, \theta^{(V)}, (j_{2}, r_{2})}$ be the c.d.f. of the multivariate normal distribution with covariance matrix $\Sigma$,
and using equalities (\ref{proof_twosplits_eq2}) and (\ref{proof_twosplits_eq3}),
\begin{align*}
\lim \limits_{n \to \infty} \mathbb{P}\big(\path_{2} \in T(\Theta,\mathscr{D}_{n})|\Theta^{(V)}
=\theta^{(V)}\big) = \Phi_{\path_{1}, \theta^{(V)}, (j_{2}, r_{2})} (\textbf{0}).
\end{align*}
We can check that 
\begin{align*}
\sum_{(j,r) \in \bestcut_2^{\star}(\theta^{(V)})} 
 \Phi_{\path_{1}, \theta^{(V)}, (j, r)} (\textbf{0})
=\mathbb{P}\big(\path_{1} \in T^{\star}(\Theta)|\Theta^{(V)}=\theta^{(V)}\big). 
\end{align*}
In the theoretical random forest, the first cut $(j_1, r_1)$ is randomly selected with probability
$\mathbb{P}\big(\path_{1} \in T^{\star}(\Theta)|\Theta^{(V)}=\theta^{(V)}\big)$ (see the proof of Lemma \ref{convergence_pn_theta_1}).
For the second cut, according to Definition \ref{def_proba_equality_2}, if $\bestcut_{2}^{\star}\big(\theta_{2}^{(V)}\big)$ has multiple elements,
$(j_{2}, r_{2})$ is randomly drawn with probability 
\begin{align*}
\frac{\Phi_{\path_{1}, \theta^{(V)}, (j_{2}, r_{2})} (\textbf{0})}
{\mathbb{P}\big(\path_{1} \in T^{\star}(\Theta)|\Theta^{(V)}=\theta^{(V)}\big) }
\end{align*}
Since the random selection at the root node of the tree and its children nodes are independent in the theoretical algorithm,
$\path_{2}$ is selected with probability
\begin{align*}
\mathbb{P}\big(\path_{1} \in T^{\star}(\Theta)|\Theta^{(V)}=\theta^{(V)}\big)
\times \frac{\Phi_{\path_{1}, \theta^{(V)}, (j_{2}, r_{2})} (\textbf{0})}
{\mathbb{P}\big(\path_{1} \in T^{\star}(\Theta)|\Theta^{(V)}=\theta^{(V)}\big) } \\
= \Phi_{\path_{1}, \theta^{(V)}, (j_{2}, r_{2})} (\textbf{0}).
\end{align*}
Ultimately, 
\begin{align*}
\lim \limits_{n \to \infty} \mathbb{P}\big(\path_{2} \in T(\Theta,\mathscr{D}_{n})|\Theta^{(V)} = \theta^{(V)}\big)
& = \mathbb{P}\big(\path_{2} \in T^{\star}(\Theta)|\Theta^{(V)}=\theta^{(V)}\big) \\
& = \Phi_{\path_{1}, \theta^{(V)}, (j_{2}, r_{2})} \big(\textbf{0}\big).
\end{align*}

\paragraph{Case 5} We assume that $(j_1, r_1) \in \bestcut_{1}^{\star}\big(\theta_{1}^{(V)}\big)$ and
$(j_{2}, r_{2}) \in \bestcut_{2}^{\star}\big(\theta_{2}^{(V)}\big)$, and that the theoretical CART-splitting criterion is null for both splits:
$L^{\star}\big(\R^p, q^{\star(j_1)}_{r_1}\big) = 0$ and  $L^{\star}\big(H^{\star}(\path_1), q^{\star(j_2)}_{r_2}\big) = 0$. 

Consequently $\bestcut_{1}^{\star}\big(\theta_1^{(V)}\big) = \theta_{1}^{(V)} \times \{1, ..., q - 1\}$,
and $\bestcut_{2}^{\star}\big(\theta_2^{(V)}\big) = \bestcut_{\path_1}\big(\theta_2^{(V)}\big)$.
Using the same notations defined in \textbf{Case 4}, we have by definition
\begin{align*}
\mathbb{P}\big(\path_{1} \in T(\Theta,\mathscr{D}_{n})|\Theta^{(V)}=\theta^{(V)}\big) 
= \P \big(\textbf{L}_{n,\path_2}^{(\bestcut_{1}^{\star}, \bestcut_{2}^{\star})} < \mathbf{0}\big).
\end{align*}
According to Lemma \ref{criterion_distribution_2} (case (b)),
\begin{align*}
a_{n} \textbf{L}_{n,\path_2}^{(\bestcut_1^{\star}, \bestcut_2^{\star})} \stackrel[n\rightarrow\infty]{\mathscr{D}}{\longrightarrow} h_{\path_2}(\mathbf{V}),
\end{align*}
where $\mathbf{V}$ is a gaussian vector of covariance matrix $\textrm{Cov}[\mathbf{Z}]$. If $\bestcut_1^{\star}$ and $\bestcut_2^{\star}$ are explicitly written
$\bestcut_1^{\star} = \{(j_k, r_k)\}_{k \in J_1}$, and $\bestcut_2^{\star} = \{(j_k, r_k)\}_{k \in J_2}$, with $J_1 = \{1,...,c_1+1\} \setminus 2$
and $J_2 = \{2\} \cup \{c_1+2,...,c_1+c_2\}$, $\mathbf{Z}$ is defined as, for $l \in \{1,2\}$ and $k \in J_l$,
\begin{align*}
Z_{2k-1} = \frac{1}{\sqrt{p_{L,k}\P(\bX \in H_l)}} (Y - \E[Y|\bX \in H_l]) \mathds{1}_{X^{(j_k)} < q_{r_k}^{\star (j_k)}}\mathds{1}_{\bX \in H_l}, \\
Z_{2k} = \frac{1}{\sqrt{p_{R,k}\P(\bX \in H_l)}} (Y - \E[Y|\bX \in H_l]) \mathds{1}_{X^{(j_k)} \geq q_{r_k}^{\star (j_k)}}\mathds{1}_{\bX \in H_l},
\end{align*}
$p_{L,k} = \P \big(X^{(j_k)}<q_{r_k}^{\star (j_k)}, \bX \in H_l\big)$, $p_{R,k} = \P\big (X^{(j_k)} \geq q_{r_k}^{\star (j_k)}, \bX \in H_l\big)$.
$h_{\path_2}$ is a multivariate quadratic form defined as
\begin{align*}
h_{\path_2}: \left( \begin{array}{c} x_1 \\ \vdots \\ x_{2(c_1 + c_2)}  \end{array} \right) \rightarrow \left( \begin{array}{c}
x_{5}^2 + x_{6}^2 - x_1^2 - x_2^2\\ \vdots \\
x_{2c_1+1}^2 + x_{2c_1+2}^2 - x_{1}^2 - x_{2}^2\\ 
x_{2c_1+3}^2 + x_{2c_1+4}^2 - x_{3}^2 - x_{4}^2\\ \vdots \\
x_{2(c_1+c_2)-1}^2 + x_{2(c_1+c_2)}^2 - x_{3}^2 - x_{4}^2  \end{array}\right),
\end{align*}
and the variance of each component of $h_{\path_2}(\mathbf{V})$ is strictly positive.

$\Phi_{\path_{1}, \theta^{(V)}, (j_{2}, r_{2})}$ is now defined as the cdf of $h_{\path_2}(\mathbf{V})$,
and the end of the proof is identical to \textbf{Case 4}. We conclude
\begin{align*}
\lim \limits_{n \to \infty} \mathbb{P}\big(\path_{2} \in T(\Theta,\mathscr{D}_{n})|\Theta^{(V)} = \theta^{(V)}\big)
& = \mathbb{P}\big(\path_{2} \in T^{\star}(\Theta)|\Theta^{(V)}=\theta^{(V)}\big) \\
& = \Phi_{\path_{1}, \theta^{(V)}, (j_{2}, r_{2})} (\textbf{0}).
\end{align*}

\paragraph{Case 6} We assume $(j_1, r_1) \in \bestcut_{1}^{\star}\big(\theta_{1}^{(V)}\big)$, $(j_2, r_2) \in \bestcut_{2}^{\star}\big(\theta_{2}^{(V)}\big)$ and $\big|\bestcut_{2}^{\star}\big(\theta_{2}^{(V)}\big)\big| > 1$ as in \textbf{Case 4}, but either 
 $L^{\star}\big(\R^p, q^{\star(j_1)}_{r_1}\big) = 0$ and $L^{\star}\big(H^{\star}(\path_1), q^{\star(j_2)}_{r_2}\big) > 0$, or
 $L^{\star}\big(\R^p, q^{\star(j_1)}_{r_1}\big) > 0$ and $L^{\star}\big(H^{\star}(\path_1), q^{\star(j_2)}_{r_2}\big) = 0$.

The same reasoning than for \textbf{Cases 4} and \textbf{5} applies where the limit law of
$\textbf{L}_{n,\path_2}^{(\bestcut_1^{\star}, \bestcut_{2}^{\star})}$ has both gaussian and $\chi$-square components and is given by 
case (c) or case (d) of Lemma \ref{criterion_distribution_2}.

\end{proof}

\section{Proof of Theorem \ref{stability_equivalent}}
\label{sec_stability_equivalent}

We recall Theorem \ref{stability_equivalent} for the sake of clarity.
\begin{theorem}
If $p_{0}\in[0, 1] \setminus \setpemp$ and $\mathscr{D}'_{n} = \mathscr{D}_{n}$, then, conditional on $\mathscr{D}_{n}$, we have
\begin{align} \label{eq_stab_M}
\lim \limits_{M \to \infty} \hat{S}_{M,n,p_{0}} = 1 \quad \textrm{in probability.}
\end{align}
In addition for $p_{0} < \max \setpemp$,
\begin{align*}
1 - & \mathbb{E} [ \hat{S}_{M,n,p_{0}} | \mathscr{D}_{n}] \\ & \underset{M \to \infty}{\sim} \sum_{\path \in \Pi}
\frac{\Phi( M p_{0}, M, p_n(\path)) ( 1 - \Phi\big( M p_{0}, M,
p_n(\path)))}
{\frac{1}{2} \sum_{\path' \in \Pi} \mathds{1}_{p_n(\path') > p_{0}} +
\mathds{1}_{p_n(\path') > p_{0} - \rho_{n}(\path, \path')
\frac{\sigma_{n}( \path')}{\sigma_{n}( \path) }(p_{0} - p_n(\path)) }},
\end{align*}
where $\Phi ( M p_{0}, M, p_n(\path))$ is the cdf of a binomial distribution with parameter 
$p_n(\path)$, $M$ trials, evaluated at $M p_{0}$, and, for all $\path, \path' \in \Pi$, 
\[
\sigma_{n}( \path) = \sqrt{p_n(\path) (1 -  p_n(\path))},
\]
and
\[
 \rho_{n}(\path, \path') = \frac{\emph{Cov}(\mathds{1}_{\path \in T(\Theta, \mathscr{D}_{n})}, 
\mathds{1}_{\path' \in T(\Theta, \mathscr{D}_{n})} | \mathscr{D}_{n} )}
{\sigma_{n}( \path ) \sigma_{n}( \path')}.
\]
\end{theorem}

Let  $p_{0}\in[0, \max ~\setpemp)
\setminus \setpemp$ and $\Dn' = \Dn$. 
Before proving Theorem \ref{stability_equivalent}, we need the following two lemmas. 

\begin{lemme} \label{F_limit}
Let $F$ be the hypergeometric function. Then, for $(a, c) \in \mathbb{Z}^2$ and $\path \in \Pi$ such that $p_n(\path) > p_0$, we have
\begin{align*}
\lim \limits_{M \to \infty} \frac{F(M + a, 1, M(1-p_0) + c, 1 - p_n(\path))}
{F(M + 1, 1, M(1-p_0) + 1, 1 - p_n(\path))} = 1.
\end{align*}
\end{lemme}
\begin{lemme}
\label{p_conditional}
Let $\path' \in \Pi$. 
For all $\path \in \Pi$ such that $p_n(\path) > p_{0}$, we have
\[
\lim  \limits_{M \to \infty} \P \big( \hat{p}_{M,n}( \path') > p_{0} 
\big | \hat{p}_{M,n}(\path) > p_{0}, \mathscr{D}_{n} \big) = \mathds{1}_{p_n ( \path' ) > p_{0}}
\]
and
\[
\lim  \limits_{M \to \infty} \P \big( \hat{p}_{M,n}( \path') > p_{0} 
\big| \hat{p}_{M,n}(\path) \leq p_{0}, \mathscr{D}_{n} \big)= \mathds{1}_{\substack{p_n(\path') > p_{0} -  \rho_{n}(\path, \path')
\frac{\sigma_{n}( \path')}{\sigma_{n}(\path) } \\ \quad \times (p_{0}
 - p_n(\path))}}.
\]
Symmetrically, for all $\path \in \Pi$ such that $p_n(\path) \leq p_{0}$, we have
\begin{align*}
\lim  \limits_{M \to \infty} \P \big( \hat{p}_{M,n}( \path') > p_{0}
\big | \hat{p}_{M,n}(\path) \leq p_{0}, \mathscr{D}_{n}  \big)
& = \mathds{1}_{p_n ( \path' ) > p_{0} },\\
\lim  \limits_{M \to \infty} \P \big( \hat{p}_{M,n}( \path') > p_{0}
\big | \hat{p}_{M,n}(\path) > p_{0}, \mathscr{D}_{n}  \big)
& = \mathds{1}_{\substack{p_n(\path') > p_{0} - \rho_{n}(\path, \path')
\frac{\sigma_{n}( \path' )}{\sigma_{n}(\path) } \\ \quad \times (p_{0} -p_n(\path)) }}.
 \end{align*}
\end{lemme}
We are now in a position to prove Theorem \ref{stability_equivalent}.
\begin{proof}[Proof of Theorem \ref{stability_equivalent}]
The first statement, identity (\ref{eq_stab_M}), is proved similarly to Corollary 2, using the law of large numbers instead of Theorem \ref{theorem_consistency}. For the second statement, we first recall that, by definition, 
\begin{align*}
\hat{S}_{M_{n},n,p_{0}} & =\frac{2\underset{\path \in \Pi}{\sum}\mathds{1}_{\hat{p}_{M_{n},n}(\path)>p_{0}
\cap\hat{p}_{M_{n},n}'(\path)>p_{0}}}{\underset{\path\in\Pi}{\sum}
\mathds{1}_{\hat{p}_{M_{n},n}(\path)>p_{0}}+\mathds{1}_{\hat{p}_{M_{n},n}'(\path)>p_{0}}} \\
& =1-\frac{\protect\underset{\path\in\Pi}{\sum}\mathds{1}_{\hat{p}_{M,n}(\path)>p_{0}\cap\hat{p}_{M,n}'(\path)\leq p_{0}}+\mathds{1}_{\hat{p}_{M,n}(\path)\leq p_{0}\cap\hat{p}_{M,n}'(\path)>p_{0}}}{\protect\underset{\path\in\Pi}{\sum}\mathds{1}_{\hat{p}_{M,n}(\path)>p_{0}}+\mathds{1}_{\hat{p}_{M,n}'(\path)>p_{0}}}.
\end{align*}
Taking the expectation conditional on $\mathscr{D}_{n}$ gives
\begin{align*}
\mathbb{E}\big[\hat{S}_{M,n,p_{0}} \big| \mathscr{D}_{n}  \big] & =1-2 ~\mathbb{E}\Bigg[\frac{\underset{\path\in\Pi}{\sum}\mathds{1}_{\hat{p}_{M,n}(\path)>p_{0}\cap\hat{p}_{M,n}'(\path)\leq p_{0}}}{\underset{\path\in\Pi}{\sum}\mathds{1}_{\hat{p}_{M,n}(\path)>p_{0}}+\mathds{1}_{\hat{p}_{M,n}'(\path)>p_{0}}} \Bigg| \mathscr{D}_{n}  \Bigg]  \\
& =1-2~\mathbb{E}\Big[\frac{U_{M}}{V_{M}+V_{M}'} \big| \mathscr{D}_{n}  \Big],
\end{align*}
where
$U_{M}=\underset{\path\in\Pi}{\sum}\mathds{1}_{\hat{p}_{M,n}(\path)>p_{0}\cap\hat{p}_{M,n}'(\path)\leq p_{0}}$,
$V_{M}=\underset{\path\in\Pi}{\sum}\mathds{1}_{\hat{p}_{M,n}(\path)>p_{0}}$, and $V_{M}'=\underset{\path\in\Pi}{\sum}\mathds{1}_{\hat{p}_{M,n}'(\path)>p_{0}}$.
Note that 
\begin{align*}
\mathbb{E}[V_{M} | \mathscr{D}_{n} ] & =\underset{\path\in\Pi}{\sum}\mathbb{P}(\hat{p}_{M,n}(\path)>p_{0}
 | \mathscr{D}_{n})\underset{M\rightarrow\infty}{\longrightarrow}\underset{\path\in\Pi}{\sum}\mathds{1}_{p_n(\path)>p_{0}}, \\
 \mathbb{E}[U_{M} | \mathscr{D}_{n}] & =\underset{\path\in\Pi}{\sum}\mathbb{P}(\hat{p}_{M,n}(\path)>p_{0}
 | \mathscr{D}_{n} )\mathbb{P}(\hat{p}_{M,n}(\path)\leq p_{0} | \mathscr{D}_{n} )\underset{M\rightarrow\infty}{\longrightarrow}0.
\end{align*}
Also, 
\begin{align*}
\mathbb{E}\Big[ \frac{U_{M}}{V_{M}+V_{M}'} & \big| \mathscr{D}_{n}  \Big] \\
 = \sum_{m,m'} & \frac{1}{m+m'}\mathbb{E}[U_{M}|V_{M}=m,V_{M}'=m' , \mathscr{D}_{n}] \\ & \times \mathbb{P}(V_{M}=m | \mathscr{D}_{n} )\mathbb{P}(V_{M}'=m' | \mathscr{D}_{n} )\\
 =\sum_{m,m'} & \frac{1}{m+m'}\mathbb{E}\Big[\underset{\path\in\Pi}{\sum}\mathds{1}_{\hat{p}_{M,n}(\path)>p_{0}\cap\hat{p}_{M,n}'(\path)\leq p_{0}}\big|V_{M}=m,V_{M}'=m' , \mathscr{D}_{n} \Big] \\ & \times \mathbb{P}(V_{M}=m | \mathscr{D}_{n})\mathbb{P}(V_{M}'=m' | \mathscr{D}_{n} ) \\
 =\sum_{m,m'}&\frac{1}{m+m'}\underset{\path\in\Pi}{\sum}\mathbb{P}(\hat{p}_{M,n}(\path)>p_{0}|V_{M}=m , \mathscr{D}_{n} ) \\ \times \mathbb{P}&(\hat{p}_{M,n}'(\path)\leq p_{0}|V_{M}'=m' , \mathscr{D}_{n} ) \mathbb{P}(V_{M}=m | \mathscr{D}_{n})\mathbb{P}(V_{M}'=m' | \mathscr{D}_{n} )\\
=\sum_{m,m'}&\frac{1}{m+m'}\underset{\path\in\Pi}{\sum}\mathbb{P}(\hat{p}_{M,n}(\path)>p_{0},V_{M}=m | \mathscr{D}_{n} ) \\ & \times \mathbb{P}(\hat{p}_{M,n}(\path)\leq p_{0},V_{M}'=m' | \mathscr{D}_{n} )\\
 =\underset{\path\in\Pi}{\sum}&\mathbb{P}(\hat{p}_{M,n}(\path)>p_{0} | \mathscr{D}_{n})
\mathbb{P}(\hat{p}_{M,n}(\path)\leq p_{0} | \mathscr{D}_{n} )\\
&\times \Big[\sum_{m,m'}\frac{1}{m+m'}\mathbb{P}(V_{M}=m|\hat{p}_{M,n}(\path)>p_{0} , \mathscr{D}_{n} ) \\ & \quad \times \mathbb{P}(V_{M}'=m'|\hat{p}_{M,n}(\path)\leq p_{0} , \mathscr{D}_{n})\Big].
\end{align*}
For all $\path\in\Pi$,
\begin{align*}
& \P (\hat{p}_{M,n}(\path)>p_{0} | \mathscr{D}_{n} )
 \P (\hat{p}_{M,n}(\path)\leq p_{0}| \mathscr{D}_{n} \big)\\
& \quad = \Phi( M p_{0}, M, p_n(\path)) ( 1 - \Phi( M p_{0}, M, p_n(\path) ) ),
\end{align*}
where $\Phi$ is the cdf of the binomial distribution. As a direct consequence of Lemma \ref{p_conditional}, 
\begin{align*}
&\lim \limits_{M \to \infty} \sum_{m, m'}  \frac{1}{m+m'}\mathbb{P}(V_{M}=m
|\hat{p}_{M,n}(\path)>p_{0}, \mathscr{D}_{n}) \\ 
& \qquad \times
\mathbb{P}(V_{M}=m'|\hat{p}_{M,n}(\path)\leq p_{0}, \mathscr{D}_{n} ) \\
& \quad = \frac{1}{\sum\limits_{\path' \in \Pi} \mathds{1}_{p_n ( \path' ) > p_{0}}
 + \mathds{1}_{p_n(\path') + \rho_{n}(\path, \path')
\frac{\sigma_{n}( \path' )}{\sigma_{n}(\path) } (p_{0} - p_n(\path)) > p_{0}}},
\end{align*}
which yields
\begin{align*}
& 1 - \mathbb{E} [ \hat{S}_{M,n,p_{0}} | \mathscr{D}_{n} ]\\ & \underset{M \to \infty}{\sim} \sum\limits_{\path \in \Pi}
\frac{2 \Phi( M p_{0}, M, p_n(\path)) ( 1 - \Phi( M p_{0}, M, p_n(\path) ) )}
{\sum\limits_{\path' \in \Pi} \mathds{1}_{\hat{p}_{n}(\path') > p_{0}} +
\mathds{1}_{p_n(\path') + \rho_{n}(\path, \path')
\frac{\sigma_{n}( \path' )}{\sigma_{n}(\path) } (p_{0} - p_n(\path)) > p_{0} }}.
\end{align*}
This is the desired result.
\end{proof}

\subsection{Proof of intermediate lemmas}

\begin{proof}[Proof of lemma \ref{F_limit}]
\cite{cvitkovic2017asymptotic} provides an asymptotic expansion of the hypergeometric function $F$ in the case where the first and third parameters
goes to infinity with a constant ratio. For $a, c, z, \varepsilon \in \R$, $b \notin \mathbb{Z} \setminus \mathbb{N}$, 
such that $\varepsilon > 1$, and $z \varepsilon < 1$,  \cite{cvitkovic2017asymptotic} gives in the section 2.2.2 (end of page 10)
\begin{align} \label{eq_F_expansion}
F(a + \varepsilon \lambda, b, c + \lambda, z)  \underset{|\lambda| \to \infty}{\sim} \frac{1}{(1 - \varepsilon z)^b}.
\end{align}
We can then derive the limit of the following ratio
\begin{align} \label{eq_F_ratio}
\lim \limits_{|\lambda| \to \infty} \frac{F(a + \varepsilon \lambda, b, c + \lambda, z)}{F(1 + \varepsilon \lambda, b, 1 + \lambda, z)} = 1
\end{align}
We use \ref{eq_F_ratio} in the specific case where $b = 1$, $a, c \in \mathbb{Z}$, $\varepsilon = \frac{1}{1 - p_0} > 1$, 
 $z = 1 - p_n(\path)$  for $\path \in \Pi$ such that $p_n(\path) > p_0$ (and then $z \varepsilon < 1$),
and $\lambda = M (1 - p_0)$, if follows that
\begin{align} \label{eq_FM_ratio}
\lim \limits_{M \to \infty} \frac{F(M + a , 1, M (1 - p_0) + c ,  1 - p_n(\path) )}
{F(M + 1, 1, M (1 - p_0) + 1,  1 - p_n(\path) )} = 1
\end{align}
\end{proof}

\begin{proof}[Proof of lemma \ref{p_conditional}]
Fix $\mathscr{D}_{n}$.  Let $\path', \path \in \Pi$. 
In what follows, when there is no ambiguity, we will replace $T(\Theta, \mathscr{D}_{n})$ by $T_n(\Theta)$ to lighten notations.

\paragraph{Case 1: $p_n(\path) > p_{0}$} 
\begin{align}
& \E\big[ \hat{p}_{M,n}\big( \path^{'} \big) | \hat{p}_{M,n}(\path) \leq p_{0}, \mathscr{D}_{n}  \big] \nonumber \\
 = & \E\big[ \frac{1}{M}\sum_{l=1}^{M}\mathds{1}_{\path^{'}\in T_n(\Theta_{l})}
 ~\big|~ \hat{p}_{M,n}(\path) \leq p_{0}, \mathscr{D}_{n}  \big] \nonumber \\
 = & \P \big( \path^{'} \in T_n(\Theta_{1})
 | \path \in T_n(\Theta_{1}), \hat{p}_{M,n}(\path) \leq p_{0}, \mathscr{D}_{n}   \big) \nonumber \\
&  \times \P \big( \path \in T_n(\Theta_{1})  | \hat{p}_{M,n}(\path) \leq p_{0}, \mathscr{D}_{n}  \big) \nonumber \\
 & +  \P \big( \path^{'} \in T_n(\Theta_{1})
 | \path \notin T_n(\Theta_{1}), \hat{p}_{M,n}(\path) \leq p_{0}, \mathscr{D}_{n}  \big) \nonumber \\
 & \quad \times \big( 1 - \P \big( \path \in T_n(\Theta_{1})
 | \hat{p}_{M,n}(\path) \leq p_{0}, \mathscr{D}_{n}  \big) \big) \nonumber \\
  = & \P \big( \path^{'} \in T_n(\Theta_{1})
 | \path \in T_n(\Theta_{1}), \mathscr{D}_{n}  \big)
\P \big( \path \in T_n(\Theta_{1})
 | \hat{p}_{M,n}(\path) \leq p_{0}, \mathscr{D}_{n}  \big) \nonumber \\
 & +   \P \big( \path^{'} \in T_n(\Theta_{1})
 | \path \notin T_n(\Theta_{1}), \mathscr{D}_{n}  \big) \nonumber \\
& \quad \times \big( 1 - \P \big( \path \in T_n(\Theta_{1})
 | \hat{p}_{M,n}(\path) \leq p_{0}, \mathscr{D}_{n}  \big) \big). \label{proof_thm2_eq1}
\end{align}
since
\begin{align}
 \P \big( \path^{'} \in T_n(\Theta_{1}) | \path \in T_n(\Theta_{1}), \hat{p}_{M,n}(\path) \leq p_{0}, \mathscr{D}_{n}   \big) \nonumber \\
= \P \big( \path^{'} \in T_n(\Theta_{1}) | \path \in T_n(\Theta_{1}), \mathscr{D}_{n}  \big). \label{proof_thm2_eq2}
\end{align}
because, conditional on $\mathscr{D}_{n}$, the events  $\path^{'}\in T_n(\Theta_{1}), \hdots, \path^{'}\in T_n(\Theta_{M})$ are independent. 
We can rewrite, 
\begin{align}
& \P \big( \path^{'} \in T_n(\Theta_{1}) | \path \notin T_n(\Theta_{1}), \mathscr{D}_{n}  \big) \nonumber \\
& = \frac{\P \big( \path^{'} \in T_n(\Theta_{1}), \path \notin T_n(\Theta_{1})
| \mathscr{D}_{n}  \big)} {1 - p_n(\path)} \nonumber \\
& = \frac{\big( 1 - \P \big( \path \in T_n(\Theta_{1}) | \path^{'} \in T_n(\Theta_{1}), \mathscr{D}_{n}  \big) \big)
p_n\big( \path^{'} \big)}{1 - p_n(\path)} \nonumber \\
& = \frac{p_n\big( \path^{'} \big)}{1 - p_n(\path)}
- \frac{p_n(\path)}{1 - p_n(\path)}
\P \big( \path^{'}  \in T_n(\Theta_{1}) | \path \in T_n(\Theta_{1}), \mathscr{D}_{n}  \big), \label{proof_thm2_eq4}
\end{align}
yielding, using equation (\ref{proof_thm2_eq1}), 
\begin{align} \label{eq_mean_temp_1_0}
 \E \big[& \hat{p}_{M,n}\big( \path^{'} \big) | \hat{p}_{M,n}(\path) \leq p_{0}, \mathscr{D}_{n} \big] \\
 = &\P \big( \path^{'} \in T_n(\Theta_{1})
 | \path \in T_n(\Theta_{1}), \mathscr{D}_{n}  \big)  \Big( \frac{\P \big( \path \in T_n(\Theta_{1})
 | \hat{p}_{M,n}(\path) \leq p_{0}, \mathscr{D}_{n}  \big)}{1 -  p_n(\path)}  \nonumber \\
&-  \frac{p_n(\path)}{1 - p_n(\path)}\Big) + \frac{p_n\big( \path^{'} \big)}{1 - p_n(\path)}
\big( 1 - \P \big( \path \in T_n(\Theta_{1})  | \hat{p}_{M,n}(\path) \leq p_{0}, \mathscr{D}_{n}  \big)  \big). \nonumber
\end{align}
Besides, by definition of the correlation
\begin{align*}
\rho_{n}\big(\path, \path^{'}\big) & = \frac{\textrm{Cov}\big(\mathds{1}_{\path \in T_n(\Theta)}, 
\mathds{1}_{\path^{'} \in T_n(\Theta)} | \mathscr{D}_{n}  \big)}
{\sigma_{n}(\path)\sigma_{n}\big( \path^{'} \big)},
\end{align*}
simple calculations show that 
\begin{align}
\label{eq_mean_temp_2}
 \P & \big( \path^{'} \in T_n(\Theta_{1}) | \path \in T_n(\Theta_{1}),\mathscr{D}_{n}  \big) \nonumber \\
& =p_n\big( \path^{'} \big) +  \rho_{n}\big(\path, \path^{'}\big)
\sqrt{\frac{p_n\big( \path^{'} \big)}{p_n(\path)}
\big(1 - p_n(\path)\big)\big(1 -p_n\big( \path^{'} \big) \big)},
\end{align}
which, together with equation (\ref{eq_mean_temp_1_0}) leads to, 
\begin{align} \label{eq_mean_temp_1}
 \E &\big[ \hat{p}_{M,n}\big( \path^{'} \big) | \hat{p}_{M,n}(\path) \leq p_{0}, \mathscr{D}_{n} \big] \\
&= p_n(\path^{'}) + \rho_{n}\big(\path, \path^{'}\big)
\frac{\sigma_{n}\big( \path^{'} \big)}{\sigma_{n}(\path) }
 \big(\P \big( \path \in T_n(\Theta_{1})
 | \hat{p}_{M,n}(\path) \leq p_{0},\mathscr{D}_{n}  \big) \nonumber \\ & \hspace*{5.5cm} - p_n(\path)\big). \nonumber  
\end{align}
Regarding the probability in the right-hand side of equation (\ref{eq_mean_temp_1}), we have 
\begin{align*}
\P& \big( \path \in T_n(\Theta_{1})
 | \hat{p}_{M,n}(\path) \leq p_{0},\mathscr{D}_{n}  \big) \\
& = p_n\big(\path \big) \frac{\P \big( \hat{p}_{M,n}(\path) \leq p_{0}
  | \path \in T_n(\Theta_{1}),\mathscr{D}_{n}  \big)}
{\P \big( \hat{p}_{M,n}(\path) \leq p_{0} | \mathscr{D}_{n}  \big)} \\
& =  p_n\big(\path \big) \frac{\P \big( (M-1) \hat{p}_{M-1,n}(\path) \leq M p_{0} - 1 | \mathscr{D}_{n}   \big)}
{\P \big( M \hat{p}_{M,n}(\path) \leq M p_{0} | \mathscr{D}_{n}   \big)} \\
& = p_n(\path) \frac{\Phi \big(Mp_0 - 1, M - 1, p_n(\path)\big)}
{\Phi \big(Mp_0, M, p_n(\path)\big)}.
\end{align*}
Using standard formulas, $\Phi$ can be expressed with the incomplete beta function, 
\begin{align*}
\Phi (k, M, p) = I_{1 - p}(M - k, k + 1) = \frac{B_{1-p}(M - k, k + 1)}{B(M - k, k + 1)},
\end{align*}
and the regularized beta function is related to the hypergeometric function $F$, for $a > 0$, $b > 0$, and $p \in [0,1]$
\citep{olver2010nist},
\begin{align*}
B_{1-p}(a,b) = \frac{(1-p)^ap^b}{a} F(a + b, 1, a + 1, 1 - p).
\end{align*}
Then, we can express the cdf of the binomial distribution using the hypergeometric function, and it follows
\begin{align} \label{eq_binom_ratio_0}
\P \big( \path& \in T_n(\Theta_{1})
 | \hat{p}_{M,n}(\path) \leq p_{0},\mathscr{D}_{n}  \big) \\
&= p_{0} \frac{F(M, 1, M (1-p_{0}) + 1, 1 - \hat{p}_{n}\big(\path \big))}{F(M + 1, 1, M (1-p_{0}) + 1, 1 - \hat{p}_{n}\big(\path \big))} \nonumber
\end{align}
According to Lemma \ref{F_limit},
\begin{align} \label{eq_binom_ratio_5}
\lim \limits_{M \to \infty} \frac{F(M, 1, M (1-p_{0}) + 1, 1 - p_n(\path))}
{F(M + 1, 1, M (1-p_{0}) + 1, 1 - p_n(\path))} = 1.
\end{align}
Consequently,
\begin{align} \label{eq_binom_ratio_1}
\lim \limits_{M \to \infty} \P \big( \path \in T(\Theta_{1},\mathscr{D}_{n})
 | \hat{p}_{M,n}(\path) \leq p_{0},\mathscr{D}_{n}  \big) = p_0,
\end{align}
and using this limiting result with equation (\ref{eq_mean_temp_1}) yields,  
\begin{align}
\label{eq_conditional_mean}
\lim \limits_{M \to \infty}  \E\big[ \hat{p}_{M,n}\big( \path^{'} \big)
 | \hat{p}_{M,n}(\path) \leq p_{0},\mathscr{D}_{n}  \big]
= p_n(\path^{'}) + &\rho_{n}\big(\path, \path^{'}\big)
\frac{\sigma_{n}\big( \path^{'} \big)}{\sigma_{n}(\path) } \\
 & \times \big(p_0 - p_n(\path)\big). \nonumber
\end{align}

Regarding the conditional variance, 
\begin{align*}
\mathbb{V} \big[ \hat{p}_{M,n}\big( \path^{'} \big) |& \hat{p}_{M,n}(\path) \leq p_{0},\mathscr{D}_{n}  \big]\\
& = \mathbb{V} \big[ \frac{1}{M}\sum_{l=1}^{M}\mathds{1}_{\path^{'}\in T_n(\Theta_{l})}
 \big| \hat{p}_{M,n}(\path) \leq p_{0},\mathscr{D}_{n}  \big] \\
\end{align*}
\begin{align*}
\mathbb{V} \big[ \hat{p}_{M,n}\big( \path^{'} \big) |& \hat{p}_{M,n}(\path) \leq p_{0},\mathscr{D}_{n}  \big] \\
& = \frac{1}{M} \mathbb{V} \big[ \mathds{1}_{\path^{'} \in T(\Theta_{1},\mathscr{D}_{n})}
 | \hat{p}_{M,n}(\path) \leq p_{0},\mathscr{D}_{n}  \big]\\
& \quad + (1 - \frac{1}{M})\textrm{Cov}(\mathds{1}_{\path^{'} \in T_n(\Theta_{1})},
\mathds{1}_{\path^{'} \in T_n(\Theta_{2})} | \hat{p}_{M,n}(\path) \leq p_{0},\mathscr{D}_{n})\\
& \leq \frac{1}{M} + C_M
 \end{align*}
where
\begin{align*}
C_M & = \textrm{Cov}(\mathds{1}_{\path^{'} \in T_n(\Theta_{1})},
\mathds{1}_{\path^{'} \in T_n(\Theta_{2})} | \hat{p}_{M,n}(\path) \leq p_{0},\mathscr{D}_{n}) \\
& = \P(\path^{'} \in T_n(\Theta_{1}), \path^{'} \in T_n(\Theta_{2})
 | \hat{p}_{M,n}(\path) \leq p_{0},\mathscr{D}_{n}) \\
& \quad - \P(\path^{'} \in T_n(\Theta_{1}) | \hat{p}_{M,n}(\path) \leq p_{0},\mathscr{D}_{n}) \\
& \qquad  \times \P(\path^{'} \in T_n(\Theta_{2}) | \hat{p}_{M,n}(\path) \leq p_{0},\mathscr{D}_{n}) \\
\end{align*}

Then, we follow the same reasoning that leads to equation (\ref{eq_binom_ratio_1}). We can fully expand $C_M$ using Bayes formula, depending whether
$\path \in T_n(\Theta_{1})$ or $\path \in T_n(\Theta_{2})$. 
Note that, since all the trees are independent conditional on $\mathscr{D}_{n}$, we can reduce the conditioning event $\big\{\path \in T_n(\Theta_{1}),  \path \in T_n(\Theta_{2}), \hat{p}_{M,n}(\path) \leq p_{0}, \mathscr{D}_{n}\big\}$
 to $\big\{ \path \in T_n(\Theta_{1}),  \path \in T_n(\Theta_{2}), \mathscr{D}_{n}\big\}$, then
\begin{align*}
C_M  = \P( &\path^{'} \in T_n(\Theta_{1}), \path^{'} \in T_n(\Theta_{2})
 | \path \in T_n(\Theta_{1}), \path \in T_n(\Theta_{2}),\mathscr{D}_{n}) \\
& \times \P(\path \in T_n(\Theta_{1}), \path \in T_n(\Theta_{2})
 | \hat{p}_{M,n}(\path) \leq p_{0},\mathscr{D}_{n}) \\
& - (\P(\path^{'} \in T_n(\Theta_{1}) | \path \in T_n(\Theta_{1}) ,\mathscr{D}_{n}) \\
& \qquad \times \P(\path \in T_n(\Theta_{1}) | \hat{p}_{M,n}(\path) \leq p_{0},\mathscr{D}_{n}))^2 \\[5pt]
 + &2[ \P( \path^{'} \in T_n(\Theta_{1}), \path^{'} \in T_n(\Theta_{2})
 | \path \in T_n(\Theta_{1}), \path \notin T_n(\Theta_{2}),\mathscr{D}_{n}) \\
& \times \P(\path \in T_n(\Theta_{1}), \path \notin T_n(\Theta_{2})
 | \hat{p}_{M,n}(\path) \leq p_{0},\mathscr{D}_{n}) \\
& - \P(\path^{'} \in T_n(\Theta_{1}) | \path \in T_n(\Theta_{1}) ,\mathscr{D}_{n}) \\
& \qquad \times \P(\path \in T_n(\Theta_{1}) | \hat{p}_{M,n}(\path) \leq p_{0},\mathscr{D}_{n}) \\
& \qquad \times \P(\path^{'} \in T_n(\Theta_{1}) | \path \notin T_n(\Theta_{1}) ,\mathscr{D}_{n}) \\
& \qquad \times \P(\path \notin T_n(\Theta_{1}) | \hat{p}_{M,n}(\path) \leq p_{0},\mathscr{D}_{n}) ] \\[5pt]
+ &\P( \path^{'} \in T_n(\Theta_{1}), \path^{'} \in T_n(\Theta_{2})
 | \path \notin T_n(\Theta_{1}), \path \notin T_n(\Theta_{2}),\mathscr{D}_{n}) \\
& \times \P(\path \notin T_n(\Theta_{1}), \path \notin T_n(\Theta_{2})
 | \hat{p}_{M,n}(\path) \leq p_{0},\mathscr{D}_{n}) \\
& - (\P(\path^{'} \in T_n(\Theta_{1}) | \path \notin T_n(\Theta_{1}) ,\mathscr{D}_{n}) \\
& \qquad \times \P(\path \notin T_n(\Theta_{1}) | \hat{p}_{M,n}(\path) \leq p_{0},\mathscr{D}_{n}))^2 \\
\end{align*}
Conditional on $\mathscr{D}_{n}$,  $T_n(\Theta_{1})$ and $ T_n(\Theta_{2})$ are independent, then
\begin{align*}
\P(& \path^{'} \in T_n(\Theta_{1}), \path^{'} \in T_n(\Theta_{2})
 | \path \in T_n(\Theta_{1}), \path \in T_n(\Theta_{2}),\mathscr{D}_{n}) \\
&= \frac{\P( \path^{'} \in T_n(\Theta_{1}), \path^{'} \in T_n(\Theta_{2}), \path \in T_n(\Theta_{1}), \path \in T_n(\Theta_{2})|\mathscr{D}_{n})}{\P(\path \in T_n(\Theta_{1}), \path \in T_n(\Theta_{2}) | \mathscr{D}_{n})}\\
&= \frac{\P( \path^{'} \in T_n(\Theta_{1}), \path \in T_n(\Theta_{1})|\mathscr{D}_{n}) \P(\path^{'} \in T_n(\Theta_{2}), \path \in T_n(\Theta_{2})|\mathscr{D}_{n})}{\P(\path \in T_n(\Theta_{1})|\mathscr{D}_{n}) \P( \path \in T_n(\Theta_{2})|\mathscr{D}_{n})} \\
&= \P( \path^{'} \in T_n(\Theta_{1}) | \path \in T_n(\Theta_{1}),\mathscr{D}_{n}) \P(\path^{'} \in T_n(\Theta_{2}) | \path \in T_n(\Theta_{2}),\mathscr{D}_{n})\\
&= \P( \path^{'} \in T_n(\Theta_{1}) | \path \in T_n(\Theta_{1}),\mathscr{D}_{n})^2
\end{align*}
we can rewrite $C_M$
\begin{align*}
 C_M  & = \P(\path^{'} \in T_n(\Theta_{1}) | \path \in T_n(\Theta_{1}) ,\mathscr{D}_{n})^2 \times \Delta_{M,1} \\
 & \quad + 2 \P(\path^{'} \in T_n(\Theta_{1}) | \path \in T_n(\Theta_{1}) ,\mathscr{D}_{n}) \\
& \qquad \times \P(\path^{'} \in T_n(\Theta_{1}) | \path \notin T_n(\Theta_{1}) ,\mathscr{D}_{n})  \times \Delta_{M,2} \\
 &  \quad + \P(\path^{'} \in T_n(\Theta_{1}) | \path \notin T_n(\Theta_{1}) ,\mathscr{D}_{n})^2  \times \Delta_{M,3},
\end{align*}
where
\begin{align*}
\Delta_{M,1} & = \P(\path \in T_n(\Theta_{1}), \path \in T_n(\Theta_{2})
 | \hat{p}_{M,n}(\path) \leq p_{0},\mathscr{D}_{n}) \\
 & \quad - \P(\path \in T_n(\Theta_{1}) | \hat{p}_{M,n}(\path) \leq p_{0},\mathscr{D}_{n})^2, \\
 \Delta_{M,2} & = \P(\path \in T_n(\Theta_{1}), \path \notin T_n(\Theta_{2})
 | \hat{p}_{M,n}(\path) \leq p_{0},\mathscr{D}_{n}) \\
& \quad - \P(\path \in T_n(\Theta_{1}) | \hat{p}_{M,n}(\path) \leq p_{0},\mathscr{D}_{n}) \\
& \qquad (1 -  \P(\path \in T_n(\Theta_{1}) | \hat{p}_{M,n}(\path) \leq p_{0},\mathscr{D}_{n})),\\
\Delta_{M,3} & = \P(\path \notin T_n(\Theta_{1}), \path \notin T_n(\Theta_{2})
 | \hat{p}_{M,n}(\path) \leq p_{0},\mathscr{D}_{n})\\
 & \quad - \P(\path \notin T_n(\Theta_{1}) | \hat{p}_{M,n}(\path) \leq p_{0},\mathscr{D}_{n})^2.
\end{align*}
We first consider the term 
\begin{align*}
\Delta_{M,1} & = \P(\path \in T_n(\Theta_{1}), \path \in T_n(\Theta_{2})
 | \hat{p}_{M,n}(\path) \leq p_{0},\mathscr{D}_{n}) \\ 
 & \quad - \P(\path \in T_n(\Theta_{1}) | \hat{p}_{M,n}(\path) \leq p_{0},\mathscr{D}_{n})^2
\end{align*}
Equation (\ref{eq_binom_ratio_1}) directly gives,
\begin{equation} \label{eq_binom_ratio_2}
\lim \limits_{M \to \infty} \P(\path \in T_n(\Theta_{1}) | \hat{p}_{M,n}(\path) \leq p_{0},\mathscr{D}_{n})^2
 = p_0^2.
\end{equation}
On the other hand
\begin{align*}
\P&(\path \in T_n(\Theta_{1}), \path \in T_n(\Theta_{2})
 | \hat{p}_{M,n}(\path) \leq p_{0},\mathscr{D}_{n}) \\
&= p_n(\path)^2 \frac{\P( \hat{p}_{M,n}(\path) \leq p_{0} | \path \in T_n(\Theta_{1}), \path \in T_n(\Theta_{2}),\mathscr{D}_{n})}
{\P( \hat{p}_{M,n}(\path) \leq p_{0} | \mathscr{D}_{n})} \\
& = p_n(\path)^2 \frac{\Phi \big(Mp_0 - 2, M - 2, p_n(\path)\big)}
{\Phi \big(Mp_0, M, p_n(\path)\big)}.
\end{align*}
Again, as for equation (\ref{eq_binom_ratio_0}), we can express the cdf of the binomial distribution using the hypergeometric function $F$
\begin{align} \label{eq_binom_ratio_4} 
\P&(\path \in T_n(\Theta_{1}), \path \in T_n(\Theta_{2})
 | \hat{p}_{M,n}(\path) \leq p_{0},\mathscr{D}_{n}) \nonumber \\
& = p_0^2 \Big(1 + \frac{p_0 - 1}{p_0 (M-1)}\Big)  \frac{F(M-1, 1, M(1-p_0) + 1, 1 - p_n(\path))}{F(M+1, 1, M(1-p_0) + 1, 1 - p_n(\path))},
\end{align}
and from Lemma \ref{F_limit},
\begin{align*}
\lim \limits_{M \to \infty} \frac{F(M-1, 1, M(1-p_0) + 1, 1 - p_n(\path))}
{F(M+1, 1, M(1-p_0) + 1, 1 - p_n(\path))} = 1,
\end{align*}
that is 
\begin{align} \label{eq_binom_ratio_3}
\lim \limits_{M \to \infty} \P(\path \in T_n(\Theta_{1}), \path \in T_n(\Theta_{2})
 | \hat{p}_{M,n}(\path) \leq p_{0},\mathscr{D}_{n}) = p_0^2.
\end{align}
Using equations (\ref{eq_binom_ratio_2}) and (\ref{eq_binom_ratio_3}), we conclude
\begin{align*}
\lim \limits_{M \to \infty} \Delta_{M,1} = 0.
\end{align*}
We follow the same reasoning for $\Delta_{M,3}$, equation (\ref{eq_binom_ratio_1}) gives
\begin{equation}
\lim \limits_{M \to \infty} \P(\path \notin T_n(\Theta_{1}) | \hat{p}_{M,n}(\path) \leq p_{0},\mathscr{D}_{n})^2
 = (1 - p_0)^2.
\end{equation}
On the other hand,
\begin{align*}
\P&(\path \notin T_n(\Theta_{1}), \path \notin T_n(\Theta_{2})
 | \hat{p}_{M,n}(\path) \leq p_{0},\mathscr{D}_{n}) \\ 
 &= (1 - p_0)^2 \Big( 1 - \frac{p_0}{M-1} \Big) \frac{F(M-1, 1, M(1-p_0) -1 1, 1 - p_n(\path))}
{F(M+1, 1, M(1-p_0) + 1, 1 - p_n(\path))}
\end{align*}
From Lemma \ref{F_limit},
\begin{align}  \label{eq_binom_ratio_5}
\lim \limits_{M \to \infty} \P(\path \notin T_n(\Theta_{1}), \path \notin T_n(\Theta_{2})
 | \hat{p}_{M,n}(\path) \leq p_{0},\mathscr{D}_{n}) = (1 - p_0)^2 
\end{align}
And finally $\lim \limits_{M \to \infty} \Delta_{M,3} = 0$. The term $\Delta_{M,2}$ can be treated in a similar way, since
equation (\ref{eq_binom_ratio_1}) gives
\begin{align*}
\lim \limits_{M \to \infty}& \P(\path \in T_n(\Theta_{1}) | \hat{p}_{M,n}(\path) \leq p_{0},\mathscr{D}_{n})
\P(\path \notin T_n(\Theta_{1}) | \hat{p}_{M,n}(\path) \leq p_{0},\mathscr{D}_{n})\\
&  = p_0(1 - p_0).    
\end{align*}
Simple identity shows 
\begin{align*}
& \P(\path \in T_n(\Theta_{1}), \path \notin T_n(\Theta_{2})
 | \hat{p}_{M,n}(\path) \leq p_{0},\mathscr{D}_{n}) \\
 & = \frac{1}{2} \Big( 1 -  \P(\path \notin T_n(\Theta_{1}), \path \notin T_n(\Theta_{2})
 | \hat{p}_{M,n}(\path) \leq p_{0},\mathscr{D}_{n}) \\
& \quad - \P(\path \in T_n(\Theta_{1}), \path \in T_n(\Theta_{2})
 | \hat{p}_{M,n}(\path) \leq p_{0},\mathscr{D}_{n}) \Big).
\end{align*}
Taking the limit of the previous equation and using equations (\ref{eq_binom_ratio_3}) and (\ref{eq_binom_ratio_5}), we get
\begin{align}
\lim \limits_{M \to \infty}& \P(\path \in T_n(\Theta_{1}), \path \notin T_n(\Theta_{2})
 | \hat{p}_{M,n}(\path) \leq p_{0},\mathscr{D}_{n}) \nonumber\\  
 & = p_0 ( 1- p_0). \label{equation_proof_bis_bis}
\end{align}
Using (\ref{eq_binom_ratio_1}) and  (\ref{equation_proof_bis_bis}), $\lim \limits_{M \to \infty} \Delta_{M,2}= 0$. Since $\Delta_{M,1}, \Delta_{M,2}, \Delta_{M,3} \to 0$, we obtain $\lim \limits_{M \to \infty} C_M = 0$, that is, 
\begin{align}
\label{eq_conditional_variance}
\lim \limits_{M \to \infty}
\mathbb{V} \big[ \hat{p}_{M,n}\big( \path^{'} \big) | \hat{p}_{M,n}(\path) \leq p_{0},\mathscr{D}_{n}  \big]
= 0.
\end{align}

Finally combining equations (\ref{eq_conditional_mean}) and (\ref{eq_conditional_variance}),
\begin{align*}
\lim  \limits_{M \to \infty} \P \big( \hat{p}_{M,n}\big( \path^{'} \big) > p_{0}
 | \hat{p}_{M,n}(\path) \leq p_{0},\mathscr{D}_{n} \big) \\
= \mathds{1}_{p_n(\path^{'}) + \rho_{n}(\path, \path^{'})
\frac{\sigma_{n}( \path^{'} )}{\sigma_{n}(\path) }
 (p_{0} - p_n(\path)) > p_{0}} 
\end{align*}

\paragraph{Case 2: $p_n(\path) \leq p_{0}$}
By the law of large numbers, $\lim \limits_{M \to \infty}  \hat{p}_{M,n}\big( \path \big) = p_n(\path) $ in probability, 
and consequently  $\lim \limits_{M \to \infty}  \P \big(\hat{p}_{M,n}\big( \path \big) \leq p_0 \big) = 1$.
Additionally, we can simply write
\begin{align*}
 \P \big( &\hat{p}_{M,n}\big( \path^{'} \big) > p_{0}
 | \hat{p}_{M,n}(\path) \leq p_{0},\mathscr{D}_{n} \big) \\
&=  \frac{\P \big( \hat{p}_{M,n}\big( \path^{'} \big) > p_{0}, \hat{p}_{M,n}(\path) \leq p_{0} | \mathscr{D}_{n} \big)}
{\P \big(\hat{p}_{M,n}(\path) \leq p_{0},\mathscr{D}_{n} \big)}
\end{align*} 
Again, by the law of large numbers, $\lim \limits_{M \to \infty}  \hat{p}_{M,n}\big(\path'\big) = p_n(\path') $ in probability.
Then, if $ p_n(\path') > p_0$, $\lim \limits_{M \to \infty}  \P \big(\hat{p}_{M,n}\big( \path' \big) > p_0 \big) = 1$,
and it follows that $\lim \limits_{M \to \infty} \P \big( \hat{p}_{M,n}( \path' ) > p_{0}, \hat{p}_{M,n}(\path) \leq p_{0} | \mathscr{D}_{n} \big) = 1$.
 If $ p_n(\path') \leq p_0$, $\lim \limits_{M \to \infty}  \P \big(\hat{p}_{M,n}\big( \path' \big) > p_0 \big) = 0$,
and consequently $\lim \limits_{M \to \infty} \P \big( \hat{p}_{M,n}\big( \path^{'} \big) > p_{0}, \hat{p}_{M,n}(\path) \leq p_{0} | \mathscr{D}_{n} \big) = 0$.
This can be compacted under the form
\begin{align*}
 \lim \limits_{M \to \infty} \P \big( \hat{p}_{M,n}\big( \path^{'} \big) > p_{0} | \hat{p}_{M,n}(\path) \leq p_{0},\mathscr{D}_{n} \big) = \mathds{1}_{p_n(\path') > p_0}.
\end{align*}

The proof for the case $\P \big[ \hat{p}_{M,n}\big( \path^{'} \big) > p_0 | \hat{p}_{M,n}(\path) > p_{0},\mathscr{D}_{n}  \big]$
is similar.
\end{proof}

\end{document}